\def\thetitle{ Pushing the Boundaries of Tractable Multiperspective Reasoning:\\ A Deduction Calculus for Standpoint $\ELp$}
\def\theauthors{%
	\mbox{Lucía Gómez Álvarez} and \mbox{Sebastian Rudolph} and \mbox{Hannes Strass}
}

\def\theaffiliation{Computational Logic Group, TU~Dresden, Germany\\[0.5ex] Center for Scalable Data Analytics and Artificial Intelligence (ScaDS.AI) Dresden/Leipzig, Germany}
\newcommand{\email}[1]{\href{mailto:#1}{\texttt{#1}}}
\def\theemails{%
	\email{\{lucia.gomez\_alvarez},
	\email{sebastian.rudolph},
	\email{hannes.strass\}@tu-dresden.de}%
}

\documentclass{article}
\usepackage[switch,displaymath,mathlines]{lineno}
\usepackage{kr}

\usepackage{times}
\usepackage{soul}
\usepackage[utf8]{inputenc}
\usepackage[small]{caption}
\usepackage{booktabs}
\usepackage{algorithm}
\usepackage{algorithmic}

\usepackage[hidelinks]{hyperref}
\usepackage{xurl}
\hypersetup{breaklinks=true}
\usepackage[british]{babel}
\usepackage{relsize,setspace}
\usepackage{stmaryrd}
\usepackage[font=it]{caption}
\usepackage{tikz,pgf}\usetikzlibrary{shapes,backgrounds}

\usepackage{amsmath,amssymb}
\usepackage[mathscr]{euscript}

\usepackage{pifont}

\usepackage{multirow}
\usepackage{comment}
\usepackage{xspace}

\usepackage{enumitem}
\setlist{leftmargin=*,noitemsep,parsep=0.2ex,topsep=0.3ex}
\usepackage{scalerel}

\usepackage[normalem]{ulem}

\usepackage{listings}
\usepackage{mdframed}

\usepackage{nimbusmononarrow} 

\newcommand{\shadydiamond}{\mbox{{\rotatebox[origin=c]{45}{\small\ding{113}}}}}
\newcommand{\shadybox}{\mbox{{\rotatebox[origin=c]{0}{\small\ding{113}}}}}

\usepackage[amsthm,thmmarks]{ntheorem}

\usepackage{cleveref}  

\theoremstyle{plain}
\theoremseparator{.}
\theoremheaderfont{\normalfont\bfseries}
\theorembodyfont{\slshape}
\newtheorem{theorem}{Theorem}
\newtheorem{lemma}[theorem]{Lemma}

\newtheorem{corollary}[theorem]{Corollary}

\newtheorem{numberclaim}{Claim}
\theoremstyle{definition}
\theoremsymbol{\ensuremath{\shadydiamond}}
\newtheorem{definition}{Definition}
\newtheorem{example}{Example}
\theoremstyle{nonumberplain}
\theoremheaderfont{\normalfont\bfseries}
\theorembodyfont{\slshape}

\theoremstyle{nonumberplain}
\theoremheaderfont{\normalfont\itshape}
\theoremsymbol{\ensuremath{\shadybox}}
\theoremsymbol{\ensuremath{\shadydiamond}}
\newtheorem{claimproof}{Proof of Claim}

\newenvironment{proofsketch}{\begin{proof}[Proof sketch]}{\end{proof}}

\usepackage{ifthen}
\usepackage{environ}
\newif\iflong
\NewEnviron{shortproof}{\iflong\else\expandafter\begin{proof}\BODY\end{proof}\fi}
\NewEnviron{shortproofsketch}{\iflong\else\expandafter\begin{proofsketch}\BODY\end{proofsketch}\fi}
\NewEnviron{longproof}[1][\unskip]{\iflong\expandafter\begin{proof}[#1]\BODY\end{proof}\fi}
\NewEnviron{supplementary}{\iflong\expandafter\color{teal}\BODY\color{black}\fi}
\NewEnviron{shortonly}{\iflong\else\expandafter\BODY\fi}
\NewEnviron{longonly}{\iflong\expandafter\BODY\fi}

\usepackage[textwidth=1.5cm]{todonotes}
\newcommand{\FormatArithmeticComplexityClass}[1]{\ensuremath{\textsc{#1}}\xspace}

\newcommand{\PTime}{\FormatArithmeticComplexityClass{PTime}}
\newcommand{\NP}{\FormatArithmeticComplexityClass{NP}}

\newcommand{\ExpTime}{\FormatArithmeticComplexityClass{ExpTime}}
\newcommand{\NExpTime}{\FormatArithmeticComplexityClass{NExpTime}}

\newcommand{\TwoExpTime}{\mbox{\sc{2\ExpTime}}\xspace}

\newcommand{\pred}[1]{\mbox{\small\tt {#1}}}

\newcommand{\spform}[1]{\mathsf{#1}}
\newcommand{\spsub}[2]{\spform{#1}_{#2}}

\newcommand{\mathcom}[3]{ \newcommand{#1}[#2]{\mbox{$#3$}}}
\mathcom{\imp}{0}{\ \rightarrow\ }            
\mathcom{\rimp}{0}{\ \leftarrow\ }            

\mathcom{\con}{0}{\ \wedge\ }                 
\mathcom{\dis}{0}{\ \vee\ }                   
\mathcom{\n}{0}{\neg}                     
\mathcom{\dimp}{0}{\ \leftrightarrow\ }       
\mathcom{\corresponds}{0}{\ \Lleftarrow\! \! \Rrightarrow\ }
\mathcom{\A}{0}{\forall}                  
\mathcom{\E}{0}{\exists}     

\def\Box{\mathop\square}
\def\Diamond{\mathop\lozenge}

\mathcom{\tuple}{1}{\left\langle #1 \right\rangle}
\mathcom{\stuple}{1}{\langle #1 \rangle}

\def\eqdef{\mathrel{\ =_{\mbox{\em \tiny def}}\ }}

\def\sp{\hbox{$\spform{s'}$}\xspace}
\def\spp{\hbox{$\spform{s''}$}\xspace}
\def\st{\hbox{$\spform{s}$}\xspace}
\def\s#1{\hbox{$\spform{s}_{#1}$}\xspace}
\def\stand#1{\hbox{$\spform{#1}$}\xspace}

\def\star{\hbox{$*$}\xspace}

\def\pr{\pi}

\def\Precs{\Pi}

\newcommand{\standb}[1]{\mathord{\Box\nolimits_{\spform{#1}}}}

\newcommand{\standd}[1]{\mathord{\Diamond\nolimits_{\spform{#1}}}}
\newcommand{\standdsub}[2]{\mathord{\Diamond\nolimits_{\spsub{#1}{#2}}}}

\newcommand{\standv}[1]{\mathord{\mathop{\odot}\nolimits_{\spform{#1}}}}

\def\standbx#1{\mathord{\Box\nolimits_{\scaleobj{0.8}{\spform{#1}}}}}
\def\standdx#1{\mathord{\Diamond\nolimits_{\scaleobj{0.8}{\spform{#1}}}}}

\newcommand{\standbs}{\standb{s}}
\def\standbu{\standb{u}}
\def\standbv{\standb{v}}
\def\standbsp{\standb{s'}}

\newcommand{\standds}{\standd{s}}

\def\standdu{\standd{u}}

\def\standdvo{\standdsub{v}{0}}
\def\standdvi{\standdsub{v}{1}}
\def\standdsp{\standd{s'}}

\newcommand{\standvs}{\standv{s}}
\def\standvu{\standv{u}}
\def\standvsp{\standv{s'}}

\def\allstandb{\standb{*}}
\def\standball{\standb{*}}

\xspace
\def\f\xspacestandtopre{\hbox{$\sigma\,$}\xspace}
\def\fpretov\xspacealue{\hbox{$\delta\,$}\xspace}

\def\ModSat#1||-#2{#1\models #2}
\def\NotModSat#1||-#2{#1\nvDash #2}

\newcommand{\define}[1]{\emph{#1}}
\newcommand{\N}{\mathbb{N}}
\newcommand{\set}[1]{\left\{#1\right\}}
\newcommand{\guard}{\;\middle\vert\;}

\newcommand{\size}[1]{\left\lVert#1\right\rVert}

\renewcommand{\eqdef}{\mathrel{\,:=\,}}
\newcommand{\iffdef}{\mathrel{\;:\mathrel{\mkern-6mu}{\Longleftrightarrow}\;}}
\newcommand{\pheq}{\mathrel{\phantom{=}}}
\newcommand{\phiff}{\mathrel{\phantom{\iff}}}
\newcommand{\phimplies}{\mathrel{\phantom{\implies}}}
\newcommand{\ebnfeq}{\mathrel{::=}}
\newcommand{\ebnfalt}{\mathrel{\hspace{-1pt}\ \mid\ \hspace{-1pt}}}

\renewcommand{\land}{\mathrel{\wedge}}

\renewcommand{\S}{\mathcal{S}}

\def\Stands{N_{\mathsf{S}}}
\def\Concepts{N_{\mathsf{C}}}
\def\Roles{N_{\mathsf{R}}}
\def\Individuals{N_{\mathsf{I}}}
\def\sts{\spform{s}} 
\def\stsp{\spform{s'}} 
\def\stu{\spform{u}} 
\def\stv{\spform{v}} 
\def\stvb{\spsub{v}{b}} 
\def\stvo{\spsub{v}{0}} 
\def\stvi{\spsub{v}{1}} 
\def\sto{\mathbf{0}}
\def\prvo{\pr'_{\stvo}}
\def\prvi{\pr'_{\stvi}}
\def\prvb{\pr'_{\stvb}}
\def\prvy{\pr'_{\spsub{v}{y}}}

\def\E{\mathcal{E}}

\def\struct{\mathcal{I}}

\def\intf{\cdot^{\struct}}
\def\Dom{\Delta}
\def\de{\delta}
\def\ve{\varepsilon}

\def\ze{\zeta}
\newcommand{\interpret}[2]{#1^{#2}}
\newcommand{\interprets}[1]{\interpret{#1}{\struct}}

\newcommand{\interpretgp}[1]{\interpret{#1}{\gamma(\pr)}}
\newcommand{\interpretgps}[1]{\interpret{#1}{\gamma(\pr_*)}}
\newcommand{\interpretpgp}[1]{\interpret{#1}{\gamma'(\pr)}}
\newcommand{\interpretgpp}[1]{\interpret{#1}{\gamma(\pr')}}
\newcommand{\interpretpgpp}[1]{\interpret{#1}{\gamma'(\pr')}}
\newcommand{\interpretpgppp}[1]{\interpret{#1}{\gamma'(\pr'')}}

\def\dland{\sqcap}

\def\dlsub{\sqsubseteq}
\newcommand{\T}{\mathcal{T}}
\renewcommand{\S}{\mathcal{S}}

\newcommand{\CT}{\mathcal{C}_\T}
\newcommand{\CTp}{\mathcal{C}_{\T'}}
\newcommand{\CTpp}{\mathcal{C}_{\T''}}
\renewcommand{\A}{\mathcal{A}}
\newcommand{\K}{\mathcal{K}}
\newcommand{\dlstruct}{\mathfrak{D}}

\def\EL{\ensuremath{\mathcal{E\!L}}\xspace}
\def\ELp{\ensuremath{\mathcal{E\!L}+}\xspace}

\newcommand{\goesto}{\mathbin{\qquad\longrightarrow\qquad}}

\newcommand{\hs}[1]{\textcolor{teal}{HS:\xspace #1}}

\newenvironment{firstexample}[2]{
\ifthenelse{\equal{#2}{}}{\begin{example}}{\begin{example}[#2]}
    \label{#1}
    }{
  \end{example}
}

\newenvironment{cexample}[1]{
  \begin{example}[Continued from \Cref{#1}]
    }{
  \end{example}
}

\newcommand{\axiom}{\xi}
\newcommand{\literal}{\lambda}
\newcommand{\monomial}{\mu}
\newcommand{\formula}{\varphi}

\newcommand{\Souffle}{\textsc{Soufflé}\xspace}

\newcommand{\implementationurl}{\url{https://github.com/cl-tud/standpoint-el-souffle-reasoner}}

\newcommand{\SROIQbs}{\ensuremath{\mathcal{SROIQ}b_s}\xspace}

\RequirePackage{graphicx}

\newcommand{\kb}{\text{\rm{KB}}} 

\newcommand{\KB}{\text{$\mathcal{K}$}\xspace}

\newcommand{\Self}{\ensuremath{\mathsf{Self}}}
\newcommand{\exself}{\ensuremath{\exists R.\Self}}

\usepackage{ upgreek }

\def\kb{\mathcal{K}}
\def\SEL{\mathbb{S}_{\EL}}
\def\SELp{\mathbb{S}_{\ELp}}
\def\SC{\mathsf{ST}_{\kb}}

\def\SFCT{\mathsf{SC}_{\T}}
\def\sub{\mathsf{sub}}

\def\BCC{\mathsf{BC}_{\kb}}

\def\int{^{\mathcal{I}}}

\usepackage{mathtools}
\newcommand{\alignedop}[2]{
    \begin{array}[t]{@{}l@{}}
        \displaystyle{#1}  \\[0.83ex]
        \mathrlap{^{\!#2}} \\[-0.91ex]
    \end{array}}

\newcommand{\mycup}[1]{\alignedop{\bigcup}{#1}}

\title{\thetitle}

 \author{
 \theauthors
 \affiliations
 \theaffiliation
 \emails
 \theemails
 }

\makeatletter
\def\namedlabel#1#2{\begingroup
   \def\@currentlabel{#2}%
   \label{#1}\endgroup
}
\makeatother

\longfalse

\begin{document}

\maketitle


\newenvironment{aligna}{\linenomathNonumbers\begin{align*}}{\end{align*}\endlinenomath}

\begin{abstract}
    Standpoint $\EL$ is a multi-modal extension of the popular description logic $\EL$ that allows for the integrated representation of domain knowledge relative to diverse standpoints or perspectives. Advantageously, its satisfiability problem has recently been shown to be in \PTime, making it a promising framework for large-scale knowledge integration.

    In this paper, we show that we can further push the expressivity of this formalism, arriving at an extended logic, called Standpoint $\ELp$, which allows for axiom negation, role chain axioms, self-loops, and other features, while maintaining tractability. This is achieved by designing a satisfiability-checking deduction calculus, which at the same time addresses the need for practical algorithms. We demonstrate the feasibility of our calculus by presenting a prototypical Datalog implementation of its deduction rules.
\end{abstract}

\section{Introduction}


The Semantic Web enables the exploitation of artefacts of knowledge representation (e.g., ontologies, knowledge bases, etc.) to support increasingly sophisticated automated reasoning tasks over linked data from various sources.
\emph{Description logics} (DLs) \cite{baader_horrocks_lutz_sattler_2017,Rudolph11} are a prominent class of logic-based KR formalisms in this context since they provide the theoretical underpinning of the Web Ontology Language (OWL 2), the main KR standard by the W3C \cite{owl2-overview}.

In particular, the lightweight description logic $\EL$ \cite{Baader05ELenvelope} serves as the core of a popular family of DLs which is the formal basis of OWL~2~EL \cite{owl2-profiles}, a widely used tractable profile of OWL~2. One major appeal of the $\EL$ family is that basic reasoning tasks can be performed in polynomial time with respect to the size of the ontology, enabling reasoning-supported creation and maintenance of very large ontologies.
An example is SNOMED CT \cite{donnelly2006snomed}, which is the largest healthcare ontology and has a broad user base including clinicians, patients, and researchers.

Beyond the scalable reasoning, the Semantic Web must provide mechanisms for the combination and integrated exploitation of the many knowledge sources available. Yet, its decentralised nature has led to the proliferation of ontologies of overlapping knowledge, which inevitably reflect different points of view and follow diverging modelling principles. For instance, in the medical domain, some sources may use the concept $\pred{Tumour}$ to denote a process and others to denote a piece of tissue. Similarly, $\pred{Allergy}$ may denote an allergic reaction or just an allergic disposition. These issues pose well-known challenges in the area of knowledge integration.

Common ontology management approaches fully merge knowledge perspectives, which often requires logical weakening in order to maintain consistency. For instance, an initiative proposed the integration of SNOMED CT with the FMA1140 (Foundational Model of Anatomy) and the NCI (National Cancer Institute Thesaurus), into a unified combination called LargeBio \cite{Osman2021OntologyIssues}, and reported ensuing challenges.
Beyond the risk of causing inconsistencies or unintended consequences, the unifying approach promotes weakly axiomatised ontologies designed to avoid conflict in any context of application at the expense of richer theories that would allow for further inferencing. Hence, while frameworks supporting the integrated representation of multiple perspectives seem preferable to recording the distinct views in a detached way, entirely merging them comes with significant downsides.

This need of handling multiple perspectives in the Semantic Web has led to several logic-based pro\-posals.
The closest regarding goals are multi-viewpoint ontologies \cite{Hemam2011MVP-OWL:Web,Hemam2018Probabilistic}, which often model the intuition of viewpoints in a tailored extension of OWL for which no complexity bounds are given.
Related problems are also addressed in work on contextuality, e.g.\ C-OWL, DDL, and CKR \cite{Bouquet2003C-OWL:Ontologies,Borgida2003DistributedSources,SERAFINI201264}).

Modal logics are natural frameworks for modelling contexts and perspectives \cite{Klarman2016DescriptionContext,gomez2021standpoint}, and in contrast to tailored multi-perspective frameworks, they benefit from well-understood semantics. However, the interplay between DL constructs and modalities is often not well-behaved and can easily endanger the decidability of reasoning tasks or increase their complexity \cite{baader1995multi,mosurovic1999complexity,WolterMultiDimensionalDescriptionLogics}. Notable examples are $\NExpTime$-completeness of the multi-modal description logic $\mathbf{K}_\mathcal{ALC}$ \cite{LutzSWZ02} 
and $\TwoExpTime$-completeness of $\mathcal{ALC}_{\mathcal{ALC}}$ \cite{Klarman2016DescriptionContext}, a modal contextual logic framework in the style proposed by \citeauthor{mccarthy1997context}~(\citeyear{mccarthy1997context}).\pagebreak


\emph{Standpoint logic}~\cite{gomez2021standpoint} is a recently proposed formalism that is rooted in modal logic and allows for the simultaneous representation of multiple, potentially contradictory viewpoints and the establishment of alignments between them.
This is achieved by extending a given base logic (propositional logic in the case of \citeauthor{gomez2021standpoint}, description logic $\EL$ herein) with labelled modal operators, where
propositions  $\standbx{S}\phi$ and $\standdx{S}\phi$ express information relative to the \emph{standpoint} $\mathsf{S}$ and read, respectively: ``according to $\mathsf{S}$, it is \emph{unequivocal/conceivable} that~$\phi$''.
Semantically, standpoints are represented by sets of \emph{precisifications},\footnote{Precisifications are analogous to the \emph{worlds} of modal-logic frameworks with possible-worlds semantics.} such that $\standbx{S}\phi$ and $\standdx{S}\phi$ hold if $\phi$ is true in \mbox{all/some} of the precisifications associated with $\mathsf{S}$.


The following example illustrates the use of standpoint logic for knowledge integration in the medical domain.

\begin{firstexample}{example:fol}{Tumour Disambiguation}
    A hospital {\small$\mathsf{H}$} and a laboratory {\small$\mathsf{L}$} have developed their own knowledge bases and aim to make them interoperable.
    Hospital {\small$\mathsf{H}$} gathers clinical data about patients, which may be used to determine whether a person has priority at the emergency service.
    According to {\small$\mathsf{H}$}, a $\pred{Tumour}$ is a process by which abnormal or damaged cells grow and multiply (Formula \ref{formula:TT-main-pre}), and patients that conceivably have a $\pred{Tumour}$ have a $\pred{HighRisk}$ priority (Formula \ref{formula:high-risk}).
    The laboratory {\small$\mathsf{L}$} annotates patients' radiographs, and models $\pred{Tumour}$ as a lump of tissue (Formula \ref{formula:TT-main}).
    \begin{align}
        \standbx{H}[\pred{Tumour}                                                    & \sqsubseteq\pred{Process}]\label{formula:TT-main-pre} \\
        \standbx{H}[\pred{Patient}\sqcap\standdx{H}[\exists\pred{HasProcess.Tumour}] & \sqsubseteq\pred{HighRisk}]\label{formula:high-risk}  \\
        \standbx{L}[\pred{Tumour}                                                    & \sqsubseteq\pred{Tissue}]\label{formula:TT-main}
    \end{align}
    \noindent
    Both institutions inherit from {\small$\mathsf{SN}$}, which contains the original SNOMED CT as well as 
    patient data (Formula \ref{formula:inheritance-from-snomed}, with the operator $\preceq$ encoding the inheritance). Among the background knowledge in {\small$\mathsf{SN}$}, we have that $\pred{Tissue}$ and $\pred{Process}$ are disjoint classes (Formula \ref{formula:SN-tissue-process-disjoint}) and that everything that has a part which has a process, has that process itself (Formula~\ref{formula:hasprocess-transitive}).

    \vspace{-2.6ex}
    \begin{linenomath*}
        {
            \begin{gather}
                \spform{\small H}\preceq\spform{\small SN}           \qquad        \spform{\small L}\preceq\spform{\small SN}\quad \label{formula:inheritance-from-snomed} \\
                \standbx{SN}[\pred{Tissue}\sqcap\pred{Process}          \sqsubseteq\bot]\qquad\ \, \label{formula:SN-tissue-process-disjoint}\\
                \standbx{SN}[\pred{HasPart}\circ\pred{HasProcess}\sqsubseteq\pred{HasProcess}]\ \label{formula:hasprocess-transitive}
            \end{gather}}
    \end{linenomath*}%
    \noindent
    While clearly incompatible due to Formula \ref{formula:SN-tissue-process-disjoint}, the perspectives of $\mathsf{H}$ and $\mathsf{L}$ are semantically close and and we may be aware of further complex relations between their perspectives. For instance, we might assert that whenever a clinician at $\mathsf{L}$ deems that a cancerous lump of tissue is large enough to conceivably be a $\pred{Tumour}$ (tissue), then it is unequivocally undergoing a $\pred{Tumour}$ (process) according to  {\small$\mathsf{H}$} (Formula \ref{formula:TT-TP-bridge}).
    Or we might want to specify negative information such as non-subsumption between the classes of unequivocal instances of $\pred{Process}$ according to {\small$\mathsf{H}$} and to
        {\small$\mathsf{L}$} (Formula \ref{formula:TT-TP-process-bridge}).
    \vspace*{-2pt}
    \begin{align}
        \standdx{L}[\pred{Tumour}]       & \sqsubseteq\standbx{H}[\pred{Tissue} \sqcap \exists\pred{HasProcess.Tumour}] \label{formula:TT-TP-bridge} \\
        \neg(\standbx{H}[\pred{Process}] & \sqsubseteq\standbx{L}[\pred{Process}]) \label{formula:TT-TP-process-bridge}
    \end{align}
    \vspace*{-2pt}
    \noindent
    Finally, these sources may also have assertional knowledge:
    \begin{align}
         & \small \standbx{SN} \small [ \pred{Patient}(p1) \wedge \pred{HasPart}(p1,a)\wedge\pred{Colon}(a) ]  \label{formula:instance-patient1} \\
         & \standbx{H} \small [ \pred{HasPart}(a,b)] \label{formula:instance-tumour}                                                             \\
         & \small \standdx{L}[\pred{Tumour}(b)] \label{formula:instance-tumour2}
    \end{align}  
    That is, through $\mathsf{SN}$, both $\mathsf{H}$ and $\mathsf{L}$ know of a patient $p1$ and their body parts (Formula \ref{formula:instance-patient1}) and, in view of some radiograph requested by $\mathsf{H}$ on part $b$ of this patient's colon (Formula~\ref{formula:instance-tumour}), $\mathsf{L}$ suspects there may be tumour tissue (Formula~\ref{formula:instance-tumour2}).
\end{firstexample}


In the first place, one should notice that a naive, standpoint-free integration of the knowledge bases without the standpoint infrastructure would trigger an inconsistency.
Specifically, from a $\pred{Tumour}(b)$ instance we could infer both that $\pred{Tissue}(b)$ and $\pred{Process}(b)$ using the background knowledge of $\mathsf{H}$ and $\mathsf{L}$, which in turn would lead to inconsistency with the $\mathsf{SN}$ axiom stating $\pred{Tissue}\sqcap\pred{Process}  \sqsubseteq\bot$.
Instead, with Standpoint $\ELp$, the logical statements {(\ref{formula:TT-main})--(\ref{formula:instance-tumour2})} formalising \Cref{example:fol} are not inconsistent, so all axioms can be jointly represented.
On the one hand, we will be able to infer that $\mathsf{H}$ and $\mathsf{L}$ are indeed incompatible, denoted by $\mathsf{H}\cap\mathsf{L}\preceq \mathbf{0}$ and obtained from Formulas (\ref{formula:TT-main}), (\ref{formula:inheritance-from-snomed}), (\ref{formula:SN-tissue-process-disjoint}) and (\ref{formula:instance-tumour}).
On the other hand, beyond preserving consistency, the use of standpoint logic supports reasoning with and across individual perspectives.

\begin{cexample}{example:fol}
    Assume that patient $p1$, of which laboratory $\mathsf{L}$ detected a tumour tissue (Formula \ref{formula:instance-tumour}), registers at emergencies in hospital $\mathsf{H}$.
    From the knowledge expressed in Formulas {(\ref{formula:TT-main})--(\ref{formula:instance-tumour2})}, we can infer
    %
    %
    \noindent
    \begin{align}
         & \text{\!\!\!\!via (\ref{formula:TT-TP-bridge}) and (\ref{formula:instance-tumour2})}                                             &  & \standbx{H}[(\exists\pred{HasProcess.Tumour})(b)]\label{formula:hasprocess-b}      \\
         & \text{\!\!\!\!via (\ref{formula:hasprocess-transitive}),\! (\ref{formula:instance-tumour}) and (\ref{formula:hasprocess-b})}\!\! &  & \standbx{H} [(\exists\pred{HasProcess.Tumour})(a)]    \label{formula:hasprocess-a} \\
         & \text{\!\!\!\!via (\ref{formula:hasprocess-transitive}),\! (\ref{formula:instance-patient1}) and (\ref{formula:hasprocess-a})}   &  & \standbx{H} [(\exists\pred{HasProcess.Tumour})(p1)]  \label{formula:hasprocess-p1} \\
         & \text{\!\!\!\!via (\ref{formula:inheritance-from-snomed}) and (\ref{formula:instance-patient1})}                                 &  & \standbx{H} [\pred{Patient}(p1)]  \label{formula:h-patient-p1}                     \\
         & \text{\!\!\!\!via (\ref{formula:high-risk}),\! (\ref{formula:hasprocess-p1}) and (\ref{formula:h-patient-p1})}\!\!               &  & \standbx{H} [\pred{HighRisk}(p1)]  \label{formula:h-highrisk}
    \end{align}
    meaning that, according to $\mathsf{H}$, $p1$ has a tumour process and is classified as `high risk'.
\end{cexample}

Formally, \emph{Standpoint logics}
are multi-modal logics characterised by a simplified Kripke semantics, which brings about beneficial computational properties in different settings. For instance, it is known that adding sentential standpoints (where applying modal operators to formulas with free variables is disallowed)
does not increase the complexity of numerous \emph{$\NP$-hard} FO-fragments \cite{gomez-alvarez22howtoagree}, including the expressive DL $\SROIQbs$, a logical basis of OWL~2~DL \cite{owl2-semantics}.

Yet, a fine-grained terminological alignment between different perspectives requires concepts preceded by modal operators, as in Axiom~(\ref{formula:TT-TP-bridge}), which falls out of the sentential fragment.
Recently, \citeauthor{ourijcaisubmission}~(\citeyear{ourijcaisubmission}) introduced a \emph{standpoint version} of the description logic $\EL$, called Standpoint~$\EL$, and established that it exhibits $\EL$'s favourable \PTime standard reasoning, while showing that introducing additional features like empty standpoints, rigid roles, and nominals makes standard reasoning tasks intractable.
In this paper, we show that we can push the expressivity of Standpoint~$\EL$ further while retaining tractability.
We present an extended logic, called Standpoint $\ELp$, which allows for axiom negation, role chain axioms, self-loops, and other features.
This result is achieved by designing the first satisfiability-checking deduction calculus for a standpoint-enhanced DL, thus at the same time addressing the need for practical algorithms.

\clearpage

Our paper is structured as follows.
After introducing the syntax and semantics of Standpoint $\EL+$ (denoted $\SELp$) and a suitable normal form (\Cref{sec:syntax-semantics}), we establish our main result: satisfiability checking and statement entailment in $\SELp$ is tractable.
We show this by providing a particular Hilbert-style deduction calculus (\Cref{sec:calculus}) that operates on axioms of a fixed shape and bounded size, which immediately warrants that saturation can be performed in \PTime. For said calculus, we establish soundness and refutation-completeness.
In \Cref{sec:implementation}, we briefly describe a proof-of-concept implementation of our approach, showing portions of the actual code to illustrate the key ideas.
We conclude the paper in \Cref{sec:conclusion} with a discussion of future work. 

An extended version of the paper with proofs of all results is available as a technical appendix.

\section{Syntax, Semantics, and Normalisation}\label{sec:syntax-semantics}

We now introduce syntax and semantics of Standpoint $\ELp$ (referred to as $\SELp$) and propose a normal form that is useful for subsequent algorithmic considerations.

\subsubsection{Syntax}

A \emph{Standpoint DL vocabulary} is a traditional DL vocabulary consisting of
sets $\Concepts$ of \define{concept names},
$\Roles$ of \define{role names}, and
$\Individuals$ of \define{individual names}, extended by an additional set $\Stands$ of \define{standpoint names} with \mbox{$\star\in\Stands$}.
A \emph{standpoint operator} is of the form  $\standds$ (``diamond'') or $\standbs$ (``box'') with \mbox{$\sts\in\Stands$};
we use $\standvs$ to refer to either.\footnote{We use brackets $[\ldots]$ to delimit the scope of the operators.}

\begin{itemize}
    \item \define{Concept terms} are defined via (where \mbox{$A\in\Concepts$}, \mbox{$R\in\Roles$})%
          \begin{linenomath*}
              \[
                  C \ebnfeq \top \ebnfalt \bot \ebnfalt A \ebnfalt C_1 \dland C_2 \ebnfalt \exists R.C \ebnfalt \standvs C \ebnfalt \exists R.\Self
              \]
          \end{linenomath*}
    \item A \define{general concept inclusion (GCI)} is of the form \mbox{$C \dlsub D$}, where $C$ and $D$ are concept terms.
    \item A \define{role inclusion axiom (RIA)} is of the form \mbox{$R_1\circ\ldots\circ R_n\dlsub R$} where \mbox{$n\geq 1$}, \mbox{$R_1,\ldots,R_n,R\in\Roles$}.
    \item A \define{concept assertion} is of the form $C(a)$, where $C$ is a concept term and \mbox{$a\in\Individuals$}.
    \item A \define{role assertion} is of the form $R(a,b)$, with \mbox{$a,b\in\Individuals$} and \mbox{$R\in\Roles$}.
%
    \item An \define{axiom} $\axiom$ is a GCI, RIA, or concept/role assertion.
    \item A \define{literal} $\literal$ is an axiom $\axiom$ or a negated axiom $\neg\axiom$.
    \item A \define{monomial} $\monomial$ is a conjunction \mbox{$\literal_1 \land \ldots \land \literal_m$} of literals.
    \item A \define{formula} $\formula$ is of the form $\standvs \monomial$ for a monomial $\monomial$ and $\sts\in\Stands$.
    \item A \define{sharpening statement} is of the form \mbox{$\sts_1\cap\ldots\cap\sts_n \preceq \sts$} where \mbox{$n\geq 1$}, \mbox{$\sts_1,\ldots,\sts_n\in\Stands$}, and \mbox{$\sts\in\Stands\cup\set{\sto}$}.\footnote{$\sto$ is used to express standpoint disjointness as in \mbox{$\sts\cap\stsp\preceq\sto$}.}
\end{itemize}
Note that in particular, monomials subsume (finite) knowledge bases of the $\EL$ family;
monomials even go beyond that in allowing for the occurence of \emph{negated} axioms.
Yet, monomials do not have the full expressive power of arbitrary Boolean combinations of axioms, which is a necessary restriction in order to maintain tractability.

A \define{$\SELp$ knowledge base} (KB) is a finite set of formulae and possibly negated sharpening statements.
We refer to arbitrary elements of $\K$ as \emph{statements}.
%
Note that all statements except sharpening statements are preceded by modal operators (\emph{``modalised''} for short).
%

\pagebreak
\subsubsection{Semantics}
The semantics of $\SELp$ is defined via standpoint structures.
Given a Standpoint DL vocabulary $\tuple{\Concepts,\Roles,\Individuals,\Stands}$, a \define{description logic standpoint structure} is a tuple
\mbox{$\dlstruct = \tuple{\Dom, \Precs, \sigma, \gamma}$} where:
\begin{itemize}
    \item $\Dom$ is a non-empty set, the \define{domain} of $\dlstruct$;
    \item $\Precs$ is a set, called the \define{precisifications} of $\dlstruct$;
    \item $\sigma$ is a function mapping each standpoint symbol to a non-empty\footnote{As shown in our prior work \cite{ourijcaisubmission}, allowing for ``empty standpoints'' immediately incurs  intractability, even for an otherwise empty vocabulary.} subset of $\Precs$ while we set $\sigma(\sto)=\emptyset$;
    \item $\gamma$ is a function mapping each precisification from $\Precs$ to an ``ordinary'' DL interpretation \mbox{$\struct=\stuple{\Dom,\intf}$} over the domain $\Dom$, where the interpretation function $\intf$ maps\/:
          \begin{itemize}
              \item each concept name \mbox{$A\in\Concepts$} to a set \mbox{$\interprets{A}\subseteq\Dom$},
              \item each role name \mbox{$R{\,\in\,}\Roles$} to a binary relation \mbox{$\interprets{R}{\,\subseteq\,}\Dom{\times}\Dom$},
              \item each individual name \mbox{$a\in\Individuals$} to an element \mbox{$\interprets{a}\in\Dom$},
          \end{itemize}
          and we require \mbox{$a^{\!\gamma(\pi)} = a^{\!\gamma(\pi')}$} for all \mbox{$\pi,\pi'\! \in \Precs$} and \mbox{$a\in\Individuals$}.
\end{itemize}
Note that by this definition, individual names (also referred to as constants) are interpreted rigidly, i.e., each individual name $a$ is assigned the same \mbox{$a^{\gamma(\pi)} \in \Delta$} across all precisifications \mbox{$\pi \in \Precs$}.
We will refer to this uniform $a^{\gamma(\pi)}$ by $a^{\dlstruct}$.

For each $\pr\in\Precs$, we extend the interpretation mapping $\struct=\gamma(\pr)$ to concept terms via structural induction:
\begin{align*}
    \interprets{\top}              & \eqdef \Dom        \\[-0.2ex]
    \interprets{\bot}              & \eqdef \emptyset\  \\[-0.2ex]
    \interprets{(\standds C)}      & \eqdef \textstyle\bigcup_{\pr'\in\sigma(\sts)}C^{\gamma(\pr')} \\[-0.2ex]
    \interprets{(\standbs C)}      & \eqdef \textstyle\bigcap_{\pr'\in\sigma(\sts)}C^{\gamma(\pr')} \\[-0.2ex]
    \interprets{(C_1\dland C_2)}   & \eqdef \interprets{C_1} \cap \interprets{C_2}                                                                                     \\[-0.2ex]
    \interprets{(\exists R.C)}     & \eqdef \set{ \delta\in\Dom \guard \tuple{\delta,\varepsilon}\in\interprets{R} \text{ for some } \varepsilon\in\interprets{C} }    \\
    \interprets{(\exists R.\Self)} & \eqdef \set{ \de\in\Dom \guard \tuple{\de,\de}\in\interprets{R} }
\end{align*}

A role chain expression \mbox{$\rho=R_1\circ R_2\circ\ldots\circ R_n$} is interpreted as \mbox{$\interprets{\rho}\eqdef((\cdots(\interprets{R_1}\circ\interprets{R_2}) \circ \ldots ) \circ \interprets{R_n} )$}, where, as usual, \mbox{$R\circ U \eqdef \set{ \tuple{x,z} \guard \tuple{x,y}\in R, \tuple{y,z}\in U }$}.

\define{Satisfaction of a statement} by a DL standpoint structure $\dlstruct$ (and precisification $\pr$) is then defined as follows\/:
\begin{align*}
    \dlstruct\hspace{-1pt},\hspace{-1pt}\pr & \models C\dlsub D                                                                                 & \iffdef & \interpretgp{C}\subseteq\interpretgp{D}                         \\[-0.2ex]
    \dlstruct\hspace{-1pt},\hspace{-1pt}\pr & \models \rho\dlsub R                                                                              & \iffdef & \interpretgp{\rho}\subseteq\interpretgp{R}                      \\[-0.2ex]
    \dlstruct\hspace{-1pt},\hspace{-1pt}\pr & \models C(a)                                                                                      & \iffdef & a^{\dlstruct}\in\interpretgp{C}                                 \\[-0.2ex]
    \dlstruct\hspace{-1pt},\hspace{-1pt}\pr & \models R(a,b)                                                                                    & \iffdef & \tuple{a^{\dlstruct},b^{\dlstruct}}\in\interpretgp{R}           \\[-0.0ex]
    \dlstruct\hspace{-1pt},\hspace{-1pt}\pr & \models\neg\axiom                                                                                 & \iffdef & \dlstruct,\pr\not\models\axiom                                  \\[-0.0ex]
    \dlstruct\hspace{-1pt},\hspace{-1pt}\pr & \models\literal_1\mathord{\,\land}\ldots\mathord{\land\,}\literal_n                               & \iffdef & \dlstruct,\pr\models\literal_i \text{ for all } 1\leq i\leq n   \\[-0.0ex]
    \dlstruct                               & \models\standbs \monomial                                                                         & \iffdef & \dlstruct,\pr\models\mu \text{ for each } \pr\in\sigma(\sts)    \\[-0.0ex]
    \dlstruct                               & \models\standds \monomial                                                                         & \iffdef & \dlstruct,\pr\models\mu \text{ for some } \pr\in\sigma(\sts)    \\[-0.0ex]
    \dlstruct                               & \models\sts_1\hspace{-1pt}\mathord{\,\cap}\ldots\mathord{\cap\,}\sts_n\preceq\sts\hspace*{-0.9em} & \iffdef & \sigma(\sts_1)\cap\ldots\cap\sigma(\sts_n)\subseteq\sigma(\sts)
\end{align*}
Finally, $\dlstruct$ is a \define{model} of a $\SELp$ knowledge base $\K$ (written $\dlstruct \models \K$) iff it satisfies every statement in $\K$.
As usual, we call $\K$ \define{satisfiable} iff some $\dlstruct$ with \mbox{$\dlstruct \models \K$} exists.
A $\SELp$ statement $\psi$ is \define{entailed} by $\K$ (written \mbox{$\K \models \psi$}) iff \mbox{$\dlstruct \models \psi$} holds for every model $\dlstruct$ of $\K$.

\begin{shortonly}
    \begin{figure*}[t]
    \begin{align}
        \standds[\monomial]                                                 & \goesto \set{ \stv\preceq\sts,\  \standbv[\monomial] } \label{norm1:diamond-formula}                                                                                         \\
        \standbs[\literal_1\land\ldots\land\literal_n]                      & \goesto \set{ \standbs[\literal_1],\ \ldots,\  \standbs[\literal_n] } \label{norm1:box-conjunction}                                                                          \\
        \standbs[\neg (C\dlsub D)]                                          & \goesto \{ \standbs[A\dlsub C],\  \standbs[A\dland D\dlsub\bot],\  \standbs[\top \dlsub \exists R'.A ] \} \label{norm1:neg-gci}                                              \\
        \standbs[\neg C(a)]                                                 & \goesto \set{ \standbs[A(a)],\  \standbs[A\dland C\dlsub\bot] } \label{norm1:neg-ca}                                                                                         \\
        \standbs[\neg R(a,b)]                                               & \goesto \{ \standbs[A_a(a)],\  \standbs[A_b(b)],\  \standbs[A_a\dland\exists R.A_b\dlsub\bot] \} \label{norm1:neg-ra}                                                        \\
        \standbs[\neg (R_1\circ\ldots\circ R_n\dlsub R)]                    & \goesto \{ \standbs[\top\dlsub\exists R'.A_a],\  \standbs[A_a\dland\exists R.A_b\dlsub\bot],\  \standbs[A_a\dlsub\exists R_1.\cdots\exists R_n.A_b] \} \label{norm1:neg-ria} \\
        \neg(\sts_1\mathord{\,\cap}\ldots\mathord{\cap\,}\sts_n\preceq\sts) & \goesto \set{ \stv\preceq\sts_1,\ldots,\  \stv\preceq\sts_n,\  \stv\cap\sts\preceq\sto }\hspace*{-1em} \label{norm1:neg-sharpening}
    \end{align}
    \vspace{-0ex}
    \caption{Normalisation rules for Phase~1.
        Therein, \mbox{$\stu\in\Stands\cup\set{\sto}$}, the $A, A_a, A_b$ denote fresh concept names, $R'$ a fresh role name, and $\stv$ a fresh standpoint name.
        \label{norm1:rules}
    }
\end{figure*}

\end{shortonly}

\subsubsection{Normalisation}
\begin{shortonly}

    Before we can present our deduction calculus for checking satisfiability of $\SELp$ knowledge bases, we need to introduce an appropriate normal form.
    \begin{definition}[Normal Form of $\SELp$ Knowledge Bases]
        \label{def:normal-form}
        A knowledge base $\K$ is in \define{normal form} iff it only contains statements of the following shapes\/:
        \begin{itemize}
            \item
                  sharpening statements of the form \mbox{$\sts\preceq\sp$} and \mbox{$\sts_1\cap\sts_2\preceq\sp$} for \mbox{$\sts,\sp,\sts_1,\sts_2\in\Stands$},
            \item
                  modalised GCIs of the shape \mbox{$\standbs [ C \sqsubseteq  D ]$}, where
                  \begin{itemize}
                      \item $C$ can be of the form $A$, $\exists R.A$, or \mbox{$A \dland A'$} with \mbox{$A, A'\in\Concepts {\,\cup\,} \{\top\} {\,\cup\,} \{\exists R'.\Self \mid R' \in \Roles\}$}, \mbox{$R\in \Roles$},
                      \item $D$ can be of the form $B$,  $\exists R.B$, $\standdsp B$, or $\standbsp B$ with \mbox{$B{\,\in\,}\Concepts{\,\cup\,}\{\bot\} {\,\cup\,} \{\exists R'.\Self\mid R' \in \Roles\}$}, \mbox{$R\in \Roles$}, \mbox{$\st,\sp\in\Stands$},
                            and
                      \item one of $C,D$ is in \mbox{$\Concepts\cup\set{\top\!,\bot}\cup\set{\exists R.\Self\guard R\in\Roles}$};
                  \end{itemize}
            \item modalised RIAs of the form \mbox{$\standbs [R_1\dlsub R_2]$} and \mbox{$\standbs [R_1\circ R_2\dlsub R_3]$} with \mbox{$R_1,R_2,R_3\in\Roles$};
            \item modalised assertions of the form $\standbs[A(a)]$ or $\standbs[R(a,b)]$ for \mbox{$a,b\in\Individuals$}, \mbox{$A\in\Concepts$}, and \mbox{$R\in\Roles$}.
        \end{itemize}
    \end{definition}
    Note that complex/nested concepts can only occur on one side of a GCI and then must not nest deeper than one level.

    For a given $\SELp$ knowledge base \mbox{$\K$}, we can compute its normal form in two phases.
    In the first phase, we “break down” formulas into modalised axioms, effectively compiling away negation, and in the second phase we “break down” complex concepts occurring within these axioms.

    \subsubsection*{Phase 1: Modalised Axioms}
    We obtain the (outer) normal form of axioms by exhaustively applying the transformation rules depicted in \Cref{norm1:rules}, where ``rule application'' means that the statement on the left-hand side is replaced with the set of statements on the right-hand side. This eliminates statements preceded by diamonds, modalised axiom sets, and negated axioms.

    \subsubsection*{Phase 2: Restricted Concept Terms}
    To obtain the (inner) normal forms of concept terms occurring in GCIs as well as the restricted forms of sharpening statements and role inclusion axioms, we use the rules displayed in \Cref{norm2:rules}.
    \begin{figure*}[t]
    \begin{align}
        \sts_1\cap\ldots\cap\sts_n \preceq \sts      & \goesto \set{ \sts_1\cap\sts_2\preceq \stsp,\  \stsp\cap\sts_3\cap\ldots\cap\sts_n\preceq \sts } \label{norm2:sharpening-intersection}                                     \\
        \sts_1\cap\ldots\cap\sts_n \preceq \sto      & \goesto \{ \standb{\sts_1}[\top\dlsub A_1], \ldots,\  \standb{\sts_n}[\top\dlsub A_n],\  \standball[A_1\dland\ldots\dland A_n\dlsub\bot] \} \label{norm2:sharpening-empty} \\
        \standbs[ R_1\circ\ldots\circ R_n \dlsub R ] & \goesto \{ \standbs[R_1\circ R_2\dlsub R'],\  \standbs[R'\circ R_3\circ\ldots\circ R_n\dlsub R] \} \label{norm2:role-chain}                                                \\
        \standbs [ \bar C(a) ]                       & \goesto \set{ \standbs [ A(a) ],\  \standbs [ A \dlsub \bar C ]} \label{norm2:complex-concept-assertion}                                                                   \\
        \standbs [ C\dlsub \top ]                    & \goesto \emptyset \label{norm2:trivial-top}                                                                                                                                \\
        \standbs [ \bot \dlsub D ]                   & \goesto \emptyset \label{norm2:trivial-bot}                                                                                                                                \\
        \standbs [ B \dlsub \exists R.\bar C ]       & \goesto \set{ \standbs [ B \dlsub \exists R. A ],\  \standbs [ A \dlsub \bar C ]} \label{norm2:existential-right}                                                          \\
        \standbs [ B \dlsub C \dland D ]             & \goesto \set{ \standbs [ B\dlsub A],\  \standbs [ A\dlsub C ],\  \standbs [ A \dlsub D ]} \label{norm2:conjunction-right}                                                  \\
        \standbs [ C \dlsub \standvu \bar D ]        & \goesto \set{ \standbs [ C \dlsub \standvu A ],\  \standbs [ A \dlsub \bar D ] } \label{norm2:modal-right}                                                                 \\
        \standbs [ \exists R.\bar C \dlsub D ]       & \goesto \set{ \standbs [\bar C \dlsub A ],\  \standbs [ \exists R. A \dlsub D ] } \label{norm2:existential-left}                                                           \\
        \standbs [ \bar C \dland D \dlsub E ]        & \goesto \set{ \standbs [\bar C \dlsub A ],\  \standbs [ A \dland D \dlsub E ] } \label{norm2:conjunction-left}                                                             \\
        \standbs [ \standdu C \dlsub D ]             & \goesto \set{ \standbu [ C \dlsub \standball A ],\  \standbs [ A \dlsub D ] } \label{norm2:diamond-left}                                                                   \\
        \standbs [ \standbu C \dlsub D ]             & \goesto \{ \stvo\preceq\stu,\  \stvi\preceq\stu,\  \standbu[C\dlsub A],\  \standbs [ \standdvo A \dland \standdvi A \dlsub D ] \} \label{norm2:box-left}
    \end{align}
    \caption{Normalisation rules for Phase 2.  Therein, \mbox{$\stu\in\Stands$}, $\bar C$ and $\bar D$ stand for complex concept terms not contained in $\Concepts \cup \set{\top,\bot} \cup \set{\exists R.\Self \guard R\in\Roles }$, whereas each occurrence of $A$ (possibly with subscript) on a right-hand side denotes the introduction of a fresh concept name;
        each occurrence of $R'$ on a right-hand side denotes the introduction of a fresh role name;
        likewise, $\stv$, $\stvo$, and $\stvi$ denote the introduction of a fresh standpoint name.
        Rule (\ref{norm2:conjunction-left}) is applied modulo commutativity of $\sqcap$.     \label{norm2:rules}
    }
\end{figure*}
    The first three transformation rules are novel, the others were already proposed and formally justified in our earlier work \cite{ourijcaisubmission}.

    \medskip

    A careful analysis yields that the overall transformation (Phase~1 + Phase~2) has the desired semantic and computational properties.

    \begin{lemma}
        \label{lem:normalisation}
        Any $\SELp$ knowledge base $\K$ can be transformed into a $\SELp$ knowledge base $\K'$ in normal form such that:
        \begin{itemize}
            \item $\K'$ is a $\SELp$-conservative extension of $\K$,
            \item the size of $\K'$ is at most linear in the size of $\K$, and
            \item the transformation can be computed in $\PTime$.
        \end{itemize}
    \end{lemma}

    In particular, $\K'$ being a $\SEL$-conservative extension of $\K$ means that $\K$ and $\K'$ are equisatisfiable.
\end{shortonly}

\begin{longonly}
    We will next show \Cref{lem:normalisation}.
    By introducing new standpoint, concept, and role names, any knowledge base $\K$ can be turned into a normalised knowledge base $\K'$ that is a conservative extension of $\K$, i.e., every model of $\K'$ is also a model of $\K$, and every model of $\K$ can be extended to a model of $\K'$ by appropriately choosing the interpretations of the additional standpoint, concept, and role names.

    To show that this transformation can be done in polynomial time, yielding a normalised KB $\K'$ whose size is linear in the size of $\K$, we next define the size $\size{\K}$ of a knowledge base $\K$ roughly as the number of symbols needed to write down $\K$, and define it formally as follows.
    \begin{definition}
        \label{def:size}
        Let $\K$ be a $\SELp$ knowledge base.
        The \define{size} of $\K$, denoted $\size{\K}$, and its various constituents is defined inductively as follows\/:
        \begin{align*}
            \size{\sts_1\cap\ldots\cap\sts_n\preceq\sts} & \eqdef n+1                                 &  & \text{ for } \sts\in\Stands\cup\set{\sto} \\
            \size\top                                    & \eqdef 1                                                                                  \\
            \size\bot                                    & \eqdef 1                                                                                  \\
            \size A                                      & \eqdef 1                                   &  & \text{ for } A\in\Concepts                \\
            \size{C\dland D}                             & \eqdef 1 + \size{C} + \size{D}                                                            \\
            \size{\exists R. C}                          & \eqdef 1 + \size C                                                                        \\
            \size{\exists R. \Self}                      & \eqdef 1                                                                                  \\
            \size{\standvs C}                            & \eqdef 1 + \size C                                                                        \\
            \size{C\dlsub D}                             & \eqdef 1 + \size{C} + \size{D}                                                            \\
            \size{R_1\circ\ldots\circ R_n\dlsub R}       & \eqdef n+1                                                                                \\
            \size{C(a)}                                  & \eqdef 1+\size{C}                                                                         \\
            \size{R(a_1,a_2)}                            & \eqdef 3                                                                                  \\
            \size{\neg\axiom}                            & \eqdef 1 + \size{\axiom}                                                                  \\
            \size{\literal_1\land\ldots\land\literal_n}  & \eqdef \sum_{1\leq i\leq n}\size{\literal}                                                \\
            \size{\standvs[\monomial]}                   & \eqdef 1 + \size{\monomial}                                                               \\
            \size{\K}                                    & \eqdef \sum_{\psi\in\K}\size{\psi}
        \end{align*}
    \end{definition}
    The subsequent proof is, for the most parts, an extension of our previous proof for a fragment of the language~\cite{ourijcaisubmission} where in this work we added the cases for the new constructors and phase~1 normalisation rules.
    We reproduce the whole proof here for coherence and reference.
\end{longonly}

\begin{longproof}[Proof of \Cref{lem:normalisation}]
    Let \mbox{$\K$} be a $\SELp$ knowledge base.
    The statements in $\K$ can be converted into normal form by exhaustively applying the replacement rules shown on page~\pageref{norm2:rules}.
    In what follows, denote by $\K'$ the result of exhaustive rule application to $\K$.
    To prove the lemma, we proceed to show the following:
    \begin{enumerate}
        \item Polynomial runtime and linear output size: Application of the normalisation rules terminates after at most a polynomial (in the size of $\K$) number of steps, and the size of the resulting KB $\K'$ is at most linear in the size of $\K$.
        \item Syntactic Correctness: $\K'$ is in normal form according to \Cref{def:normal-form}.
        \item Semantic Correctness:
              $\K'$ is a conservative extension of $\K$, more specifically:
              \begin{description}
                  \item[(a)] For every model $\dlstruct$ of $\K$ there exists a DL standpoint structure $\dlstruct'$ that extends $\dlstruct$ (agrees with $\dlstruct$ on the vocabulary of $\K$) and that is a model of $\K'$.
                  \item[(b)] Every model $\dlstruct'$ of $\K'$ is also a model of $\K$.
              \end{description}
    \end{enumerate}
    \begin{enumerate}
        \item We first show that overall normalisation must terminate after at most $\size{\K}$ normalisation rule applications.
              The proof plan is as follows:
              For phase 1, we observe that application of a phase~1 rule~$(n)$ produces statements to which rule~$(n)$ is not applicable again.
              Rule~(\ref{norm1:box-conjunction}) incurs a linear blowup in size for the monomial in question (a constant size increase for each literal in the monomial);
              all other phase~1 rules induce only a constant size increase per application, so the overall size increase of normalisation phase~1 is linear.

              For phase 2, we concentrate on sets of GCIs (as they are the only places where nested concept terms can occur due to rule~(\ref{norm2:complex-concept-assertion})) and denote $\K$ by $\T$ (for TBox) if only GCIs are involved, and proceed thus:
              We define the multiset \mbox{$\CT:\SFCT\to\N$} that contains one copy for each occurence of a concept term occurring in $\T$.
              We observe that for each complex concept $\bar C$ causing some GCI not to be in normal form, a constant number of normalisation rule applications can be used to strictly decrease the cardinality of $\bar C$ in $\CT$.
              Together with the fact that \mbox{$\sum_{D\in\SFCT}\CT(D)\leq\size{\T}$}, the claim then follows.

              Define the set $\SFCT$ as the least set such that:
              \begin{itemize}
                  \item If $\standbs[C\dlsub D]\in\T$, then $C,D\in\SFCT$;
                  \item if $C\in\SFCT$, then for all subconcepts $C'$ of $C$ (written $C'\in\sub(C)$) we have $C'\in\SFCT$.
              \end{itemize}
              Now for any TBox $\T'$ (typically obtained from $\T$ by applying zero or more normalisation rules), the multiset \mbox{$\CTp\colon\SFCT\to\N$} is then as follows:
              \begin{align*}
                  C       & \mapsto \sum_{\standbs[D\dlsub E]\in\T'}\left( c(C,D) + c(C,E) \right)
                  \intertext{where we define the concept-counting function $c\colon\SFCT\times\SFCT\to\N$ inductively: }
                  c(C, D) & \eqdef
                  \begin{cases}
                      1                 & \text{ if } C=D,                                                          \\
                      c(C,E)            & \text{ if } C\neq D \text{ and } [D=\exists R.E \text{ or } D=\standvs E] \\
                      c(C,E_1)+c(C,E_2) & \text{ if } C\neq D \text{ and } D=E_1\dland E_2                          \\
                      0                 & \text{ if } C\neq D \text{ and } D\in\Concepts\cup\set{\top,\bot}
                  \end{cases}
              \end{align*}
              For example if \mbox{$D=\top \dland \standds(\top\dland \exists R.(\top \dland \top))$}, then \mbox{$c(\top, D)=4$} while \mbox{$c(\top\dland \top,D)=1$}.

              We next relate the overall cardinality (sum of number of occurrences) of $\CT$ to the size $\size{\T}$ on the basis that both can be represented as disjoint unions and we can therefore sum up the individual cardinalities.
              \begin{numberclaim}
                  \label{claim:norm2:size-multi}
                  We have $\sum_{C\in\SFCT}\CT(C)\leq\size{\T}$.
                  \begin{claimproof}
                      As $\size{\T}=\sum_{\tau\in\T}\size{\tau}$ and
                      \[ \sum_{C\in\SFCT}\CT(C)=\sum_{C\in\SFCT}\sum_{\standbs[D\dlsub E]\in\T}\left( c(C,D) + c(C,E) \right) = \sum_{\standbs[D\dlsub E]\in\T}\sum_{C\in\SFCT}\left( c(C,D) + c(C,E) \right)
                      \]
                      it suffices to look at a single $\standbs[D\dlsub E]\in\T$ and show
                      \[
                          \sum_{C\in\SFCT}\left( c(C,D) + c(C,E) \right) \leq \size{\standbs[D\dlsub E]}
                      \]
                      which by
                      \[
                          \sum_{C\in\SFCT}\left( c(C,D) + c(C,E) \right) = \sum_{D'\in\sub(D)}c(D',D) + \sum_{E'\in\sub(E)}c(E',E)
                      \]
                      develops into
                      \[
                          \sum_{D'\in\sub(D)}c(D',D) + \sum_{E'\in\sub(E)}c(E',E) \leq \size{\standbs[D\dlsub E]} = 1 + 1 + \size{D} + \size{E}
                      \]
                      for which it suffices to show that for any concept \mbox{$C\in\SFCT$}, we have \mbox{$\sum_{C'\in\sub(C)}c(C',C)\leq\size{C}$}, which we show by induction.
                      \begin{itemize}
                          \item The base case is clear, as $\sum_{C'\in\sub(C)}c(C',C)=c(C,C)=1\leq 1=\size{C}$ for any $C\in\Concepts\cup\set{\top,\bot}$.
                          \item $C=\exists R.D$:
                                \begin{align*}
                                     & \pheq \sum_{C'\in\sub(\exists R.D)}c(C',\exists R.D)                 \\
                                     & = c(\exists R.D,\exists R.D) + \sum_{C'\in\sub(D)}c(C',D)            \\
                                     & \stackrel{\text{(IH)}}{\leq} = c(\exists R.D,\exists R.D) + \size{D} \\
                                     & = 1 + \size{D}                                                       \\
                                     & = \size{\exists R.D}
                                \end{align*}
                          \item $C=\standvs D$: Similar.
                          \item $C=C_1\dland C_2$:
                                \begin{align*}
                                     & \pheq \sum_{C'\in\sub(C_1\dland C_2)}c(C',C_1\dland C_2)                                           \\
                                     & = c(C_1\dland C_2,C_1\dland C_2) + \sum_{C'\in\sub(C_1)}c(C',C_1) + \sum_{C'\in\sub(C_2)}c(C',C_2) \\
                                     & \stackrel{\text{(IH)}}{\leq} c(C_1\dland C_2,C_1\dland C_2) + \size{C_1} + \size{C_2}              \\
                                     & = 1 + \size{C_1} + \size{C_2}                                                                      \\
                                     & = \size{C_1\dland C_2}
                                \end{align*}
                      \end{itemize}
                      This concludes the proof of \Cref{claim:norm2:size-multi}.
                  \end{claimproof}
              \end{numberclaim}
              To prove an overall linear number of rule applications, we next show that for each complex concept $\bar C$ whose occurence in a GCI $\tau\in\T'$ causes $\tau$ not to be in normal form, there is a constant number of rule applications (that is, constant for all TBoxes) such that after rule application, the number of overall occurrences of $\bar C$ has strictly decreased and additionally, any intermediate complex concepts introduced by the rule have been normalised in turn.
              To this end, we need to define some more notions.
              Let $\T$ be a TBox and $\T'$ be obtained by application of an arbitrary number of phase~2 normalisation rules.
              We say that a complex concept $\bar C$ occurring in a GCI \mbox{$\tau\in\T'$} is a \define{culprit for $\tau$} iff $\tau$ is not in normal form because of $\bar C$;
              a culprit $\bar C$ is \define{top-level} for \mbox{$\tau=\standbs[D\dlsub E]$} iff \mbox{$\bar C=D$} or \mbox{$\bar C=E$}.
              We say that $\T'$ is \define{faithful to $\T$} iff every culprit occurring in $\T'$ already occurs in $\T$.
              In the proof below, we show that although normalisation rules sometimes introduce new culprits, those culprits will not lead to problems because each one can only cause a constant overhead.
              \begin{numberclaim}
                  \label{claim:norm2:step-decrease}
                  Let $\T$ be a TBox and let $\T'$ be obtained from $\T$ by applying any number of rules from (\ref{norm2:existential-right})--(\ref{norm2:box-left}).
                  Let $\bar C$ be a top-level culprit for \mbox{$\tau\in\T'$}.
                  Then there is a constant number of rule applications leading to a TBox $\T''$ that is faithful to $\T$ and where \mbox{$\CTp(\bar C)>\CTpp(\bar C)$}.
                  \begin{claimproof}
                      We do a case distinction on the occurrence (left-hand vs.\ right-hand side) and form of $\bar C$.
                      In every case, we will explicitly give the (constantly many) rules to apply to decrease the cardinality of $\bar C$.
                      In most cases, it is easy to see that only one rule is needed, we show this exemplarily for one case and then concentrate on the two non-trivial cases.
                      \begin{itemize}
                          \item $\bar C=\exists R.\bar D$ occurs on the right-hand side:
                                Then we apply rule~(\ref{norm2:existential-right}), removing one occurrence of $\bar C$.
                                The resulting $\T''$ is faithful because the only newly introduced concept terms are $\exists R.A$ on a right-hand side and $A$ on a left-hand side, and neither of these is a culprit.
                          \item $\bar C=D_1\dland D_2$ or $\bar C=\standvs D$ and $\bar C$ occurs on the right-hand side: Similar.
                          \item $\bar C=\exists R.\bar D$ or $\bar C=\standds D$ and $\bar C$ occurs on the left-hand side: Exercise.
                          \item $\bar C=D_1\dland D_2$ occurs on the left-hand side:
                                Denote $\tau=\standbs[\bar C\dlsub E]=\standbs[D_1\dland D_2\dlsub E]$.
                                We apply rule~(\ref{norm2:conjunction-left}) and obtain $\standbs[D_1\dlsub A]$ and $\standbs[A\dland D_2\dlsub E]$, thus removing once occurence of $\bar C$.
                                The latter rule contains the (new) culprit $A\dland D_2$ to which we can apply the same rule again (modulo commutativity) to obtain $\standbs[D_2\dlsub A']$ and $\standbs[A\dland A'\dlsub E]$.
                                The only remaining (potential) culprits are $D_1$, $D_2$, and $E$, whence the resulting $\T''$ is faithful to $\T$.
                          \item $\bar C=\standbu D$ occurs on the left-hand side:
                                Denote $\tau=\standbs[\standbu D\dlsub E]$.
                                We apply rules (to underlined GCIs) as follows\/:
                                \begin{align*}
                                     & \mathrel{\phantom{\leadsto}} \set{ \underline{\standbs[\standbu D\dlsub E ]} }                                                                                                                                                                            \\
                                     & \stackrel{\text{(\ref{norm2:box-left})}}{\leadsto} \set{ \standbu[D\dlsub A_1], \underline{\standbs[\standdvo A_1\dland\standdvi A_1\dlsub E]} }                                                                                                          \\
                                     & \stackrel{\text{(\ref{norm2:conjunction-left})}}{\leadsto} \set{ \standbu[D\dlsub A_1], \standbs[\standdvo A_1\dlsub A_2], \underline{\standbs[A_2\dland\standdvi A_1\dlsub E]} }                                                                         \\
                                     & \stackrel{\text{(\ref{norm2:conjunction-left})}}{\leadsto} \set{ \standbu[D\dlsub A_1], \underline{\standbs[\standdvo A_1\dlsub A_2]}, \standbs[\standdvi A_1\dlsub A_3], \standbs[A_2\dland A_3\dlsub E] }                                               \\
                                     & \stackrel{\text{(\ref{norm2:diamond-left})}}{\leadsto} \set{ \standbu[D\dlsub A_1], \standb{\stvo}[A_1\dlsub\standball A_4], \standbs[A_4\dlsub A_2], \underline{\standbs[\standdvi A_1\dlsub A_3]}, \standbs[A_2\dland A_3\dlsub E] }                    \\
                                     & \stackrel{\text{(\ref{norm2:diamond-left})}}{\leadsto} \set{ \standbu[D\dlsub A_1], \standb{\stvo}[A_1\dlsub\standball A_4], \standbs[A_4\dlsub A_2], \standb{\stvi}[A_1\dlsub\standball A_5], \standbs[A_5\dlsub A_3], \standbs[A_2\dland A_3\dlsub E] }
                                \end{align*}
                                It is easy to see that $D$ and $E$ are the only remaining (potential) culprits, and that one occurence of $\bar C=\standbu D$ has been removed.
                      \end{itemize}
                  \end{claimproof}
              \end{numberclaim}
              Thus each culprit will be not only be removed eventually, but removing each single occurrence will take only a constant number of steps.
              Therefore, the number of rule applications is linear in $\sum_{C\in\SFCT}\CT(C)$.
              By \Cref{claim:norm2:size-multi}, the number of rule applications is linear in $\size{\T}$, thus linear in $\size{\K}$.

              It remains to show that the overall increase in size is at most linear.
              We do this by showing that:
              \begin{itemize}
                  \item for each single rule that is applied to nested concept terms and can potentially be applied recursively, the size increase caused by its application is constant;
                  \item for each single rule that is applied only once to a statement (i.e., rule~\ref{norm2:sharpening-empty}), the size increase is at most linear.
              \end{itemize}
              Together with the overall linear number of rule applications, it follows that the size of the resulting normalised TBox $\T'$ is at most linear in the size of the original TBox $\T$.
              \begin{description}
                  \item[Rule~(\ref{norm2:sharpening-intersection}):]
                      \begin{align*}
                           & \pheq \size{\sts_1\cap\sts_2\preceq \stsp} + \size{\stsp\cap\sts_3\cap\ldots\cap\sts_n\preceq \sts} - \size{\sts_1\cap\ldots\cap\sts_n \preceq \sts} \\
                           & = 3 + n - (n+1)                                                                                                                                      \\
                           & = 2
                      \end{align*}
                  \item[Rule~(\ref{norm2:sharpening-empty}):]
                      \begin{align*}
                           & \pheq \size{\standb{\sts_1}[\top\dlsub A_1]} + \ldots + \size{\standb{\sts_n}[\top\dlsub A_n]} + \size{\standball[A_1\dland\ldots\dland A_n\dlsub\bot]} - \size{\sts_1\cap\ldots\cap\sts_n \preceq \sto} \\
                           & = 4n + (n+3) - (n+1)                                                                                                                                                                                     \\
                           & = 4n + 2
                      \end{align*}
                  \item[Rule~(\ref{norm2:role-chain}):]
                      \begin{align*}
                           & \pheq \size{\standbs[R_1\circ R_2\dlsub R']} + \size{\standbs[R'\circ R_3\circ\ldots\circ R_n\dlsub R]} - \size{\standbs[ R_1\circ\ldots\circ R_n \dlsub R ]} \\
                           & = 4 + (n+1) - (n+2)                                                                                                                                           \\
                           & = 3
                      \end{align*}
                  \item[Rule~(\ref{norm2:complex-concept-assertion}):]
                      \begin{align*}
                           & \pheq \size{\standbs[A(a)]} + \size{\standbs[A\dlsub\bar C]} - \left( \size{\standbs[\bar C(a)]} \right) \\
                           & = 3 + 3 + \size{\bar C} - \left( 1 + \size{\bar C} + 1 \right)                                           \\
                           & 4
                      \end{align*}
                  \item[Rule~(\ref{norm2:trivial-top}):] Clear.
                  \item[Rule~(\ref{norm2:trivial-bot}):] Clear.
                  \item[Rule~(\ref{norm2:existential-right}):]
                      \begin{align*}
                           & \pheq \size{\standbs [ B \dlsub \exists R. A ]} + \size{\standbs [ A \dlsub \bar C ]} - \size{\standbs [ B \dlsub \exists R.\bar C ]} \\
                           & = 1 + 1 + \size{B} + 2 + 1 + 1 + 1 + \size{\bar C} - \left( 1 + 1 + \size{B} + 1 + \size{\bar C} \right)                              \\
                           & = 4                                                                                                                                   \\
                      \end{align*}
                  \item[Rule~(\ref{norm2:conjunction-right}):]
                      \begin{align*}
                           & \pheq \size{\standbs [ B\dlsub A]} + \size{\standbs [ A\dlsub C ]} + \size{\standbs [ A\dlsub D ]} - \size{\standbs [ B \dlsub C \dland D ]} \\
                           & = (1 + 1 + \size{B} + 1) + (1 + 1 + 1 + \size{C}) + (1 + 1 + 1 + \size{D}) - \left(1 + 1 + \size{B} + 1 + \size{C} + \size{D} \right)        \\
                           & = 6
                      \end{align*}
                  \item[Rule~(\ref{norm2:modal-right}):]
                      \begin{align*}
                           & \pheq \size{\standbs [ C \dlsub \standvu A ]} + \size{\standbs [ A \dlsub \bar D ]} - \size{\standbs [ C \dlsub \standvu \bar D ]} \\
                           & = 1 + 1 + \size{C} + 2 + 1 + 1 + 1 + \size{\bar D} - \left( 1 + 1 + \size{C} + 1 + \size{\bar D} \right)                           \\
                           & = 4
                      \end{align*}
                  \item[Rule~(\ref{norm2:existential-left}):]
                      \begin{align*}
                           & \pheq \size{\standbs [\bar C \dlsub A ]} + \size{\standbs [ \exists R. A \dlsub D ]} - \size{\exists R.\bar C \dlsub D} \\
                           & = 1 + 1 + \size{\bar C} + 1 + 1 + 1 + 2 + \size{D} - \left( 1 + 1 + \size{\bar C} + \size{D} \right)                    \\
                           & = 5
                      \end{align*}
                  \item[Rule~(\ref{norm2:conjunction-left}):]
                      \begin{align*}
                           & \pheq \size{\standbs [\bar C \dlsub A ]} + \size{\standbs [ A \dland D \dlsub E ]} - \size{\standbs [ \bar C \dland D \dlsub E ]}  \\
                           & = 1 + 1 + \size{\bar C} + 1 + 1 + 1 + 1 + 1 + \size{D} + \size{E} - \left( 1 + 1 + 1 + \size{\bar C} + \size{D} + \size{E} \right) \\
                           & = 4
                      \end{align*}
                  \item[Rule~(\ref{norm2:diamond-left}):]
                      \begin{align*}
                           & \pheq \size{\standbu [ C \dlsub \standball A ]} + \size{\standbs [ A \dlsub D ]} - \size{\standbs [ \standdu C \dlsub D ]} \\
                           & = 1 + 1 + \size{C} + 2 + 1 + 1 + 1 + \size{D} - ( 1 + 1 + 1 + \size{C} + \size{D} )                                        \\
                           & = 4
                      \end{align*}
                  \item[Rule~(\ref{norm2:box-left}):]
                      \begin{align*}
                           & \pheq \size{\stvo\preceq\stu} + \size{\stvi\preceq\stu} + \size{\standbu[C\dlsub A]} + \size{\standbs [ \standdvo A \dland \standdvi A \dlsub D ]} - \size{\standbs [ \standbu C \dlsub D ]} \\
                           & = 2 + 2 + 1 + 1 + \size{C} + 1 + 1 + 1 + 2 + 2 + \size{D} - ( 1 + 1 + 1 + \size{C} + \size{D})                                                                                               \\
                           & = 10
                      \end{align*}
              \end{description}
        \item Assume that $\K'$ is the result of exhaustively applying the normalisation rules (\ref{norm1:diamond-formula})--(\ref{norm2:box-left}) to $\K$.
              The proof is by contradiction.
              Assume that a GCI $\standvs[C\dlsub D]$ is not in normal form;
              we do a case distinction on the possible reasons for this where in each case it will turn out that at least one of the rules is applicable in contradiction to the presumption.
              \begin{itemize}
                  \item $\standvs\neq\standbs$: Then \mbox{$\standvs=\standds$} and Rule~(\ref{norm1:diamond-formula}) is applicable.
                  \item $D$ is of the form $\exists R.E$ with \mbox{$E\notin\BCC\cup\set{\bot}$}: Then Rule~(\ref{norm2:existential-right}) is applicable.
                  \item $D$ is of the form $E_1\dland E_2$: Then Rule~(\ref{norm2:conjunction-right}) is applicable.
                  \item $D$ is of the form $\standvsp E$ with \mbox{$E\notin\BCC\cup\set{\bot}$}: Then Rule~(\ref{norm2:modal-right}) is applicable.
                  \item $D=\top$ or $C=\bot$: Then Rule~(\ref{norm2:trivial-top}) or Rule~(\ref{norm2:trivial-bot}) is applicable.
                  \item $C$ is of the form $\exists R.B$ with \mbox{$B\notin\BCC$}: Then Rule~(\ref{norm2:existential-left}) is applicable.
                  \item $C$ is of the form $B_1\dland B_2$ with \mbox{$\set{B_1,B_2}\not\subseteq\BCC$}: Then Rule~(\ref{norm2:conjunction-left}) is applicable.
                  \item $C$ is of the form $\standdsp B$ with \mbox{$B\notin\BCC$}: Then Rule~(\ref{norm2:diamond-left}) is applicable.
                  \item $C$ is of the form $\standbsp B$ with \mbox{$B\notin\BCC$}: Then Rule~(\ref{norm2:box-left}) is applicable.
              \end{itemize}
        \item We show correctness of each rule.
              Correctness of the overall normalisation process follows by induction on the number of rule applications.
              We slightly adapt the notation to denote by \mbox{$\K'=\tuple{\S,\T',\A}$} the KB that results from application \emph{of a single rule} and do a case distinction on the rules.
              In each case, assume \mbox{$\dlstruct = \tuple{\Dom,\Precs,\gamma,\sigma}$} with \mbox{$\dlstruct\models\K$} and denote \mbox{$\dlstruct'=\tuple{\Dom',\Precs',\gamma',\sigma'}$} in case the components differ from those of $\dlstruct$.
              (In most cases, we only need to show \mbox{$\dlstruct'\models\T'$} and therefore do not mention the other KB components.)
              We start with phase~1 rules.
              \begin{description}
                  \item[Rule~(\ref{norm1:diamond-formula})]
                      (a)~
                      Let \mbox{$\dlstruct\models\standds[\monomial]$}.
                      Then there is a \mbox{$\pr\in\sigma(\sts)$} such that \mbox{$\dlstruct,\pr\models\monomial$}.
                      Define $\dlstruct'$ from $\dlstruct$ by \mbox{$\sigma'(\stu)\eqdef\set{\pr}$}.
                      It follows that \mbox{$\dlstruct'\models\stu\preceq\sts$} and \mbox{$\dlstruct'\models\standbu[\monomial]$}.

                      (b)~
                      Let \mbox{$\dlstruct'\models\K'$} and consider any \mbox{$\pr\in\sigma'(\stu)$} (which exists due to standpoint-non-emptiness).
                      Clearly \mbox{$\dlstruct'\models\S'$} shows \mbox{$\pr\in\sigma'(\stu)\subseteq\sigma'(\sts)$}, whence \mbox{$\dlstruct,\pr\models\monomial$} witnesses that \mbox{$\dlstruct'\models\standds[\monomial]$}.
                  \item[Rule~(\ref{norm1:box-conjunction})]
                      (a)~
                      Let \mbox{$\dlstruct\models\standbs[\literal_1\land\ldots\land\literal_n]$}.
                      Then for every \mbox{$\pr\in\sigma(\sts)$} and \mbox{$1\leq i\leq n$}, we have \mbox{$\dlstruct,\pr\models\literal_i$}.
                      Thus also \mbox{$\dlstruct\models\standbs[\literal_i]$} for all \mbox{$1\leq i\leq n$}.

                      (b)~
                      Let \mbox{$\dlstruct'\models\standbs[\literal_i]$} for all \mbox{$1\leq i\leq n$}.
                      Then for all \mbox{$\pr\in\sigma'(\sts)$} and all \mbox{$1\leq i\leq n$}, we have \mbox{$\dlstruct,\pr\models\literal_i$}.
                      Therefore, for all \mbox{$\pr\in\sigma'(\sts)$} we have \mbox{$\dlstruct,\pr\models\literal_1\land\ldots\land\literal_n$} and thus \mbox{$\dlstruct\models\standbs[\literal_1\land\ldots\land\literal_n]$}.
                  \item[Rule~(\ref{norm1:neg-gci})]
                      (a)~
                      Let \mbox{$\dlstruct\models\standbs[\neg (C\dlsub D)]$}, that is, for all \mbox{$\pr\in\sigma(\sts)$} assume \mbox{$\dlstruct,\pr\not\models C\dlsub D$}.
                      Then, in every \mbox{$\pr\in\sigma(\sts)$} there is some \mbox{$\de_\pr\in\Dom$} such that \mbox{$\de_\pr\in\interpretgp{C}\setminus\interpretgp{D}$}.
                      Define $\dlstruct'$ from $\dlstruct$ as follows:
                      For every \mbox{$\pr\in\sigma(\sts)$} set \mbox{$\interpretpgp{A}\eqdef\set{\de_\pr}$} and \mbox{$\interpretpgp{R'}\eqdef\set{\tuple{\ve,\de_\pr}\guard \ve\in\Dom}$}.
                      Now for every \mbox{$\pr\in\sigma'(\sts)=\sigma(\sts)$}:
                      We have \mbox{$\de_\pr\in\interpretpgp{A}$} and \mbox{$\de_\pr\in\interpretpgp{C}$} whence \mbox{$\dlstruct'\models\standbs[A\dlsub C]$}.
                      We have \mbox{$\de_\pr\notin\interpretpgp{D}$} whence \mbox{$\dlstruct'\models\standbs[A\dland D\dlsub\bot]$}.
                      We have \mbox{$\ve\in\Dom$} implies \mbox{$\tuple{\ve,\de_\pr}\in\interpretpgp{R'}$} and \mbox{$\de_\pr\in\interpretpgp{A}$} whence \mbox{$\dlstruct'\models\standbs[\top\dlsub\exists R'.A]$}.

                      (b)~
                      Let \mbox{$\dlstruct'\models\standbs[A\dlsub C]\land\standbs[A\dland D\dlsub\bot]\land\standbs[\top \dlsub \exists R'.A ]$} and let \mbox{$\pr'\in\sigma'(\sts)$} be arbitrary.
                      Then by \mbox{$\dlstruct'\models\top\dlsub\exists R'.A$} and \mbox{$\Dom'\neq\emptyset$} we get \mbox{$\interpretpgpp{A}\neq\emptyset$}.
                      Thus assume \mbox{$\de'\in\interpretpgpp{A}$}.
                      By \mbox{$\dlstruct'\models\standbs[A\dlsub C]$} we get \mbox{$\de'\in\interpretpgpp{C}$}.
                      By \mbox{$\dlstruct'\models\standbs[A\dland D\dlsub\bot]$} we get \mbox{$\de'\notin\interpretpgpp{D}$}.
                      Therefore, for every \mbox{$\pr'\in\sigma'(\sts)$}, we find a \mbox{$\de'\in\interpretpgpp{C}\setminus\interpretpgpp{D}$};
                      we conclude that \mbox{$\dlstruct'\models\standbs[\neg (C\dlsub D)]$}.
                  \item[Rule~(\ref{norm1:neg-ca})]
                      (a)~
                      Let \mbox{$\dlstruct\models\standbs[\neg C(a)]$}, that is, for every \mbox{$\pr\in\sigma(\sts)$} assume \mbox{$a^\dlstruct\notin\interpretgp{C}$}.
                      Define $\dlstruct'$ from $\dlstruct$ as follows:
                      For every \mbox{$\pr\in\sigma(\sts)$}, set \mbox{$\interpretpgp{A}\eqdef\set{a^\dlstruct}$}.
                      Then clearly \mbox{$\dlstruct'\models\standbs[A(a)]$}, and due to the above also \mbox{$\dlstruct'\models\standbs[A\dland C\dlsub\bot]$}.

                      (b)~
                      Let \mbox{$\dlstruct'\models\standbs[A\dland C\dlsub\bot]\land\standbs[A(a)]$}.
                      Then in every \mbox{$\pr'\in\sigma'(\sts)$}, we have \mbox{$a^{\dlstruct'}\in\interpretpgpp{A}$} and \mbox{$\interpretpgpp{A}\cap\interpretpgpp{C}=\emptyset$}, that is, \mbox{$a^{\dlstruct'}\notin\interpretpgpp{C}$}.
                  \item[Rule~(\ref{norm1:neg-ra})]
                      (a)~
                      Let \mbox{$\dlstruct\models\standbs[\neg R(a,b)]$}, that is, for every \mbox{$\pr\in\sigma(\sts)$} assume \mbox{$\dlstruct,\pr\not\models R(a,b)$} (i.e.\ \mbox{$\tuple{a^\dlstruct,b^\dlstruct}\notin\interpretgp{R}$}).
                      Define $\dlstruct'$ from $\dlstruct$ as follows:
                      For every \mbox{$\pr\in\sigma(\sts)$}, set \mbox{$\interpretpgp{A_a}\eqdef\set{a^\dlstruct}$} and \mbox{$\interpretpgp{A_b}\eqdef\set{b^\dlstruct}$}.
                      We then have \mbox{$\dlstruct'\models\standbs[A_a(a)]$}, \mbox{$\dlstruct'\models\standbs[A_b(b)]$}, and \mbox{$a^\dlstruct\notin\interpretpgp{(\exists R.A_b)}$}.

                      (b)~
                      Let \mbox{$\dlstruct'\models\standbs[A_b(b)]\land\standbs[A_a\dland\exists R.A_b\dlsub\bot]\land\standbs[A_a(a)]$}.
                      Then for every \mbox{$\pr'\in\sigma'(\sts)$}:
                      \mbox{$b^{\dlstruct'}\in\interpretpgpp{A_b}$}, \mbox{$a^{\dlstruct'}\in\interpretpgpp{A_a}$}, and for all \mbox{$\de\in\interpretpgpp{A_a}$} we find \mbox{$\de\notin\interpretpgpp{(\exists R.A_b)}$}, thus in particular \mbox{$a^{\dlstruct'}\notin\interpretpgpp{(\exists R.A_b)}$} and \mbox{$\tuple{a^{\dlstruct'},b^{\dlstruct'}}\notin\interpretpgpp{R}$}.
                  \item[Rule~(\ref{norm1:neg-ria})]
                      (a)~
                      Let \mbox{$\dlstruct\models\standbs[\neg (R_1\circ\ldots\circ R_n\dlsub R)]$}, that is, in every \mbox{$\pr\in\sigma(\sts)$} there exist \mbox{$\de_0,\de_1,\ldots,\de_n\in\Dom$} such that for all \mbox{$1\leq i\leq n$}, \mbox{$\tuple{\de_{i-1},\de_i}\in\interpretgp{R_i}$}, but \mbox{$\tuple{\de_0,\de_n}\notin\interpretgp{R}$}.
                      Define $\dlstruct'$ from $\dlstruct$ as follows:
                      For every \mbox{$\pr\in\sigma(\sts)$}, set \mbox{$\interpretpgp{A_a}\eqdef\set{\de_0}$}, \mbox{$\interpretpgp{A_b}\eqdef\set{\de_n}$}, and \mbox{$\interpretpgp{R'}\eqdef\set{\tuple{\ve,\de_0}\guard\ve\in\Dom}$}.
                      It follows by construction that \mbox{$\dlstruct'\models\standbs[\top\dlsub\exists R'.A_a]$}, and \mbox{$a^{\dlstruct'}\notin\interpretpgp{(\exists R.A_b)}$} whence \mbox{$\dlstruct\models\standbs[A_a\dland\exists R.A_b\dlsub\bot]$}.
                      We can also show for each \mbox{$\pr\in\sigma'(\sts)=\sigma(\sts)$} and for all \mbox{$1\leq i\leq n$} (by induction on $i$), that we have \mbox{$\de_{i-1}\in\interpretpgp{(\exists R_i.\cdots\exists R_n.A_b)}$}, whence in particular \mbox{$\dlstruct'\models\standbs[A_a\dlsub\exists R_1.\cdots\exists R_n.A_b]$}.

                      (b)~
                      Let \mbox{$\dlstruct'\models\standbs[\top\dlsub\exists R'.A_a]\land\standbs[A_a\dland\exists R.A_b\dlsub\bot]\land\standbs[A_a\dlsub\exists R_1.\cdots\exists R_n.A_b]$} and consider any \mbox{$\pr\in\sigma'(\sts)$}.
                      From \mbox{$\Dom'\neq\emptyset$} and \mbox{$\dlstruct'\models\standbs[\top\dlsub\exists R'.A_a]$} we get \mbox{$\interpretpgp{A_a}\neq\emptyset$}, say \mbox{$\de_0\in\interpretpgp{A_a}$}.
                      Therefore, from \mbox{$\dlstruct'\models\standbs[A_a\dlsub\exists R_1.\cdots\exists R_n.A_b]$} we get that there exist \mbox{$\de_1,\ldots,\de_n$} such that \mbox{$\tuple{\de_0,\de_1}\in\interpretpgp{R_1}$}, \ldots, and \mbox{$\tuple{\de_{n-1},\de_n}\in\interpretpgp{R_n}$} where \mbox{$\de_n\in\interpretpgp{A_b}$}.
                      By \mbox{$\dlstruct'\models\standbs[A_a\dland\exists R.A_b\dlsub\bot]$} we get that \mbox{$\tuple{\de_0,\de_n}\notin\interpretpgp{R}$}.
                      Therefore, for any \mbox{$\pr\in\sigma'(\sts)$} we find $\tuple{\de_0,\de_n}\in\interpretpgp{(R_1\circ\ldots\circ R_n)}\setminus\interpretpgp{R}$.
                  \item[Rule~(\ref{norm1:neg-sharpening})]
                      (a)~
                      Let \mbox{$\dlstruct\models\neg(\sts_1\mathord{\,\cap}\ldots\mathord{\cap\,}\sts_n\preceq\sts)$}.
                      Then there is a \mbox{$\pr^*\in\sigma(\sts_1)\cap\ldots\cap\sigma(\sts_n)\setminus\sigma(\sts)$}.
                      Construct $\dlstruct'$ from $\dlstruct$ by setting \mbox{$\sigma'(\stu)\eqdef\set{\pr^*}$}.
                      It is clear that \mbox{$\sigma'(\stu)\subseteq\sigma'(\sts_i)$} for all \mbox{$1\leq i\leq n$};
                      furthermore \mbox{$\pr^*\notin\sigma(\sts)$} shows that \mbox{$\dlstruct'\models\stu\cap\sts\preceq\sto$}.

                      (b)~
                      Let \mbox{$\dlstruct'\models\stu\preceq\sts_1\land\ldots\land\stu\preceq\sts_n\land\stu\cap\sts\preceq\sto$}.
                      By non-emptiness of \mbox{$\sigma'(\stu)$}, there is some \mbox{$\pr_*\in\sigma'(\stu)$}.
                      Clearly \mbox{$\pr_*\in\sigma'(\sts_i)$} for all \mbox{$1\leq i\leq n$}, whence \mbox{$\pr_*\in\sigma'(\sts_1)\cap\ldots\cap\sigma'(\sts_n)$}.
                      By \mbox{$\dlstruct'\models\stu\cap\sts\preceq\sto$} we furthermore get \mbox{$\pr_*\notin\sigma'(\sts)$}.
              \end{description}
              We conclude with correctness of phase~2 rules.
              \begin{description}
                  \item[Rule~(\ref{norm2:sharpening-intersection}):]
                      (a)~
                      Straightforward, we interpret $\sigma'(\sts')\eqdef(\sigma(\sts_1)\cap\sigma(\sts_2))\cup\set{\pr_\sts}$ for some arbitrary $\pr_\sts\in\sigma(\sts)$.

                      (b)~Clear.
                  \item[Rule~(\ref{norm2:sharpening-empty}):]
                      (a)~
                      Let $\dlstruct\models\sts_1\cap\ldots\cap\sts_n\preceq\sto$.
                      It is clear that in $\dlstruct'$ we must define, for all $1\leq i\leq n$ and $\pr\in\sigma(\sts_i)$, $\interpretpgp{A_i}\eqdef\Dom$, to satisfy the axioms of the form $\standb{\sts_i}[\top\dlsub A_i]$.
                      Assume for the sake of obtaining a contradiction that $\dlstruct'\not\models\standball[A_1\dland\ldots\dland A_n\dlsub\bot]$.
                      Then there is a $\pr\in\Precs$ and some $\de\in\Dom'=\Dom$ such that $\de\in\interpretpgp{A_i}$ for all $1\leq i\leq n$.
                      But by construction this means that $\pr\in\sigma(\sts_1)\cap\ldots\cap\sigma(\sts_n)$, a contradiction to the presumption.
                      Thus $\dlstruct'\models\standball[A_1\dland\ldots\dland A_n\dlsub\bot]$.

                      (b)~
                      Let $\dlstruct'\models\standb{\sts_1}[\top\dlsub A_1]\land\ldots\land\standb{\sts_n}[\top\dlsub A_n]\land\standball[A_1\dland\ldots\dland A_n\dlsub\bot]$.
                      The proof works as in direction (a).
                  \item[Rule~(\ref{norm2:role-chain}):]
                      (a)~
                      Straightforward, we interpret $\interpretpgp{R'}\eqdef\interpretgp{R_1}\circ\interpretgp{R_2}$.

                      (b)~
                      Clear.
                  \item[Rule~(\ref{norm2:complex-concept-assertion}):]
                      (a)~
                      Let $\dlstruct\models\standbs[\bar C(a)]$.
                      In $\dlstruct'$ we define $\interpretpgp{A}\eqdef\interpretgp{\bar C}$ and $\dlstruct'\models\standbs[A(a)]\land\standbs[A\dlsub\bar C]$ follows.

                      (b)~
                      Clear.
                  \item[Rules~(\ref{norm2:trivial-top})~and~(\ref{norm2:trivial-bot}):] Clear, as any DL interpretation satisfies \mbox{$C\dlsub\top$} and \mbox{$\bot\dlsub D$}.
                  \item[Rule~(\ref{norm2:existential-right}):]
                      (a)~
                      Let \mbox{$\dlstruct\models\standbs[B\dlsub\exists R.\bar C]$}.
                      Define $\dlstruct'$ from $\dlstruct$ as follows:
                      For any \mbox{$\pr\in\sigma(\sts)$}, set \mbox{$\interpretpgp{A}\eqdef\interpretgp{\bar C}$}.
                      It follows by definition that \mbox{$\dlstruct'\models\standbs[B\dlsub\exists R.A]$} and \mbox{$\dlstruct'\models\standbs[A\dlsub\bar C]$}.

                      (b)~
                      Let \mbox{$\dlstruct'\models\T'$} and consider any \mbox{$\pr\in\sigma'(\sts)$} and $\de\in\interpretpgp{B}$.
                      By \mbox{$\dlstruct'\models\standbs[B\dlsub\exists R.A]$} we get \mbox{$\de\in\interpretpgp{(\exists R.A)}$}, whence there is an \mbox{$\ve\in\Dom'$} with \mbox{$\tuple{\de,\ve}\in\interpretpgp{r}$} and \mbox{$\ve\in\interpretpgp{A}$}.
                      By \mbox{$\dlstruct'\models\standbs[A\dlsub\bar C]$}, we get \mbox{$\ve\in\interpretpgp{\bar C}$} and ultimately \mbox{$\de\in\interpretpgp{(\exists R.\bar C)}$}.
                  \item[Rule~(\ref{norm2:conjunction-right}):]
                      (a)~
                      Let \mbox{$\dlstruct\models\standbs[B\dlsub C\dland D]$} and define $\dlstruct'$ from $\dlstruct$ by setting, for each \mbox{$\pr\in\sigma(\sts)$}, \mbox{$\interpretpgp{A}\eqdef\interpretgp{B}$}.
                      Then by definition we get \mbox{$\dlstruct'\models\standbs[B\dlsub A]$}, as well as \mbox{$\dlstruct'\models\standbs[A\dlsub C]$} and \mbox{$\dlstruct'\models\standbs[A\dlsub D]$}.

                      (b)~
                      Let $\dlstruct'\models\T'$ and consider any $\pr\in\sigma'(\sts)$ and $\de\in\interpretpgp{B}$.
                      By $\T'\models\standbs[B\dlsub A]$ we get $\de\in\interpretpgp{A}$.
                      By $\T'\models\standbs[A\dlsub C]$ and $\T'\models\standbs[A\dlsub D]$ we get $\de\in\interpretpgp{C}$ and $\de\in\interpretpgp{D}$, respectively.
                      In combination, $\de\in\interpretpgp{(C\dland D)}$.
                  \item[Rule~(\ref{norm2:modal-right}):]
                      (a)~
                      Let \mbox{$\dlstruct\models\standbs[C\dlsub \standvu\bar D]$} and define $\dlstruct'$ from $\dlstruct$ by setting, for each \mbox{$\pr\in\sigma(\sts)$}, \mbox{$\interpretpgp{A}\eqdef\interpretgp{\bar D}$}.
                      It follows directly that \mbox{$\dlstruct'\models\standbs[C\dlsub\standvu A]$} and \mbox{$\dlstruct'\models\standbs[A\dlsub\bar D]$}.

                      (b)~
                      Let $\dlstruct'\models\T'$ and consider any $\pr\in\sigma'(\sts)$ and $\de\in\interpretpgp{C}$.
                      From $\dlstruct'\models\standbs[C\dlsub\standvu A]$ we get $\de\in\interpretpgp{(\standvu A)}$.
                      Thus, for some (all) $\pr'\in\sigma'(\stu)$, we get $\de\in\interpretpgpp{A}$;
                      and in turn, by $\dlstruct'\models\standbs[A\dlsub\bar D]$, for some (all) $\pr'\in\sigma'(\stu)$, we get $\de\in\interpretpgpp{\bar D}$.
                  \item[Rule~(\ref{norm2:existential-left}):]
                      (a)~
                      Let \mbox{$\dlstruct\models\standbs[\exists R.\bar C \dlsub D]$}.
                      Define $\dlstruct'$ from $\dlstruct$ as follows:
                      For any \mbox{$\pr\in\sigma(\sts)$}, set \mbox{$\interpretpgp{A}\eqdef\interpretgp{\bar C}$}.
                      It follows from this definition that \mbox{$\dlstruct'\models\standbs[\bar C\dlsub A]$} and \mbox{$\dlstruct'\models\standbs[\exists R.A\dlsub D]$}.

                      (b)~
                      Let \mbox{$\dlstruct'\models\T'$} and consider \mbox{$\pr\in\sigma'(\sts)$} and any \mbox{$\de\in\interpretpgp{\exists R.\bar C}$}.
                      There is thus a \mbox{$\ve\in\Dom$} with \mbox{$\tuple{\de,\ve}\in\interpretpgp{r}$} and \mbox{$\ve\in\interpretpgp{\bar C}$}.
                      By \mbox{$\dlstruct'\models\standbs[\bar C\dlsub A]$} we get \mbox{$\ve\in\interpretpgp{A}$}, and in turn \mbox{$\de\in\interpretpgp{(\exists R.A)}$}.
                      By \mbox{$\dlstruct'\models\standbs[\exists R.A\dlsub D]$}, we get \mbox{$\de\in\interpretpgp{D}$}.
                  \item[Rule~(\ref{norm2:conjunction-left}):]
                      (a)~
                      Let \mbox{$\dlstruct\models\standbs[\bar C\dland D\dlsub E]$}.
                      Define $\dlstruct'$ from $\dlstruct$ as follows:
                      For any \mbox{$\pr\in\sigma(\sts)$}, set \mbox{$\interpretpgp{A}\eqdef\interpretgp{\bar C}$}.
                      It follows from this definition that \mbox{$\dlstruct'\models\standbs[\bar C\dlsub A]$} and \mbox{$\dlstruct'\models\standbs[A\dland D\dlsub E]$}.

                      (b)~
                      Let \mbox{$\dlstruct'\models\T'$} and consider \mbox{$\pr\in\sigma'(\sts)$} and \mbox{$\de\in\interpretpgp{\bar C} \cap \interpretpgp{D}$}.
                      By \mbox{$\dlstruct'\models\standbs[\bar C\dlsub A]$} we get \mbox{$\de\in\interpretpgp{A}$}, and
                      by \mbox{$\dlstruct'\models\standbs[A\dland D\dlsub E]$} and \mbox{$\de\in\interpretpgp{D}$} we get \mbox{$\de\in\interpretpgp{E}$}.
                  \item[Rule~(\ref{norm2:diamond-left}):]
                      (a)~
                      Let \mbox{$\dlstruct\models\standbs[\standdu C\dlsub D]$} and define $\dlstruct'$ from $\dlstruct$ by setting \mbox{$\interpretpgp{A} \eqdef \interpretgp{(\standdu C)}$} for every \mbox{$\pr\in\Precs$}.
                      \begin{itemize}
                          \item $\dlstruct'\models\standbu [ C \dlsub \standball A ]$:
                                Let \mbox{$\pr\in\sigma'(\stu)=\sigma(\stu)$} and \mbox{$\de\in\interpretpgp{C}=\interpretgp{C}$}.
                                Then by definition, we have \mbox{$\de\in\interpretpgp{A}$} for all \mbox{$\pr\in\Precs$}, whence \mbox{$\de\in\interpretpgp{(\standball A)}$}.
                          \item $\dlstruct'\models\standbs  [ A \dlsub D ]$:
                                Let \mbox{$\pr\in\sigma'(\sts)=\sigma(\sts)$} and \mbox{$\de\in\interpretpgp{A}$}.
                                Then by definition, \mbox{$\de\in\interpretgp{(\standdu C)}$} and by the presumption that \mbox{$\dlstruct\models\standbs[\standdu C\dlsub D]$} we get \mbox{$\de\in\interpretgp{D}=\interpretpgp{D}$}.
                      \end{itemize}
                      (b)~
                      Let \mbox{$\dlstruct'\models\T'$} and consider \mbox{$\pr\in\sigma'(\sts)$} and \mbox{$\de\in\interpretpgp{(\standdu C)}$}.
                      Then there is a \mbox{$\pr'\in\sigma'(\stu)$} such that \mbox{$\de\in\interpretpgpp{C}$}.
                      Thus \mbox{$\de\in\interpretpgpp{(\standball A)}$}, that is, \mbox{$\de\in\bigcap_{\pr''\in\Precs}\interpretpgppp{A}$}.
                      Thus in particular for \mbox{$\pr\in\sigma'(\sts)$}, we get \mbox{$\de\in\interpretpgp{A}$} and by \mbox{$\dlstruct'\models\standbs [ A \dlsub D ]$} we obtain \mbox{$\de\in\interpretpgp{D}$}.

                  \item[Rule~(\ref{norm2:box-left}):]
                      (a)~Let \mbox{$\dlstruct\models\standbs[\standbu C\dlsub D]$}.
                      We define \mbox{$\dlstruct' = \tuple{\Dom', \Precs', \sigma', \gamma'}$} as follows:
                      \begin{itemize}
                          \item $\Delta'$ consists of two copies $\delta'_0, \delta'_1$ of every function \mbox{$\de'\colon\sigma(\mathsf{u}) \times \Delta\to\Delta$}.
                          \item $\Precs'$ consists of all functions \mbox{$\pr'\colon\sigma(\stu) \times \Delta\to\Precs$} plus two extra, distinct copies of the particular function \mbox{$\pr'_\mathrm{diag} \eqdef \set{ (\pr,\de) \mapsto \pr }$}, denoted $\pr'_{\stv_0}$ and $\pr'_{\stv_1}$.
                          \item \mbox{$\sigma'(\stv_b)\eqdef\set{ \pr'_{\stv_b} }$} for each \mbox{$b\in\set{0,1}$} and, for all other $\sts$,
                                $$\sigma'(\sts) \eqdef \set{ \pr' \in \Precs'\guard \mathop{\bigcup_{\pr \in \sigma(\stu),}}_{\delta \in \Dom} \set{\pr'(\pr,\de)} \subseteq \sigma(\sts) }$$
                      \end{itemize}
                      For every \mbox{$\pr'\in\Precs'$}, the DL interpretation \mbox{$\gamma'(\pr')$} over $\Dom'$ is such that
                      \begin{align*}
                          \interpret{a}{\gamma'(\pr')} & \eqdef \set{ (\pr,\de) \mapsto \delta_a}_0 & \text{ for } a\in\Individuals
                      \end{align*}
                      where \mbox{$\de_a \in \Delta$} denotes the domain element for which \mbox{$\de_a = a^{\gamma(\pr)}$} for all \mbox{$\pr \in \Precs$}.\\
                      Further, for \mbox{$\pr'\in\Precs' \setminus \{\pr'_{\stv_0},\pr'_{\stv_1}\}$}, the DL interpretation \mbox{$\gamma'(\pr')$} over $\Dom'$ is such that
                      \begin{align*}
                          \interpret{A}{\gamma'(\pr')}                                         & \eqdef \set{ \de'_b \in\Dom' \guard \de'_b (\pi,\delta)\in\interpret{A}{\gamma(\pr'(\pi,\delta))} \text{ for all } {\pr \in \sigma(\mathsf{u})},\ {\delta \in \Delta}} & \text{ for } A\in\Concepts \\
                          \interpret{R}{\gamma'(\pr')}                                         & \eqdef \set{ \tuple{\de'_b,\ve'_b}\in\Dom'\times\Dom' \guard \tuple{\de'_b(\pi,\delta),\ve'_b(\pi,\delta)}\in\interpret{R}{\gamma(\pr'(\pi,\delta))}
                          \text{ for all } {\pr \in \sigma(\mathsf{u})},\ {\delta \in \Delta}} & \text{ for } R\in\Roles
                      \end{align*}
                      and \mbox{$\gamma'(\prvb)$} assigns as follows:
                      \begin{align*}
                          \interpret{A}{\gamma'(\prvb)}                                                                                                                                                                                  & \eqdef \left( \set{ \de'_b\in\Dom' } \cap \interpret{A}{\gamma'(\pr'_\mathrm{diag})} \right)
                          \cup  \set{ \de'_{1-b}\in\Dom' \guard \bigcup_{{\pr \in \sigma(\mathsf{u}),}\atop{\delta \in \Delta}} \set{ \de'_{1-b}(\pr,\de) } \subseteq{\bigcap_{\pr \in \sigma(\mathsf{u})}}\interpret{A}{\gamma(\pr)}  } & \text{ for } A\in\Concepts                                                                                             \\
                          \interpret{R}{\gamma'(\prvb)}                                                                                                                                                                                  & \eqdef \left( \set{ \tuple{\de'_b,\ve'_b}\in\Dom'\times\Dom' } \cap \interpret{R}{\gamma'(\pr'_\mathrm{diag})} \right) \\
                                                                                                                                                                                                                                         & \quad \cup  \set{ \tuple{\de'_{1-b},\ve'_{1-b}}\in\Dom'\times\Dom' \guard
                          \bigcup_{{\pr \in \sigma(\mathsf{u}),}\atop{\delta \in \Delta}} \set{ \tuple{\de'_{1-b}(\pr,\de), \ve'_{1-b}(\pr,\de)}} \subseteq{\bigcap_{\pr \in \sigma(\mathsf{u})}}\interpret{R}{\gamma(\pr)}  }                                                                                                                                    \\
                                                                                                                                                                                                                                         & \quad \cup  \set{ \tuple{\de'_{1-b},\ve'_b}\in\Dom'\times\Dom' \guard
                          \tuple{\zeta'_b,\ve'_b}\in \interpret{R}{\gamma'(\pr'_\mathrm{diag})} \text{ for some }  \zeta'_b \approx \de'_{1-b}
                          }
                                                                                                                                                                                                                                         & \text{ for } R\in\Roles
                      \end{align*}
                      where \mbox{$\zeta'_b \approx \de'_{1-b} $} holds iff for every \mbox{$\pr \in \sigma(\mathsf{u})$} we have
                      $
                          \textstyle\bigcup_{{\delta \in \Delta}} \{\zeta'_b(\pi,\delta)\}
                          =
                          \textstyle\bigcup_{{\pr_* \in \sigma(\mathsf{u}),}\atop{\delta_* \in \Delta}} \{\de'_{1-b}(\pi_*,\delta_*) \}
                      $.
                      Finally for the fresh concept name $A$ introduced by the rule, define \mbox{$\interpretpgpp{A} \eqdef \interpretpgpp{C}$} for all \mbox{$\pr'\in\sigma'(\stu)$}.

                      We now show that \mbox{$\dlstruct'\models\K'$}.
                      For this, we start out with some useful observations.
                      \begin{numberclaim}
                          \label{claim:norm2:back-and-forth}
                          For every concept $E$, \mbox{$\de'\in\Dom'$}, and \mbox{$\pr'\in\Precs'\setminus\set{\prvo,\prvi}$}\/:
                          \[
                              \de'\in\interpretpgpp{E}
                              \iff
                              \left(
                              \forall\pr\in\sigma(\stu):\forall\de\in\Dom:\de'(\pr,\de)\in\interpret{E}{\gamma(\pr'(\pr,\de))}
                              \right)
                          \]
                          \begin{claimproof}
                              We use structural induction on $E$.
                              \begin{itemize}
                                  \item $E=A\in\BCC$: By definition.
                                  \item $E=E_1\dland E_2$:
                                        \begin{align*}
                                             & \phiff \de'\in\interpretpgpp{(E_1\dland E_2)}                                                                                                                                            \\
                                             & \iff \de'\in\interpretpgpp{E_1}\cap\interpretpgpp{E_2}                                                                                                                                   \\
                                             & \iff \de'\in\interpretpgpp{E_1} \text{ and } \de'\in\interpretpgpp{E_2}                                                                                                                  \\
                                             & \stackrel{\text{(IH)}}{\iff} \left[\forall\pr\in\sigma(\stu):\forall\de\in\Dom:\de'(\pr,\de)\in\interpret{E_1}{\gamma(\pr'(\pr,\de))}\right]
                                            \text{ and } \left[\forall\pr\in\sigma(\stu):\forall\de\in\Dom:\de'(\pr,\de)\in\interpret{E_2}{\gamma(\pr'(\pr,\de))}\right]                                                                \\
                                             & \iff \forall\pr\in\sigma(\stu):\forall\de\in\Dom:\left[\de'(\pr,\de)\in\interpret{E_1}{\gamma(\pr'(\pr,\de))} \text{ and } \de'(\pr,\de)\in\interpret{E_2}{\gamma(\pr'(\pr,\de))}\right] \\
                                             & \iff \forall\pr\in\sigma(\stu):\forall\de\in\Dom:\de'(\pr,\de)\in\interpret{E_1}{\gamma(\pr'(\pr,\de))} \cap \interpret{E_2}{\gamma(\pr'(\pr,\de))}                                      \\
                                             & \iff \forall\pr\in\sigma(\stu):\forall\de\in\Dom:\de'(\pr,\de)\in\interpret{(E_1\dland E_2)}{\gamma(\pr'(\pr,\de))}
                                        \end{align*}
                                  \item $E=\exists R.B$:
                                        \begin{align*}
                                             & \phiff \de'\in\interpretpgpp{\exists R.B}                                                                                                                                                                                                                                                    \\
                                             & \iff \exists\ve'\in\Dom': \left[ \tuple{\de',\ve'}\in\interpretpgpp{R} \text{ and } \ve'\in\interpretpgpp{B}\right]                                                                                                                                                                          \\
                                             & \stackrel{\text{(IH)}}{\iff} \exists\ve'\in\Dom': \left[ \tuple{\de',\ve'}\in\interpretpgpp{R} \text{ and } \forall\pr\in\sigma(\stu):\forall\de\in\Dom:\ve'(\pr,\de)\in\interpret{B}{\gamma(\pr'(\pr,\de))} \right]                                                                         \\
                                             & \iff \exists\ve'\in\Dom': \left[ \left[ \forall\pr\in\sigma(\stu):\forall\de\in\Dom:\tuple{\de'(\pr,\de),\ve'(\pr,\de)}\in\interpret{R}{\gamma(\pr'(\pr,\de))} \right] \text{ and } \forall\pr\in\sigma(\stu):\forall\de\in\Dom:\ve'(\pr,\de)\in\interpret{B}{\gamma(\pr'(\pr,\de))} \right] \\
                                             & \iff \exists\ve'\in\Dom': \left[ \forall\pr\in\sigma(\stu):\forall\de\in\Dom: \left[ \tuple{\de'(\pr,\de),\ve'(\pr,\de)}\in\interpret{R}{\gamma(\pr'(\pr,\de))} \text{ and } \ve'(\pr,\de)\in\interpret{B}{\gamma(\pr'(\pr,\de))} \right] \right]                                            \\
                                             & \stackrel{\ddagger}{\iff} \forall\pr\in\sigma(\stu):\forall\de\in\Dom:\exists\ve'\in\Dom': \left[ \tuple{\de'(\pr,\de),\ve'(\pr,\de)}\in\interpret{R}{\gamma(\pr'(\pr,\de))} \text{ and } \ve'(\pr,\de)\in\interpret{B}{\gamma(\pr'(\pr,\de))} \right]                                       \\
                                             & \stackrel{\dagger}{\iff} \forall\pr\in\sigma(\stu):\forall\de\in\Dom:\exists\ve\in\Dom: \left[ \tuple{\de'(\pr,\de),\ve}\in\interpret{R}{\gamma(\pr'(\pr,\de))} \text{ and } \ve\in\interpret{B}{\gamma(\pr'(\pr,\de))} \right]                                                              \\
                                             & \iff \forall\pr\in\sigma(\stu):\forall\de\in\Dom:\de'(\pr,\de)\in\interpret{(\exists R.B)}{\gamma(\pr'(\pr,\de))}
                                        \end{align*}
                                        Note that for the ``$\impliedby$'' direction, there is always an \mbox{$\ve'\in\Dom'$} such that for all \mbox{$\pr\in\sigma(\stu)$} and \mbox{$\de\in\Dom$} we find \mbox{$\ve'(\pr,\de)=\ve$} for the appropriate \mbox{$\ve\in\Dom$} (step $\dagger$), and thus this $\ve'$ is in this sense independent of concrete \mbox{$\pr\in\sigma(\stu)$} and \mbox{$\de\in\Dom$} (step $\ddagger$).
                                  \item $E=\standdsp B$:
                                        \begin{align*}
                                             & \phiff \de'\in\interpretpgpp{(\standdsp B)}                                                                                                                   \\
                                             & \iff \exists\pr''\in\sigma'(\stsp):\de'\in\interpretpgppp{B}                                                                                                  \\
                                             & \stackrel{\text{(IH)}}{\iff} \exists\pr''\in\sigma'(\stsp):\forall\pr\in\sigma(\stu):\forall\de\in\Dom: \de'(\pr,\de)\in\interpret{B}{\gamma(\pr''(\pr,\de))} \\
                                             & \iff \forall\pr\in\sigma(\stu):\forall\de\in\Dom:\exists\pr''\in\sigma'(\stsp):\de'(\pr,\de)\in\interpret{B}{\gamma'(\pr''(\pr,\de))}                         \\
                                             & \iff \forall\pr\in\sigma(\stu):\forall\de\in\Dom:\exists\pr_\stsp\in\sigma(\stsp):\de'(\pr,\de)\in\interpret{B}{\gamma(\pr_\stsp)}                            \\
                                             & \iff \forall\pr\in\sigma(\stu):\forall\de\in\Dom:\de'(\pr,\de)\in\interpret{(\standdsp B)}{\gamma(\pr'(\pr,\de))}
                                        \end{align*}
                                        Again, for ``$\impliedby$'', we find a \mbox{$\pr''\in\sigma(\stsp)$} for which always \mbox{$\pr''(\pr,\de)=\pr_\stsp$} as desired.
                                  \item $E=\standbsp B$:
                                        \begin{align*}
                                             & \phiff \de'\in\interpretpgpp{(\standbsp B)}                                                                                                                   \\
                                             & \iff \forall\pr''\in\sigma'(\stsp):\de'\in\interpretpgpp{B}                                                                                                   \\
                                             & \stackrel{\text{(IH)}}{\iff} \forall\pr''\in\sigma'(\stsp):\forall\pr\in\sigma(\stu):\forall\de\in\Dom: \de'(\pr,\de)\in\interpret{B}{\gamma(\pr''(\pr,\de))} \\
                                             & \iff \forall\pr\in\sigma(\stu):\forall\de\in\Dom:\forall\pr''\in\sigma'(\stsp):\de'(\pr,\de)\in\interpret{B}{\gamma'(\pr''(\pr,\de))}                         \\
                                             & \stackrel{\dagger}{\iff} \forall\pr\in\sigma(\stu):\forall\de\in\Dom:\forall\pr_\stsp\in\sigma(\stsp):\de'(\pr,\de)\in\interpret{B}{\gamma(\pr_\stsp)}        \\
                                             & \iff \forall\pr\in\sigma(\stu):\forall\de\in\Dom:\de'(\pr,\de)\in\interpret{(\standbsp B)}{\gamma(\pr'(\pr,\de))}
                                        \end{align*}
                                        The equivalence marked $\dagger$ holds because for every \mbox{$\pr_\stsp\in\sigma(\stsp)$} we can find a \mbox{$\pr''\in\sigma'(\stsp)$} such that \mbox{$\pr''(\pr,\de)=\pr_\stsp$} for all \mbox{$\pr\in\sigma(\stu)$} and \mbox{$\de\in\Dom$} (e.g.\ the constant function \mbox{$\set{(\pr,\de)\mapsto \pr_\stsp}$});
                                        conversely, for every \mbox{$\pr''\in\sigma'(\stsp)$}, \mbox{$\pr\in\sigma(\stu)$}, and \mbox{$\de\in\Dom$}, we have \mbox{$\pr''(\pr,\de)\in\sigma(\stsp)$} by definition.
                              \end{itemize}
                          \end{claimproof}
                      \end{numberclaim}
                      The claim above does not consider the newly introduced precisifications.
                      They are however covered by the next claim.
                      \begin{numberclaim}
                          \label{claim:norm2:back-and-forth:v}
                          For every concept $E$, \mbox{$\de'\in\Dom'$}, and \mbox{$\pr'\in\set{\prvo,\prvi}$}\/:
                          \begin{enumerate}
                              \item\label{claim:norm2:back-and-forth:v:back} $\de'\in\interpretpgpp{E}                                                                                                        \implies \forall\pr\in\sigma(\stu):\forall\de\in\Dom:\de'(\pr,\de)\in\interpret{E}{\gamma(\pr'(\pr,\de))}$
                              \item\label{claim:norm2:back-and-forth:v:forth} $\forall\pr\in\sigma(\stu):\forall\de\in\Dom:\forall\pr_*\in\sigma(\stu): \de'(\pr,\de)\in\interpret{E}{\gamma(\pr'(\pr_*,\de))} \implies \de'\in\interpretpgpp{E}$
                          \end{enumerate}
                          \begin{claimproof}
                              We use structural induction on $E$.
                              \begin{itemize}
                                  \item $E=A\in\BCC$:
                                        Assume \mbox{$\pr'=\prvo$} (the remaining case is symmetric).
                                        For direction (\ref{claim:norm2:back-and-forth:v:back}) we have\/:
                                        \begin{align*}
                                             & \phiff \de'_b\in\interpret{A}{\gamma'(\prvo)} \text{ for some } b\in\set{0,1}                                                                                                                                                                              \\
                                             & \iff \de'_0\in\interpret{A}{\gamma'(\pr'_{\mathrm{diag}})} \text{ or } \mathop{\bigcup_{\pr\in\sigma(\stu),}}_{\de\in\Dom}\set{\de'_1(\pr,\de)}\subseteq\bigcap_{\pr\in\sigma(\stu)}\interpretgp{A}                                                        \\
                                             & \iff \forall\pr\in\sigma(\stu):\forall\de\in\Dom:\de'_0(\pr,\de)\in\interpret{A}{\gamma(\pr'_{\mathrm{diag}}(\pr,\de))} \text{ or } \forall\pr\in\sigma(\stu):\forall\de\in\Dom:\forall\pr_*\in\sigma(\stu):\de'_1(\pr,\de)\in\interpret{A}{\gamma(\pr_*)} \\
                                             & \iff \forall\pr\in\sigma(\stu):\forall\de\in\Dom:\de'_0(\pr,\de)\in\interpret{A}{\gamma(\pr)} \text{ or } \forall\pr\in\sigma(\stu):\forall\de\in\Dom:\forall\pr_*\in\sigma(\stu):\de'_1(\pr,\de)\in\interpret{A}{\gamma(\pr_*)}                           \\
                                             & \implies \forall\pr\in\sigma(\stu):\forall\de\in\Dom:\de'_0(\pr,\de)\in\interpret{A}{\gamma(\pr)} \text{ or } \forall\pr\in\sigma(\stu):\forall\de\in\Dom:\de'_1(\pr,\de)\in\interpret{A}{\gamma(\pr)}                                                     \\
                                             & \iff \forall\pr\in\sigma(\stu):\forall\de\in\Dom:\de'_b(\pr,\de)\in\interpret{A}{\gamma(\pr)} \text{ for some } b\in\set{0,1}                                                                                                                              \\
                                             & \iff \forall\pr\in\sigma(\stu):\forall\de\in\Dom:\de'_b(\pr,\de)\in\interpret{A}{\gamma(\pr'_{\mathrm{diag}}(\pr,\de))} \text{ for some } b\in\set{0,1}                                                                                                    \\
                                             & \iff \forall\pr\in\sigma(\stu):\forall\de\in\Dom:\de'_b(\pr,\de)\in\interpret{A}{\gamma(\prvo(\pr,\de))} \text{ for some } b\in\set{0,1}
                                        \end{align*}
                                        In direction (\ref{claim:norm2:back-and-forth:v:forth}), we have\/:
                                        \begin{align*}
                                             & \phiff \forall\pr\in\sigma(\stu):\forall\de\in\Dom:\forall\pr_*\in\sigma(\stu): \de'_b(\pr,\de)\in\interpret{E}{\gamma(\pr'(\pr_*,\de))} \text{ for some } b\in\set{0,1}                                                                                                             \\
                                             & \iff \forall\pr\in\sigma(\stu):\forall\de\in\Dom:\forall\pr_*\in\sigma(\stu): \de'_0(\pr,\de)\in\interpret{E}{\gamma(\pr'(\pr_*,\de))} \text{ or } \forall\pr\in\sigma(\stu):\forall\de\in\Dom:\forall\pr_*\in\sigma(\stu): \de'_1(\pr,\de)\in\interpret{E}{\gamma(\pr'(\pr_*,\de))} \\
                                             & \implies \forall\pr\in\sigma(\stu):\forall\de\in\Dom:\de'_0(\pr,\de)\in\interpret{A}{\gamma(\pr)} \text{ or } \forall\pr\in\sigma(\stu):\forall\de\in\Dom:\forall\pr_*\in\sigma(\stu):\de'_1(\pr,\de)\in\interpret{A}{\gamma(\pr_*)}
                                        \end{align*}
                                        In the last line, we continue with the equivalences from direction (\ref{claim:norm2:back-and-forth:v:back}) above the line with ``$\implies\!$''.
                                  \item $E=\exists R.B$: We show the claim for \mbox{$\pr'=\prvo$}, as the other case is symmetric.
                                        In direction (\ref{claim:norm2:back-and-forth:v:back}), let \mbox{$\de'_b\in\interpretpgpp{(\exists R.B)}$}.
                                        By definition of the semantics, there is a \mbox{$\ve'_c\in\Dom'$} such that \mbox{$\tuple{\de'_b,\ve'_c}\in\interpretpgpp{R}$} and \mbox{$\ve'_c\in\interpretpgpp{B}$}.
                                        Applying the induction hypothesis to \mbox{$\ve'_c\in\interpretpgpp{B}$} yields that
                                        \begin{gather}
                                            \label{eq:back-and-forth:v:exists:ih}
                                            \forall\pr\in\sigma(\stu):\forall\de\in\Dom:\ve'_c(\pr,\de)\in\interpret{B}{\gamma(\pr'(\pr,\de))}=\interpret{B}{\gamma(\pr)}
                                        \end{gather}
                                        By definition of $\interpretpgpp{R}$, we have three cases for \mbox{$\tuple{\de'_b,\ve'_c}\in\interpretpgpp{R}$}, in each of which we show the claim.
                                        \begin{enumerate}
                                            \item $\tuple{\de'_0,\ve'_0}\in\interpret{R}{\gamma'(\pr'_{\mathrm{diag}})}$.
                                                  That is, \mbox{$\tuple{\de'_0(\pr,\de),\ve'_0(\pr,\de)}\in\interpret{R}{\gamma(\pr)}$} for all \mbox{$\pr\in\sigma(\stu)$} and \mbox{$\de\in\Dom$}.
                                                  Together with \Cref{eq:back-and-forth:v:exists:ih}, we obtain
                                                  \[
                                                      \forall\pr\in\sigma:\forall\de\in\Dom:\left( \tuple{\de'_0(\pr,\de),\ve'_0(\pr,\de)}\in\interpret{R}{\gamma(\pr)} \text{ and } \ve'_c(\pr,\de)\in\interpret{B}{\gamma(\pr)} \right)
                                                  \]
                                                  Now since $\ve'_0$ and $\ve'_1$ are two distinct copies of the same function, this means that
                                                  \[
                                                      \forall\pr\in\sigma:\forall\de\in\Dom:\left( \tuple{\de'_0(\pr,\de),\ve'_0(\pr,\de)}\in\interpret{R}{\gamma(\pr)} \text{ and } \ve'_0(\pr,\de)\in\interpret{B}{\gamma(\pr)} \right)
                                                  \]
                                                  Thus
                                                  \[
                                                      \forall\pr\in\sigma:\forall\de\in\Dom:\exists\ve\in\Dom:\left( \tuple{\de'_0(\pr,\de),\ve}\in\interpret{R}{\gamma(\pr)} \text{ and } \ve\in\interpret{B}{\gamma(\pr)} \right)
                                                  \]
                                                  namely we set \mbox{$\ve\eqdef\ve'(\pr,\de)$} in each case.
                                                  By the definition of the semantics we get
                                                  \[
                                                      \forall\pr\in\sigma:\forall\de\in\Dom:\de'(\pr,\de)\in\interpret{(\exists R.B)}{\gamma(\pr)}
                                                  \]
                                                  which, by \mbox{$\pr'=\prvo=\pr'_{\mathrm{diag}}$} means
                                                  \[
                                                      \forall\pr\in\sigma:\forall\de\in\Dom:\de'(\pr,\de)\in\interpret{(\exists R.B)}{\gamma(\pr'(\pr,\de))}
                                                  \]
                                            \item $\forall\pr\in\sigma(\stu):\forall\de\in\Dom:\forall\pr_*\in\sigma(\stu):\tuple{\de'_1(\pr,\de),\ve'_1(\pr,\de)}\in\interpret{R}{\gamma(\pr_*)}$.
                                                  In combination with \Cref{eq:back-and-forth:v:exists:ih}, we get
                                                  \[
                                                      \forall\pr\in\sigma(\stu):\forall\de\in\Dom:\forall\pr_*\in\sigma(\stu):
                                                      \left(
                                                      \tuple{\de'_1(\pr,\de),\ve'_1(\pr,\de)}\in\interpret{R}{\gamma(\pr_*)}
                                                      \text{ and }
                                                      \ve'_c(\pr,\de)\in\interpret{B}{\gamma(\pr)}
                                                      \right)
                                                  \]
                                                  which, by virtue of $\ve'_0$ and $\ve'_1$ being the same function, means
                                                  \[
                                                      \forall\pr\in\sigma(\stu):\forall\de\in\Dom:\forall\pr_*\in\sigma(\stu):
                                                      \left(
                                                      \tuple{\de'_1(\pr,\de),\ve'_1(\pr,\de)}\in\interpret{R}{\gamma(\pr_*)}
                                                      \text{ and }
                                                      \ve'_1(\pr,\de)\in\interpret{B}{\gamma(\pr)}
                                                      \right)
                                                  \]
                                                  which in particular implies
                                                  \[
                                                      \forall\pr\in\sigma(\stu):\forall\de\in\Dom:
                                                      \left(
                                                      \tuple{\de'_1(\pr,\de),\ve'_1(\pr,\de)}\in\interpret{R}{\gamma(\pr)}
                                                      \text{ and }
                                                      \ve'_1(\pr,\de)\in\interpret{B}{\gamma(\pr)}
                                                      \right)
                                                  \]
                                                  which implies the claim as in the case above.
                                            \item there is some \mbox{$\zeta'_0\in\Dom'$} with \mbox{$\zeta'_0\approx\de'_1$} and \mbox{$\tuple{\zeta'_0,\ve'_0}\in\interpret{R}{\gamma'(\pr'_{\mathrm{diag}})}$}.
                                                  The latter property implies
                                                  \[
                                                      \forall\pr\in\sigma(\stu):\forall\de\in\Dom:\tuple{\zeta'_0(\pr,\de),\ve'_0(\pr,\de)}\in\interpret{R}{\gamma(\pr)}
                                                  \]
                                                  In turn, from \mbox{$\zeta'_0\approx\de'_0$} we obtain
                                                  \[
                                                      \forall\pr\in\sigma(\stu):\bigcup_{\de\in\Dom}\set{\zeta'_0(\pr,\de)}=\mathop{\bigcup_{\pr_*\in\sigma(\stu),}}_{\de_*\in\Dom}\set{\de'_1(\pr_*,\de_*)}
                                                  \]
                                                  This means in particular that
                                                  \[
                                                      \forall\pr,\pr_*\in\sigma(\stu):\forall\de_*\in\Dom:\exists\de_+\in\Dom:\de'_1(\pr_*,\de_*)=\zeta'_0(\pr,\de_+)
                                                  \]
                                                  Thus
                                                  \[
                                                      \forall\pr,\pr_*\in\sigma(\stu):\forall\de_*\in\Dom:\exists\de_+\in\Dom:\tuple{\de'_1(\pr_*,\de_*),\ve'_0(\pr,\de_+)}\in\interpret{R}{\gamma(\pr)}
                                                  \]
                                                  which we combine with \Cref{eq:back-and-forth:v:exists:ih} as usual to obtain
                                                  \[
                                                      \forall\pr,\pr_*\in\sigma(\stu):\forall\de_*\in\Dom:\exists\de_+\in\Dom:
                                                      \left(
                                                      \tuple{\de'_1(\pr_*,\de_*),\ve'_0(\pr,\de_+)}\in\interpret{R}{\gamma(\pr)}
                                                      \text{ and }
                                                      \ve'_c(\pr,\de_+)\in\interpret{B}{\gamma(\pr)}
                                                      \right)
                                                  \]
                                                  which in particular implies
                                                  \[
                                                      \forall\pr\in\sigma(\stu):\forall\de\in\Dom:\exists\de_+\in\Dom:
                                                      \left(
                                                      \tuple{\de'_1(\pr,\de),\ve'_0(\pr,\de_+)}\in\interpret{R}{\gamma(\pr)}
                                                      \text{ and }
                                                      \ve'_0(\pr,\de_+)\in\interpret{B}{\gamma(\pr)}
                                                      \right)
                                                  \]
                                                  whence we get
                                                  \[
                                                      \forall\pr\in\sigma(\stu):\forall\de\in\Dom:\de'_1(\pr,\de)\in\interpretgp{(\exists R.B)}
                                                  \]
                                        \end{enumerate}
                                        In direction (\ref{claim:norm2:back-and-forth:v:forth}), assume that
                                        \[
                                            \forall\pr\in\sigma:\forall\de\in\Dom:\forall\pr_*\in\sigma(\stu):\de'_0(\pr,\de)\in\interpret{(\exists R.B)}{\gamma(\pr'(\prvo,\de))}
                                        \]
                                        By the definition of the semantics, we obtain
                                        \[
                                            \forall\pr\in\sigma:\forall\de\in\Dom:\forall\pr_*\in\sigma(\stu):\exists\ve\in\Dom:
                                            \left(
                                            \tuple{\de'_0(\pr,\de),\ve}\in\interpret{R}{\gamma(\pr'(\pr_*,\de))} \text{ and } \ve_{\pr,\de}\in\interpret{B}{\gamma(\pr'(\pr_*,\de))}
                                            \right)
                                        \]
                                        Since \mbox{$\pr'=\prvo=\pr'_{\mathrm{diag}}$}, this means that
                                        \[
                                            \forall\pr\in\sigma:\forall\de\in\Dom:\forall\pr_*\in\sigma(\stu):\exists\ve_{\pr,\de}\in\Dom:
                                            \left(
                                            \tuple{\de'(\pr,\de),\ve}\in\interpret{R}{\gamma(\pr_*)} \text{ and } \ve\in\interpret{B}{\gamma(\pr_*)}
                                            \right)
                                        \]
                                        Since $\Dom'$ contains every function \mbox{$\zeta'\colon\sigma(\stu)\times\Dom\to\Dom$}, there in particular exists a function \mbox{$\ve'_0\in\Dom'$} such that \mbox{$\ve'_0(\pr,\de)=\ve_{\pr,\de}$} for all \mbox{$\pr\in\sigma(\stu)$} and \mbox{$\de\in\Dom$}.
                                        Thus
                                        \[
                                            \exists\ve'_0\in\Dom':\forall\pr\in\sigma(\stu):\forall\de\in\Dom:\forall\pr_*\in\sigma(\stu):
                                            \left(
                                            \tuple{\de'_0(\pr,\de),\ve'_0(\pr,\de)}\in\interpret{R}{\gamma(\pr_*)} \text{ and } \ve'_0(\pr,\de)\in\interpret{B}{\gamma(\pr_*)}
                                            \right)
                                        \]
                                        which means
                                        \begin{multline*}
                                            \left(
                                            \exists\ve'_0\in\Dom':\forall\pr\in\sigma(\stu):\forall\de\in\Dom:\forall\pr_*\in\sigma(\stu):
                                            \tuple{\de'_0(\pr,\de),\ve'_0(\pr,\de)}\in\interpret{R}{\gamma(\pr_*)}
                                            \right)
                                            \text{ and } \\
                                            \left(
                                            \exists\ve'_0\in\Dom':\forall\pr\in\sigma(\stu):\forall\de\in\Dom:\forall\pr_*\in\sigma(\stu):
                                            \ve'_0(\pr,\de)\in\interpret{B}{\gamma(\pr_*)}
                                            \right)
                                        \end{multline*}
                                        where we can apply the induction hypothesis to the second line to obtain
                                        \begin{multline*}
                                            \left(
                                            \exists\ve'_0\in\Dom':\forall\pr\in\sigma(\stu):\forall\de\in\Dom:\forall\pr_*\in\sigma(\stu):
                                            \tuple{\de'_0(\pr,\de),\ve'_0(\pr,\de)}\in\interpret{R}{\gamma(\pr_*)}
                                            \right)
                                            \text{ and } \\
                                            \ve'_0\in\interpret{B}{\gamma'(\prvo)}
                                        \end{multline*}
                                        which in particular means that
                                        \[
                                            \exists\ve'_0\in\Dom':
                                            \left(
                                            \left(
                                                \forall\pr\in\sigma(\stu):\forall\de\in\Dom:
                                                \tuple{\de'_0(\pr,\de),\ve'_0(\pr,\de)}\in\interpret{R}{\gamma(\pr)}
                                                \right)
                                            \text{ and } \ve'_0\in\interpret{B}{\gamma'(\prvo)}
                                            \right)
                                        \]
                                        which (employing among other things that \mbox{$\prvo(\pr,\de)=\pr$} for all \mbox{$\pr\in\sigma(\stu)$} and \mbox{$\de\in\Dom$}) implies
                                        \[
                                            \exists\ve'_0\in\Dom':\left( \tuple{\de'_0,\ve'_0}\in\interpret{R}{\gamma'(\prvo)} \text{ and } \ve'\in\interpret{B}{\gamma'(\prvo)} \right)
                                        \]
                                        that is,
                                        \(
                                        \de'_0\in\interpret{(\exists R.B)}{\gamma'(\prvo)}
                                        \).
                                  \item $E=\standdsp B$:
                                        In direction (\ref{claim:norm2:back-and-forth:v:back}), assume \mbox{$\de'\in\interpretpgpp{(\standdsp B)}$}.
                                        Thus there exists some \mbox{$\pr''\in\sigma'(\stsp)$} such that \mbox{$\de'\in\interpretpgppp{B}$}.
                                        The induction hypothesis yields
                                        \[
                                            \forall\pr\in\sigma(\stu):\forall\de\in\Dom:\de'(\pr,\de)\in\interpret{B}{\gamma(\pr''(\pr,\de))}
                                        \]
                                        Thus for all \mbox{$\pr\in\sigma(\stu)$} and \mbox{$\de\in\Dom$} there exists a \mbox{$\pr_{\pr,\de}^{\stsp}\in\sigma(\stsp)$}, namely \mbox{$\pr_{\pr,\de}^\stsp\eqdef\pr''(\pr,\de)$}, such that \mbox{$\de'(\pr,\de)\in\interpret{B}{\gamma(\pr_{\pr,\de}^\stsp)}$}.
                                        This directly yields
                                        \(
                                        \forall\pr\in\sigma(\stu):\forall\de\in\Dom:\de'(\pr,\de)\in\interpretpgpp{(\standdsp B)}
                                        \).

                                        In direction (\ref{claim:norm2:back-and-forth:v:forth}), assume
                                        \[
                                            \forall\pr\in\sigma(\stu):\forall\de\in\Dom:\forall\pr_*\in\sigma(\stu): \de'(\pr,\de)\in\interpret{(\standdsp B)}{\gamma(\prvo(\pr_*,\de))}
                                        \]
                                        This means that
                                        \[
                                            \forall\pr\in\sigma(\stu):\forall\de\in\Dom:\forall\pr_*\in\sigma(\stu):\exists\pr_{\pr_*,\de}^\stsp\in\sigma(\stsp): \de'(\pr,\de)\in\interpret{B}{\gamma(\pr_{\pr_*,\de}^\stsp)}
                                        \]
                                        Choose \mbox{$\pr'_\stsp\in\sigma'(\stsp)$} such that for all \mbox{$\pr_*\in\sigma(\stu)$} and \mbox{$\de\in\Dom$}, we have \mbox{$\pr'_\stsp(\pr_*,\de)=\pr_{\pr_*,\de}^\stsp$}.
                                        Thus we obtain
                                        \[
                                            \exists\pr'_\stsp\in\sigma'(\stsp):\forall\pr\in\sigma(\stu):\forall\de\in\Dom:\forall\pr_*\in\sigma(\stu): \de'(\pr,\de)\in\interpret{B}{\gamma(\pr'_\stsp(\pr_*,\de))}
                                        \]
                                        We apply the induction hypothesis and get
                                        \(
                                        \exists\pr'_\stsp\in\Dom':\de'\in\interpret{B}{\gamma'(\pr'_\stsp)}
                                        \), that is,
                                        \(
                                        \de'\in\interpretpgpp{(\standdsp B)}
                                        \).
                                  \item $E=\standbsp B$:
                                        In direction (\ref{claim:norm2:back-and-forth:v:back}), we have
                                        \begin{align*}
                                             & \phimplies \de'\in\interpretpgpp{(\standbsp B)}                                                                                                                  \\
                                             & \implies \forall\pr''\in\sigma'(\stsp):\de'\in\interpretpgppp{B}                                                                                                 \\
                                             & \stackrel{\text{(IH)}}{\implies} \forall\pr''\in\sigma'(\stsp):\forall\pr\in\sigma(\stu):\forall\de\in\Dom:\de'(\pr,\de)\in\interpret{B}{\gamma(\pr''(\pr,\de))} \\
                                             & \implies \forall\pr\in\sigma(\stu):\forall\de\in\Dom:\forall\pr''\in\sigma'(\stsp):\de'(\pr,\de)\in\interpret{B}{\gamma(\pr''(\pr,\de))}                         \\
                                             & \stackrel{\dagger}{\implies} \forall\pr\in\sigma(\stu):\forall\de\in\Dom:\forall\pr_\stsp\in\sigma(\stsp):\de'(\pr,\de)\in\interpret{B}{\gamma(\pr_\stsp)}       \\
                                             & \implies \forall\pr\in\sigma(\stu):\forall\de\in\Dom:\de'(\pr,\de)\in\interpret{(\standbsp B)}{\gamma(\pr)}
                                        \end{align*}
                                        and for the converse direction, (\ref{claim:norm2:back-and-forth:v:forth}), we consider
                                        \begin{align*}
                                             & \phimplies \forall\pr\in\sigma(\stu):\forall\de\in\Dom:\forall\pr_*\in\sigma(\stu): \de'(\pr,\de)\in\interpret{(\standbs B)}{\gamma(\pr'(\pr_*,\de))}                                     \\
                                             & \implies \forall\pr\in\sigma(\stu):\forall\de\in\Dom:\forall\pr_*\in\sigma(\stu):\forall\pr_\stsp\in\sigma(\stsp): \de'(\pr,\de)\in\interpret{B}{\gamma(\pr_\stsp)}                       \\
                                             & \stackrel{\dagger}{\implies} \forall\pr\in\sigma(\stu):\forall\de\in\Dom:\forall\pr_*\in\sigma(\stu):\forall\pr''\in\sigma'(\stsp): \de'(\pr,\de)\in\interpret{B}{\gamma(\pr''(\pr,\de))} \\
                                             & \implies \forall\pr''\in\sigma'(\stsp):\forall\pr\in\sigma(\stu):\forall\de\in\Dom:\forall\pr_*\in\sigma(\stu): \de'(\pr,\de)\in\interpret{B}{\gamma(\pr''(\pr,\de))}                     \\
                                             & \stackrel{\text{(IH)}}{\implies} \forall\pr''\in\sigma'(\stsp):\de'\in\interpret{B}{\gamma'(\pr'')}                                                                                       \\
                                             & \implies \de'\in\interpretpgpp{(\standbsp B)}
                                        \end{align*}
                                        Here, in both directions the implications marked $\dagger$ can be justifed just like in the proof of \Cref{claim:norm2:back-and-forth}.
                              \end{itemize}
                          \end{claimproof}
                      \end{numberclaim}
                      The two previous claims can be combined to obtain the following “global” property:
                      \begin{numberclaim}
                          \label{claim:norm2:forth:all}
                          For every concept $E$, \mbox{$\de'\in\Dom'$}, and \mbox{$\pr'\in\Precs'$}\/:
                          \[
                              \de'\in\interpretpgpp{E}
                              \implies
                              \left(
                              \forall\pr\in\sigma(\stu):\forall\de\in\Dom:\de'(\pr,\de)\in\interpret{E}{\gamma(\pr'(\pr,\de))}
                              \right)
                          \]
                          \begin{claimproof}
                              Follows from direction ``$\implies\!$'' of \Cref{claim:norm2:back-and-forth} and \Cref{claim:norm2:back-and-forth:v:back} of \Cref{claim:norm2:back-and-forth:v}.
                          \end{claimproof}
                      \end{numberclaim}
                      For the new precisifications, we sometimes also have to pay attention to which copy of the functions in $\Dom'$ we are currently considering, especially for the main case we will need below, namely where both copies are contained in the same concept's extension in both precisifications.
                      \begin{numberclaim}
                          \label{claim:norm2:forth:b}
                          For every concept $E$, function \mbox{$\de'\colon\sigma(\stu)\times\Dom\to\Dom$}, and \mbox{$b\in\set{0,1}$}:
                          \begin{align}
                              \de'_b \in \interpret{E}{\prvb}    & \implies \forall\pr\in\sigma(\stu):\forall\de\in\Dom:\de'(\pr,\de)\in\interpretgp{E}                                          \label{eq:same} \\
                              \de'_{1-b}\in \interpret{E}{\prvb} & \implies \forall\pr\in\sigma(\stu):\forall\de\in\Dom:\forall\pr_*\in\sigma(\stu):\de'(\pr,\de)\in\interpret{E}{\gamma(\pr_*)} \label{eq:diff}
                          \end{align}
                          \begin{claimproof}
                              We use structural induction on $E$.
                              \begin{itemize}
                                  \item $E=A\in\BCC$: By definition.
                                  \item $E=\exists R.B$:
                                        \begin{description}
                                            \item[\normalfont(\ref{eq:same}):]
                                                Let \mbox{$\de'_b\in\interpret{(\exists R.B)}{\gamma'(\prvb)}$} for some \mbox{$b\in\set{0,1}$}.
                                                By definition of the semantics, there is an \mbox{$\ve'_c\in\Dom'$} such that \mbox{$\tuple{\de'_b,\ve'_c}\in\interpret{R}{\gamma'(\prvb)}$} and \mbox{$\ve'_c\in\interpret{B}{\gamma'(\prvb)}$}.
                                                By definition of \mbox{$\interpret{R}{\gamma'(\prvb)}$} the only possible reason for \mbox{$\tuple{\de'_b,\ve'_c}\in\interpret{R}{\gamma'(\prvb)}$} is that
                                                \mbox{$b=c$} and \mbox{$\tuple{\de'_b,\ve'_c}\in\interpret{R}{\gamma'(\pr'_{\mathrm{diag}})}$}, that is,
                                                \[
                                                    \forall\pr\in\sigma(\stu):\forall\de\in\Dom:\tuple{\de'_b(\pr,\de),\ve'_c(\pr,\de)}\in\interpret{R}{\gamma(\pr)}
                                                \]
                                                Due to \mbox{$b=c$}, we can apply the induction hypothesis of (\ref{eq:same}) to \mbox{$\ve'_c\in\interpret{B}{\gamma'(\prvb)}$} and obtain
                                                \[
                                                    \forall\pr\in\stu:\forall\de\in\Dom:\ve'_c(\pr,\de)\in\interpretgp{B}
                                                \]
                                                Togher with the above, we obtain
                                                \[
                                                    \forall\pr\in\sigma:\forall\de\in\Dom:\left( \tuple{\de'_b(\pr,\de),\ve'_c(\pr,\de)}\in\interpret{R}{\gamma(\pr)} \text{ and } \ve'_c(\pr,\de)\in\interpret{B}{\gamma(\pr)} \right)
                                                \]
                                                Thus
                                                \[
                                                    \forall\pr\in\sigma:\forall\de\in\Dom:\exists\ve_{\pr,\de}\in\Dom:\left( \tuple{\de'_b(\pr,\de),\ve_{\pr,\de}}\in\interpret{R}{\gamma(\pr)} \text{ and } \ve_{\pr,\de}\in\interpret{B}{\gamma(\pr)} \right)
                                                \]
                                                namely we set \mbox{$\ve_{\pr,\de}\eqdef\ve'_c(\pr,\de)$} in each case.
                                                By the definition of the semantics we get
                                                \[
                                                    \forall\pr\in\sigma:\forall\de\in\Dom:\de'(\pr,\de)\in\interpret{(\exists R.B)}{\gamma(\pr)}
                                                \]
                                            \item[\normalfont(\ref{eq:diff}):]
                                                Let \mbox{$\de'_{1-b}\in\interpret{(\exists R.B)}{\gamma'(\prvb)}$}.
                                                By the definition of the semantics, there exists some \mbox{$\ve'_c\in\Dom'$} such that \mbox{$\tuple{\de'_{1-b},\ve'_c}\in\interpret{R}{\gamma'(\prvb)}$} and \mbox{$\ve'_c\in\interpret{B}{\gamma'(\prvb)}$}.
                                                According to the definition of $\interpret{R}{\gamma'(\prvb)}$, there are two possible cases:
                                                \begin{enumerate}
                                                    \item \mbox{$c=1-b$} and \mbox{$\forall\pr\in\sigma(\stu):\forall\de\in\Dom:\forall\pr_*\in\sigma(\stu):\tuple{\de'_{1-b}(\pr,\de),\ve'_{c}(\pr,\de)}\in\interpretgps{R}$}.
                                                          We can apply the induction hypothesis of (\ref{eq:diff}) to the fact that \mbox{$\ve'_c=\ve'_{1-b}\in\interpret{B}{\gamma'(\prvb)}$} and obtain
                                                          \[
                                                              \forall\pr\in\stu:\forall\de\in\Dom:\forall\pr_*\in\sigma(\stu):\ve'_c(\pr,\de)\in\interpretgps{B}
                                                          \]
                                                          In combination with the presumption of this case, we obtain
                                                          \[
                                                              \forall\pr\in\sigma(\stu):\forall\de\in\Dom:\forall\pr_*\in\sigma(\stu):
                                                              \left(
                                                              \tuple{\de'_b(\pr,\de),\ve'_c(\pr,\de)}\in\interpretgps{R}
                                                              \text{ and }
                                                              \ve'_c(\pr,\de)\in\interpretgps{B}
                                                              \right)
                                                          \]
                                                          thus proving the claim.
                                                    \item $c=b$ and there is some \mbox{$\ze'_b\in\Dom'$} with \mbox{$\ze'_b\approx\de'_{1-b}$} and \mbox{$\tuple{\ze'_b,\ve'_c}\in\interpret{R}{\gamma'(\pr'_{\mathrm{diag}})}$}.
                                                          The latter property implies
                                                          \[
                                                              \forall\pr_*\in\sigma(\stu):\forall\de\in\Dom:\tuple{\ze'_b(\pr_*,\de),\ve'_c(\pr_*,\de)}\in\interpret{R}{\gamma(\pr_*)}
                                                          \]
                                                          In turn, from \mbox{$\ze'_b\approx\de'_{1-b}$} we obtain
                                                          \[
                                                              \forall\pr_*\in\sigma(\stu):\bigcup_{\de\in\Dom}\set{\ze'_b(\pr_*,\de)}=\mathop{\bigcup_{\pr\in\sigma(\stu),}}_{\de\in\Dom}\set{\de'_{1-b}(\pr,\de)}
                                                          \]
                                                          This means in particular that
                                                          \[
                                                              \forall\pr\in\sigma(\stu):\forall\de\in\Dom:\forall\pr_*\in\sigma(\stu):\exists\de_+\in\Dom:\ze'_b(\pr_*,\de_+)=\de'_{1-b}(\pr,\de)
                                                          \]
                                                          whence
                                                          \[
                                                              \forall\pr\in\sigma(\stu):\forall\de\in\Dom:\forall\pr_*\in\sigma(\stu):\exists\de_+\in\Dom:\tuple{\de'_{1-b}(\pr,\de),\ve'_c(\pr_*,\de_+)}\in\interpretgps{R}
                                                          \]
                                                          Now \mbox{$b=c$} means that we can apply the induction hypothesis of (\ref{eq:same}) to \mbox{$\ve'_c=\ve'_b\in\interpret{B}{\gamma'(\prvb)}$}, from which we get
                                                          \[
                                                              \forall\pr_*\in\sigma(\stu):\forall\de\in\Dom:\ve'_b(\pr_*,\de)\in\interpretgps{B}
                                                          \]
                                                          which we combine with the line above as usual to obtain
                                                          \[
                                                              \forall\pr,\pr_*\in\sigma(\stu):\forall\de\in\Dom:\exists\de_+\in\Dom:
                                                              \left(
                                                              \tuple{\de'_{1-b}(\pr,\de),\ve'_c(\pr_*,\de_+)}\in\interpretgps{R}
                                                              \text{ and }
                                                              \ve'_c(\pr_*,\de_+)\in\interpretgps{B}
                                                              \right)
                                                          \]
                                                          which in particular implies
                                                          \[
                                                              \forall\pr\in\sigma(\stu):\forall\de\in\Dom:\forall\pr_*\in\sigma(\stu):\exists\ve\in\Dom:
                                                              \left(
                                                              \tuple{\de'_{1-b}(\pr,\de),\ve}\in\interpretgps{R}
                                                              \text{ and }
                                                              \ve\in\interpretgps{B}
                                                              \right)
                                                          \]
                                                          namely in each case \mbox{$\ve\eqdef\ve'_c(\pr_*,\de_+)$},
                                                          whence we get
                                                          \[
                                                              \forall\pr\in\sigma(\stu):\forall\de\in\Dom:\forall\pr_*\in\sigma(\stu):\de'_{1-b}(\pr,\de)\in\interpret{(\exists R.B)}{\gamma(\pr_*)}
                                                          \]
                                                          thus proving the claim.
                                                \end{enumerate}
                                        \end{description}
                                  \item $E=\standdsp B$:
                                        Let \mbox{$\pr'\in\set{\prvo,\prvi}$} and consider \mbox{$\de'_b\in\interpretpgpp{(\standdsp B)}$}.
                                        Then there is some \mbox{$\pr''\in\sigma'(\stsp)$} such that \mbox{$\de'_{1-b}\in\interpretpgppp{B}$}.
                                        By \Cref{claim:norm2:forth:all} we obtain
                                        \[
                                            \forall\pr\in\sigma(\stu):\forall\de\in\Dom:
                                            \de'_{b}(\pr,\de)\in\interpret{B}{\gamma(\pr''(\pr,\de))}
                                        \]
                                        whence for all \mbox{$\pr\in\sigma(\stu)$} and \mbox{$\de\in\Dom$} we can set \mbox{$\pr_{\pr,\de}\eqdef\pr''(\pr,\de)$} to rewrite this into
                                        \[
                                            \forall\pr\in\sigma(\stu):\forall\de\in\Dom:\exists\pr_{\pr,\de}\in\sigma(\stsp):
                                            \de'_{b}(\pr,\de)\in\interpret{B}{\gamma(\pr_{\pr,\de})}
                                        \]
                                        and we conclude the claim:
                                        \[
                                            \forall\pr\in\sigma(\stu):\forall\de\in\Dom:\forall\pr_*\in\sigma(\stu):
                                            \de'_{b}(\pr,\de)\in\interpretgps{(\standdsp B)}
                                        \]
                                  \item $E=\standbsp B$:
                                        Let \mbox{$\pr'\in\set{\prvo,\prvi}$} and consider \mbox{$\de'_b\in\interpretpgpp{(\standbsp B)}$}.
                                        Then for all \mbox{$\pr''\in\sigma'(\stsp)$}, we have \mbox{$\de'_b\in\interpretpgppp{B}$}, which by \Cref{claim:norm2:forth:all} means that
                                        \[
                                            \forall\pr''\in\sigma'(\stsp):\forall\pr\in\sigma(\stu):\forall\de\in\Dom:\de'_b(\pr,\de)\in\interpret{B}{\gamma(\pr''(\pr,\de))}
                                        \]
                                        Let \mbox{$\pr_\stsp\in\sigma(\stsp)$} be arbitrary.
                                        Consider the function \mbox{$\pr'_\stsp = \set{ (\pr,\de)\mapsto \pr_\stsp }$}, for which we have \mbox{$\pr'_\stsp\in\sigma'(\stsp)$} by definition
                                        and thus also
                                        \[
                                            \forall\pr\in\sigma(\stu):\forall\de\in\Dom:\de'_b(\pr,\de)\in\interpret{B}{\gamma(\pr'_\stsp(\pr,\de))}=\interpret{B}{\gamma(\pr_\stsp)}
                                        \]
                                        Since $\pr_\stsp$ was arbitrarily chosen, we get
                                        \[
                                            \forall\pr\in\sigma(\stu):\forall\de\in\Dom:
                                            \forall\pr_\stsp\in\sigma(\stsp):
                                            \de'_b(\pr,\de)\in\interpret{B}{\gamma(\pr_\stsp)}
                                        \]
                                        which shows the claim.
                              \end{itemize}
                          \end{claimproof}
                      \end{numberclaim}
                      We now continue the main proof, showing that \mbox{$\dlstruct'\models\K'$}.
                      \begin{itemize}
                          \item \mbox{$\dlstruct'\models\standbu[C\dlsub A]$}: By definition.
                          \item \mbox{$\dlstruct'\models\standbs[\standdvo A \dland \standdvi A \dlsub D ]$}:
                                Let \mbox{$\pr'\in\sigma'(\sts)$} and \mbox{$\de'_b\in\interpretpgpp{(\standdvo A\dland \standdvi A)}$} for some \mbox{$b\in\set{0,1}$}.
                                We have to show \mbox{$\de'_b\in\interpretpgpp{D}$}.
                                By definition of $\dlstruct'$, we conclude that \mbox{$\de'_b\in\interpret{A}{\gamma'(\prvo)} \cap \interpret{A}{\gamma'(\prvi)}$};
                                we furthermore have \mbox{$\prvo,\prvi\in\sigma'(\stu)$} and
                                \mbox{$\interpret{A}{\gamma'(\prvy)}=\interpret{C}{\gamma'(\prvy)}$} for all \mbox{$y\in\set{0,1}$};
                                thus it follows that \mbox{$\de'_b\in\interpret{C}{\gamma'(\prvo)} \cap \interpret{C}{\gamma'(\prvi)}$}.
                                So in any case we obtain that $\de'_b\in\interpret{C}{\gamma'(\pr_{\stv_{1-b}})}$, whence by \Cref{claim:norm2:forth:b} we get
                                \[
                                    \forall\pr\in\sigma(\stu):\forall\de\in\Dom:\forall\pr_*\in\sigma(\stu):
                                    \de'(\pr,\de)\in\interpret{C}{\gamma(\pr_*)}
                                \]
                                Hence, for all \mbox{$\pr\in\sigma(\stu)$} and \mbox{$\de\in\Dom$} we have \mbox{$\de'(\pr,\de)\in\interpretgp{(\standbu C)}$}.
                                From \mbox{$\dlstruct\models\standbs[\standbu C\dlsub D]$} it therefore follows that \mbox{$\de'(\pr_*,\de)\in\interpretgp{D}$} for all \mbox{$\pr_*\in\sigma(\stu)$}, \mbox{$\de\in\Dom$}, and \mbox{$\pr\in\sigma(\sts)$}.
                                Now by \mbox{$\pr'\in\sigma'(\sts)$} we obtain that for any \mbox{$\pr_*\in\sigma(\stu)$} and \mbox{$\de\in\Dom$} we get \mbox{$\pr'(\pr_*,\de)\in\sigma(\sts)$}.
                                Thus it follows that \mbox{$\de'(\pr,\de)\in\interpret{D}{\gamma(\pr'(\pr_*,\de))}$} for all \mbox{$\pr,\pr_*\in\sigma(\stu)$} and \mbox{$\de\in\Dom$}.
                                By \Cref{claim:norm2:back-and-forth}/\Cref{claim:norm2:back-and-forth:v} above, we get \mbox{$\de'\in\interpretpgpp{D}$}.
                          \item \mbox{$\dlstruct'\models\tau$} for all axioms \mbox{$\tau\in\K'\setminus\set{\standbs[C\dlsub A],\standbs[\standdvo A \dland \standdvi A \dlsub D]}$}:
                                We do a case distinction on the possible forms of axioms.
                                \begin{itemize}
                                    \item \mbox{$\tau=\standbsp[E\dlsub F]$}.
                                          By presumption \mbox{$\dlstruct\models\K$} we get that for every \mbox{$\pr\in\sigma(\sts')$} we have \mbox{$\interpretgp{E}\subseteq\interpretgp{F}$}.
                                          Let \mbox{$\pr'\in\sigma'(\sts')$} and \mbox{$\de'\in\interpretpgpp{E}$}.
                                          From \Cref{claim:norm2:forth:all} we get that \mbox{$\de'(\pr,\de)\in\interpret{E}{\gamma(\pr'(\pr,\de))}$} for all \mbox{$\pr\in\sigma(\stu)$} and \mbox{$\de\in\Dom$}.
                                          From \mbox{$\pr'\in\sigma'(\sts')$} we get that \mbox{$\pr'(\pr_*,\de)\in\sigma(\sts')$} for all \mbox{$\pr_*\in\sigma(\stu)$} and \mbox{$\de\in\Dom$}.
                                          It follows that \mbox{$\interpret{E}{\gamma(\pr'(\pr_*,\de))}\subseteq\interpret{F}{\gamma(\pr'(\pr_*,\de))}$} for all \mbox{$\pr_*\in\sigma(\stu)$} and \mbox{$\de\in\Dom$}.
                                          In particular, \mbox{$\de'(\pr,\de)\in\interpret{F}{\gamma(\pr'(\pr_*,\de))}$} for all \mbox{$\pr,\pr_*\in\sigma(\stu)$} and \mbox{$\de\in\Dom$}.
                                          Therefore, \Cref{claim:norm2:back-and-forth}/\Cref{claim:norm2:back-and-forth:v} yields \mbox{$\de'\in\interpretpgpp{F}$}.
                                    \item \mbox{$\tau=\standbsp[R_1\circ\ldots\circ R_n\dlsub R]$}.
                                          Let \mbox{$\pr'\in\sigma'(\stsp)$} and \mbox{$\de_0',\de_1',\ldots,\de_n'\in\Dom'$} with \mbox{$\tuple{\de_{i-1}',\de_{i}'}\in\interpretpgpp{R_i}$} for all \mbox{$1\leq i\leq n$}.
                                          Then by definition, for all \mbox{$1\leq i\leq n$} we get \mbox{$\tuple{\de_{i-1}'(\pr,\de),\de_{i}'(\pr,\de)}\in\interpretgpp{R_i}$} for all \mbox{$\pr\in\sigma(\stu)$} and \mbox{$\de\in\Dom$}.
                                          Since by presumption \mbox{$\dlstruct\models\standbsp[R_1\circ\ldots\circ R_n\dlsub R]$}, we get \mbox{$\tuple{\de_0'(\pr,\de),\de_n(\pr,\de)}\in\interpret{R}{\gamma'(\pr'(\pr,\de))}$} for all \mbox{$\pr'\in\sigma(\stsp)$}, \mbox{$\pr\in\sigma(\stu)$}, and \mbox{$\de\in\Dom$} (since \mbox{$\pr'\in\sigma'(\stsp)$} implies that \mbox{$\pr'(\pr_*,\de)\in\sigma(\sts')$} for all \mbox{$\pr_*\in\sigma(\stu)$} and \mbox{$\de\in\Dom$}).
                                          Thus overall, we get \mbox{$\tuple{\de_0',\de_n'}\in\interpretpgpp{R}$};
                                          since \mbox{$\pr'\in\sigma'(\stsp)$} was arbitrary, we get \mbox{$\dlstruct'\models R_1\circ\ldots\circ R_n\dlsub R$}.
                                \end{itemize}
                          \item \mbox{$\dlstruct'\models\alpha$} for all assertions $\alpha\in\K$:
                                We do a case distinction on the possible forms of assertions.
                                \begin{itemize}
                                    \item $\alpha=\standbsp[B(a)]$:
                                          It follows that \mbox{$B\in\BCC$} since $\K$ is in normal form.
                                          From \mbox{$\dlstruct\models\A$} we get that for any \mbox{$\pr\in\sigma(\sts')$} we have \mbox{$\interpretgp{a}=\de_a\in\interpretgp{B}$}.
                                          Let \mbox{$\pr'\in\sigma'(\sts')$}.
                                          By definition, we get that \mbox{$\interpretpgpp{a}(\pr,\de)=\de_a$} for all \mbox{$\de\in\Dom$} and all \mbox{$\pr\in\Precs$} and in particular for all \mbox{$\pr\in\sigma(\stu)$}.
                                          From \mbox{$\pr'\in\sigma'(\sts')$} we get that for any \mbox{$\pr\in\sigma(\stu)$} and \mbox{$\de\in\Dom$} we find \mbox{$\pr'(\pr,\de)\in\sigma(\sts')$}.
                                          It thus follows that \mbox{$\interpretpgpp{a}(\pr,\de)=\de_a\in\interpret{B}{\gamma(\pr'(\pr,\de))}$} for all \mbox{$\pr\in\sigma(\stu)$} and \mbox{$\de\in\Dom$}.
                                          Consequently, \mbox{$\interpretpgpp{a}\in\interpretpgpp{B}$}.
                                    \item $\alpha=\standbsp[R(a,c)]$:
                                          From \mbox{$\dlstruct\models\A$} it follows that for every \mbox{$\pr\in\sigma(\sts')$} we have \mbox{$\tuple{\interpretgp{a},\interpretgp{c}}\in\interpretgp{R}$}.
                                          Let \mbox{$\pr'\in\sigma'(\sts')$}.
                                          As above, we get \mbox{$\interpretpgpp{a}(\pr,\de)=\de_a$} and \mbox{$\interpretpgpp{c}(\pr,\de)=\de_c$} for every \mbox{$\de\in\Dom$} and \mbox{$\pr\in\sigma(\stu)$}.
                                          From \mbox{$\pr'\in\sigma'(\sts')$}, we get \mbox{$\pr'(\pr,\de)\in\sigma(\sts')$} for every \mbox{$\de\in\Dom$} and \mbox{$\pr\in\sigma(\stu)$}.
                                          Thus \mbox{$\tuple{\interpretpgpp{a}(\pr,\de),\interpretpgpp{a}(\pr,\de)}\in\interpret{R}{\gamma(\pr'(\pr,\de))}$} for all \mbox{$\de\in\Dom$} and \mbox{$\pr\in\sigma(\stu)$} whence
                                          \mbox{$\tuple{\interpretpgpp{a},\interpretpgpp{a}}\in\interpretpgpp{R}$}.
                                \end{itemize}
                          \item \mbox{$\dlstruct'\models\zeta$} for all sharpening statements \mbox{$\zeta\in\K$}:
                                Let \mbox{$\sts_1\cap\ldots\cap\sts_n\preceq\sts'\in\K$} and \mbox{$\pr'\in\sigma'(\sts_1)\cap\ldots\cap\sigma'(\sts_n)$}.
                                Then by definition of $\sigma'$, we get \mbox{$\mathop{\bigcup_{\pr\in\sigma(\stu),}}_{\de\in\Dom}\set{\pr'(\pr,\de)}\subseteq\sigma(\sts_i)$} for all \mbox{$1\leq i\leq n$}.
                                Since \mbox{$\dlstruct\models\sts_1\cap\ldots\cap\sts_n\preceq\stsp$}, we get that for all \mbox{$\pr\in\sigma(\stu)$} and \mbox{$\de\in\Dom$} we have \mbox{$\pr'(\pr,\de)\in\sigma(\stsp)$}.
                                By definition of $\sigma'$ this means that \mbox{$\pr'\in\sigma'(\stsp)$}.
                      \end{itemize}
                      (b)~
                      Let \mbox{$\dlstruct'\models\T'$} and consider \mbox{$\pr'\in\sigma'(\sts)$} and \mbox{$\de'\in\interpretpgpp{(\standbu C)}$}.
                      Since \mbox{$\prvo,\prvi\in\sigma'(\stu)$}, we thus get \mbox{$\de'\in\interpret{C}{\gamma'(\prvo)} \cap \interpret{C}{\gamma'(\prvi)}$}.
                      By \mbox{$\dlstruct'\models\standbu[C\dlsub A]$}, we obtain that \mbox{$\de'\in\interpret{A}{\gamma'(\prvo)} \cap \interpret{A}{\gamma'(\prvi)}$};
                      that is, \mbox{$\de'\in\interpretpgpp{(\standdvo A\dland \standdvi A)}$}.
                      Since also \mbox{$\dlstruct'\models\standbs[\standdvo A \dland \standdvi A \dlsub D]$}, it follows that \mbox{$\de'\in\interpretpgpp{D}$}.
              \end{description}
    \end{enumerate}
\end{longproof}

\subsubsection{Reasoning problems and reductions}

The deduction calculus we are going to present in the next section decides the fundamental reasoning task of satisfiability for $\SELp$:

\newcommand{\probbox}[1]{\bigskip\centerline{\framebox{\parbox{0.95\columnwidth}{#1}}}\bigskip}

\probbox{
    \textbf{Problem:} $\SELp\ \textsc{Knowledge base satisfiability}$\\
    \textbf{Input:} $\SELp$ knowledge base $\K$.\\
    \textbf{Output:} \textsc{yes}, if $\K$ has a model, \textsc{no} otherwise.
}


This reasoning task is useful in itself, e.g.\ for knowledge engineers to check for grave modelling errors that turn the specified knowledge base globally inconsistent. From a user's  perspective, however, a more application-relevant problem is that of statement entailment, as it allows to ``query'' the specified knowledge for consequences:

\probbox{
    \textbf{Problem:} $\SELp\ \textsc{Statement entailment}$\\
    \textbf{Input:} $\SELp$ knowledge base $\K$,~~ $\SELp$ statement $\phi$.\\
    \textbf{Output:} \textsc{yes}, if \mbox{$\K\models\phi$}, \textsc{no} otherwise.
}


Typically, entailment \mbox{$\K\models\phi$} can be (many-one-)reduced to unsatisfiability of \mbox{$\K\cup\set{\neg\phi}$}.
This is not immediately possible in the case of $\SELp$ due to its restricted syntax: note that despite the possibility to negate single axioms or sharpening statements, the negation of monomials or formulae in general is not supported by the syntax of $\SELp$. However, it turns out that the a similar technique can be applied with some additional care.

It is clear that the straightforward reduction works for
arbitrary modalised literals $\standvs[\literal]$,
for negated formulas $\neg\standvs[\monomial]$,
and for (possibly negated) sharpening statements,
so what remains to be detailed is the reduction for modalised monomials.
Consider \mbox{$\standvs\monomial = \standvs[\literal_1\land\ldots\land\literal_n]$}.
For \mbox{$\standvs=\standbs$}, we have that $\standbs\monomial$ is logically equivalent to $\standbs[\literal_1]\land\ldots\land\standbs[\literal_n]$ and thus we can (Turing-)reduce checking entailment \mbox{$\K\models\standbs[\monomial]$} to checking whether all of \mbox{$\K\cup\set{\neg\standbs[\literal_1]}$}, \ldots, \mbox{$\K\cup\set{\neg\standbs[\literal_n]}$} are unsatisfiable (which is still polynomial in $\K$ and $\standbs[\monomial]$).
Finally, for \mbox{$\standvs=\standds$}, we employ the idea underlying normalisation rule~(\ref{norm1:diamond-formula}) to obtain the following\/:
\begin{lemma}
    \label{thm:lem:diamond-entailment}
    Let $\K$ be a $\SELp$ knowledge base with normal form $\K'$ and $\monomial$ be a monomial.
    It holds that \mbox{$\K\models\standds[\monomial]$} if and only if \mbox{$\K'\models\standbu[\monomial]$} for some \mbox{$\stu\in\Stands$} with \mbox{$\K'\models\stu\preceq\sts$}.%
    \begin{proofsketch}
        The main idea is that for \mbox{$\K\models\standds[\monomial]$} to hold there
        is a formula \mbox{$\standdsp[\monomial']\in\K$} with \mbox{$\K\models\stsp\preceq\sts$} and formulas \mbox{$\standb{\sts_{1}}[\monomial_{1}],\ldots,\standb{\sts_{m}}[\monomial_{m}]\in\K$} with \mbox{$\K\models\stsp\preceq\sts_i$} for all \mbox{$1\leq i\leq n$} (where neither type of formula is strictly required, and \mbox{$\stsp=\sts$} in case no $\standdsp$ formula is involved), such that \mbox{$\set{\monomial',\monomial_1,\ldots,\monomial_m}\models\monomial$}.
        In case all relevant formulas are of the form $\standb{\sts_i}[\monomial_i]$, then \mbox{$\K\models\standbs[\monomial]$} and the claim trivially holds (with \mbox{$\stu=\sts$}).
        In case some \mbox{$\standd{\stsp}[\monomial']\in\K$} is involved, normalisation rule~(\ref{norm1:diamond-formula}) will introduce the new standpoint name $\stu$ that can serve as witness in the normalised KB $\K'$.
    \end{proofsketch}
    \begin{longproof}
        Recall that by \Cref{lem:normalisation}, $\K'$ is a conservative extension of $\K$.
        If $\K$ is inconsistent, the claim holds trivially, so assume that $\K$ is consistent (whence $\K'$ is consistent).
        We use induction on the number $n$ of normalisation rule applications needed to transform $\K$ into its normal form $\K'$.
        \begin{description}
            \item[\normalfont $n=0$:] Then $\K$ is in normal form, i.e.\ $\K'=\K$.
                \begin{description}
                    \item[\normalfont “if”:]
                        Let \mbox{$\K\models\standbu[\monomial]$} for some \mbox{$\stu\in\Stands$} with \mbox{$\K\models\stu\preceq\sts$}.
                        Let \mbox{$\dlstruct=\tuple{\Dom, \Precs, \sigma, \gamma}$} be an arbitrary model of $\K$.
                        Then \mbox{$\sigma(\stu)\neq\emptyset$} whence there is a \mbox{$\pr\in\sigma(\stu)\subseteq\sigma(\sts)$} with \mbox{$\dlstruct,\pr\models\monomial$}.
                        This witnesses \mbox{$\dlstruct\models\standds[\monomial]$}.
                        Since $\dlstruct$ was arbitrary, we obtain \mbox{$\K\models\standds[\monomial]$}.
                    \item[\normalfont “only if”:]
                        Let \mbox{$\K\models\standds[\monomial]$}.
                        Note that due to being in normal form, all formulas in $\K$ are of the form $\standbsp[\axiom]$ for $\axiom$ an axiom and \mbox{$\stsp\in\Stands$}.
                        Consider an arbitrary model $\dlstruct$ of $\K$.
                        Let \mbox{$\K'\subseteq\K$} be such that \mbox{$\K'\models\standds[\monomial]$} and let $\K'$ be $\subseteq$-minimal with respect to this property.
                        Consider some $\standbsp[\axiom]\in\K'$.
                        By $\dlstruct\models\K$ we get $\dlstruct\models\standbsp[\axiom]$.
                        \hs{TODO}
                \end{description}
            \item[\normalfont $n\leadsto n+1$:]
                \begin{description}
                    \item[\normalfont “if”:]
                        Let \mbox{$\K'\models\standbu[\monomial]$} for some \mbox{$\stu\in\Stands$} with \mbox{$\K'\models\stu\preceq\sts$}.
                        Consider any structure $\dlstruct$ with \mbox{$\dlstruct\models\K$}, we have to show \mbox{$\dlstruct\models\standds[\monomial]$}.
                        Assume, for the sake of obtaining a contradiction later on, that \mbox{$\dlstruct\not\models\standds[\monomial]$}.
                        Then \mbox{$\dlstruct\models\standbs[\neg \monomial]$}, that is, for every \mbox{$\pr\in\sigma(\sts)$} we have \mbox{$\dlstruct,\pr\not\models\literal_i$} for some \mbox{$1\leq i\leq n$}.
                        Since $\K'$ is a conservative extension of $\K$, there exists a $\dlstruct'$ extending $\dlstruct$ such that \mbox{$\dlstruct'\models\K'$}.
                        By presumption \mbox{$\dlstruct'\models\stu\preceq\sts$} as well as \mbox{$\dlstruct'\models\standbu[\monomial]$}.
                        Now $\sigma'(\stu)$ is non-empty and every \mbox{$\pr'\in\sigma'(\stu)\subseteq\sigma'(\sts)$} witnesses \mbox{$\dlstruct'\models\standds[\monomial]$}.
                        Since \mbox{$\dlstruct\not\models\standds[\monomial]$}, we have, in addition to \mbox{$\sigma(\sts)\subseteq\sigma'(\sts)$}, also \mbox{$\sigma'(\stu)\cap\sigma(\sts)=\emptyset$}.
                        \hs{TODO}
                    \item[\normalfont “only if”:]
                        Let \mbox{$\K\models\standds[\monomial]$}.
                        We have to show that there exists a \mbox{$\stu\in\Stands$} such that for every $\dlstruct'$ with \mbox{$\dlstruct'\models\K'$}, we have \mbox{$\dlstruct'\models\stu\preceq\sts$} and \mbox{$\dlstruct'\models\standbu[\monomial]$}.
                        Consider $\dlstruct'$ with \mbox{$\dlstruct'\models\K'$}.
                        Since $\K'$ is a conservative extension of $\K$, we know that \mbox{$\dlstruct'\models\K$} whence \mbox{$\dlstruct'\models\standds[\monomial]$}.
                        \hs{TODO}
                \end{description}
        \end{description}
    \end{longproof}
\end{lemma}

So to decide \mbox{$\K\models\standds[\monomial]$}, we normalise $\K$ into $\K'$ and then successively enumerate \mbox{$\stsp\in\Stands$} occurring in $\K'$ for which \mbox{$\K'\models\stsp\preceq\sts$} and test \mbox{$\K\models\standbsp[\monomial]$} for each.
In view of these considerations, we arrive at the following reducibility result.

\begin{theorem}\label{thm:reducibility}
    There exists a \PTime Turing reduction from $\SELp$ \textsc{Statement entailment} to $\SELp$ \textsc{knowledge base satisfiability}.
\end{theorem}

Thus, every tractable decision procedure for satisfiability can be leveraged to construct a tractable entailment checker. Therefore, we will concentrate on a method for the former.





\pagebreak

\section{Refutation-Complete Deduction Calculus for Normalised KBs}
\label{sec:calculus}

In this section, we present the Hilbert-style deduction calculus for $\SELp$.
Premises and consequents of the calculus' deduction rules will be axioms in normal form with one notable exception:
We allow for extended versions of modalised GCIs of the general shape
\begin{equation}
	\standb{t}[ A \sqsubseteq \standb{s}	[ B \Rightarrow  C ]],
\end{equation}
the meaning of which should be intuitively clear, despite the fact that $\Rightarrow$ is not a connective available in $\SELp$. In terms of more expressive Standpoint DLs, such an axiom could be written
\mbox{$\standb{t}[ A \sqsubseteq \standb{s}	[ \neg B \sqcup  C ] ]$},
but this would obfuscate the ``Horn nature'' of the statement.  Note that the axiom can be expressed in $\SELp$ by the two axioms $\standb{t}[A \sqsubseteq \standb{s} D]$ and $\standb{s}[D \sqcap B \sqsubseteq  C]$ using an auxiliary fresh concept $D$. Yet, for better treatment in the calculus, we need all the information ``bundled'' within one axiom. With this new axiom type in place, we dispense with axioms of the shapes $\standbs [A \sqsubseteq B]$ and $\standbs [A \sqsubseteq \standbsp B]$, replacing them by $\standball[ \top \sqsubseteq \standb{s}	[ A \Rightarrow  B ] ]$ and $\standb{s}[ A \sqsubseteq \standb{s'}	[ \top \Rightarrow  B ] ]$, respectively. Similarly, we will replace concept assertions of the form $\standbs A(a)$ by $\standball[ \{a\} \sqsubseteq \standb{s}	[ \top \Rightarrow  A ] ]$, where $\{a\}$ is understood as a ``nominal concept'', to be interpreted by the singleton set $\{a^{\dlstruct}\}$ in the usual way.\footnote{Despite us using this convenient representation ``under the hood'', we emphasise that our calculus is not meant to be used with input knowledge bases with free use of nominals. In fact, we have shown that extending Standpoint $\EL$ by nominal concepts leads to intractability \cite{ourijcaisubmission}.}
Then, it should be clear that these are equivalent axiom replacements.

As one final preprocessing step, we introduce, for every concept $\standds B$ that occurs in the normalised KB and every ABox individual $a$, a fresh standpoint symbol denoted $\mathsf{s}[a,B]$ and a fresh concept name $P_{\mathsf{s},a,B}$ and add the following background axioms:
\begin{align}
	\mathsf{s}[a,B]                 & \preceq \mathsf{s}                                           \\
	\standball [ \{a\}              & \sqsubseteq \standb{s} [ B \Rightarrow P_{\mathsf{s},a,B}] ] \\
	\standball [ P_{\mathsf{s},a,B} & \sqsubseteq {\standb{s}}_{[a,B]} [ \top \Rightarrow B ] ]
\end{align}
Intuitively, the purpose of this conservative extension is that, whenever $a$ is required to satisfy $\standds B$, it will satisfy $B$ in all $\mathsf{s}[a,B]$-precisifications, this way arranging for a concrete, ``addressable'' witness for $\standds B(a)$.

Given an $\SELp$ knowledge base $\K$, let $\K^\mathrm{prep}$ denote its preprocessed variant, obtained through normalisation and the steps described above. Again note that $\K^\mathrm{prep}$ can be computed in deterministic polynomial time.
Now let $\K^\vdash$ denote the set of axioms obtained from $\K^\mathrm{prep}$ by saturating it under the deduction rules of \Cref{fig:calculus}.
We note that each axiom type has a bounded number of parameters, each of which can be instantiated by a polynomial number of elements (concepts, roles, individuals, standpoints) occurring in $\K^\mathrm{prep}$. Consequently, the overall number of distinct inferrable axioms is polynomial and therefore the saturation process to obtain $\K^\vdash$ runs in deterministic polynomial time.
We will see in the next section that these observations give rise to a worst-case optimal Datalog implementation of the saturation procedure.

\newcommand{\dedsr}[3]{${(#3)\ \ }\begin{array}{@{\,}c@{\,}}{#1}\\[-2.2ex]\\\hline\\[-2.2ex]{#2}\end{array}$}

%
%
\def\incons{\standball[ \top \sqsubseteq \standball[\top \Rightarrow \bot]]}

\def\axA{T.1\xspace}
\def\axE{T.2\xspace}
\def\axB{T.3\xspace}
\def\axC{T.4\xspace}
\def\axD{T.5\xspace}

\def\hierA{S.1\xspace}
\def\hierB{S.2\xspace}
\def\genr{S.3\xspace}

\def\loc{I.1\xspace}
\def\locN{I.1N\xspace}
\def\globtop{I.2\xspace}

\def\inhierA{\genr}
\def\inhierArole{\genr}
\def\inhierAcirc{\genr}
\def\inhierB{S.4\xspace}

\def\rsub{R.1\xspace}
\def\inhierC{\genr}
\def\inhierD{\genr}
\def\inhierE{\genr}
\def\inhierF{\genr}
\def\inhierG{\genr}

\def\subA{C.1\xspace}
\def\subB{C.2\xspace}
\def\subC{C.3\xspace}
\def\subD{C.4\xspace}

\def\flatA{F.1\xspace}
\def\flatB{F.2\xspace}
\def\flatC{F.3\xspace}
\def\flatD{F.4\xspace}

\def\exA{E.1\xspace}
\def\exB{E.3\xspace}
\def\exC{E.2\xspace}
\def\con{E.4\xspace}

\def\abAA{A.1\xspace}
\def\latestrule{A.2\xspace}
\def\locB{A.3\xspace}
\def\abB{A.4\xspace}
\def\abC{A.5\xspace}
\def\abeA{A.6\xspace}
\def\abeB{A.7\xspace}
\def\lastminute{A.8\xspace}

\def\selfF{L.1\xspace}
\def\selfG{L.2\xspace}
\def\selfC{L.3\xspace}
\def\selfE{L.4\xspace}
\def\selfA{L.5\xspace}
\def\selfB{L.6\xspace}

\def\retA{B.1\xspace}
\def\retB{B.2\xspace}
\def\retC{B.3\xspace}
\def\abr{B.4\xspace}

\begin{figure*}[p]
	\footnotesize

	\hrule\vspace{1ex}
	Tautologies

	\quad
	\dedsr{{}} {\st \preceq *}{\axA}
	\qquad
	\dedsr{{}} {\st \preceq \st}{\axE}
	\qquad
	\dedsr{{}} { \standball[ \top \sqsubseteq \standball [ C \Rightarrow  C ]] }{\axB}
	\qquad
	\dedsr{{}} { \standball[ \top \sqsubseteq \standball [ C \Rightarrow  \top ]] }{\axC}
	\qquad
	\dedsr{{}} {\standball[R \sqsubseteq R]}{\axD}

	\medskip

	\hrule\vspace{1ex}
	Standpoint hierarchy rules (for all $\st\in \Stands$, $\xi$ being any extended GCI, RIA, or role assertion)

	\smallskip

	\quad
	\dedsr{{\st \preceq \sp \quad \sp \preceq \spp}} {\st \preceq \spp}{\hierA}
	\qquad\!\!
	\dedsr{{\st \preceq \st_1 \quad \st \preceq \st_2 \quad \st_1 \cap \st_2 \preceq \sp}} {\st \preceq \sp}{\hierB}
	\qquad\!\!
	\dedsr{\standb{s'}\xi \quad \st \preceq \sp } {\standb{s}\xi}{\genr}
	\qquad\!\!
	\dedsr{	\standb{t} [ C \sqsubseteq \standb{s'} [D \Rightarrow E]] \quad \st \preceq \sp } { \standb{t} [ C \sqsubseteq \standb{s} [D \Rightarrow E]] }{\inhierB}

	\medskip

	\medskip

	\hrule\vspace{1ex}
	\begin{tabular}{@{}l|l@{}}
		Internal inferences for extended GCIs & \qquad  Role subsumptions                                                                     \\[1ex]

		\quad
		\dedsr{	\standb{s} [ C \sqsubseteq \standb{s} [\top \Rightarrow D] ] } {  \standball 	[\top \sqsubseteq \standb{s} [C \Rightarrow D]] }{\loc}
		\qquad\qquad
		\dedsr{  \standb{u}	[\top \sqsubseteq \standb{s} [C \Rightarrow D]] }{	\standball	[\top \sqsubseteq \standb{s} [C \Rightarrow D]] } {\globtop}
		\qquad \qquad
		                                      &
		\qquad\qquad \dedsr{ \standb{s} [ R\sqsubseteq R'' ] \quad \standb{s} [ R''\sqsubseteq R' ] } {\standb{s} [ R\sqsubseteq R' ]}{\rsub} \\
	\end{tabular}

	\medskip




	\hrule\vspace{1ex}

	Forward chaining

	\smallskip

	\quad
	\dedsr{ \standb{t} [ B \sqsubseteq 	\standb{s}	[C \Rightarrow  D]] \quad \standb{t} [ B \sqsubseteq  \standb{s}	[D \Rightarrow E] ]} { \standb{t} [ B \sqsubseteq  \standb{s}	[C \Rightarrow E] ]}{\subA}
	\qquad\qquad
	\dedsr{	\standb{u} [ \top \sqsubseteq \standb{t} [ B \Rightarrow C ]] \quad \standb{t} [ C \sqsubseteq \standb{s} [D \Rightarrow E ]] } {  \standb{t} [ B \sqsubseteq \standb{s} [ D \Rightarrow E ]] }{\subB}

	\medskip

	\quad
	\dedsr{	\standb{u} [ \top \sqsubseteq \standb{t} [ C \Rightarrow D ]] \quad \standb{t} [D \sqsubseteq \standd{s} E]} { \standb{t} [C \sqsubseteq \standd{s} E]}{\subC}
	\qquad\qquad\qquad\quad
	\dedsr{	 \standb{t} [ C \sqsubseteq  \standd{s} D ] \quad \standb{t} [ C \sqsubseteq \standb{s} [ D \Rightarrow  E ]] } { \standb{t} [ C \sqsubseteq \standd{s} E ]  }{\subD}

	\bigskip

	\hrule\vspace{1ex}
	Flattening of modalities

	\smallskip

	\quad
	\dedsr{	 \standb{t} [ C \sqsubseteq  \standb{s'} [\top \Rightarrow D] ] \quad  \standb{s'} [ D \sqsubseteq \standb{s} [E \Rightarrow F]]  } { \standb{t} [ C \sqsubseteq \standb{s} [E \Rightarrow F]] }{\flatA}
	\qquad\quad\,
	\dedsr{	 \standb{t} [ C \sqsubseteq  \standb{s'} [\top \Rightarrow D] ] \quad  \standb{s'} [ D \sqsubseteq \standd{s} E ]  } {  \standb{t} [ C \sqsubseteq \standd{s} E ]  }{\flatB}

	\medskip

	\quad
	\dedsr{ \standb{t} [ C \sqsubseteq  \standd{s'} D ]  \quad  \standb{s'} [ D \sqsubseteq \standb{s} [E \Rightarrow F] ] } { \standb{t} [ C \sqsubseteq \standb{s} [E \Rightarrow F]] }{\flatC}
	\qquad\qquad\qquad\ \
	\dedsr{ \standb{t} [C \sqsubseteq  \standd{s'} D] \quad \standb{s'} [ D \sqsubseteq \standd{s} E ] } { \standb{t} [ C \sqsubseteq \standd{s} E ] }{\flatD}

	\medskip

	\hrule\vspace{1ex}
	Inferences involving existential quantifiers and conjunction

	\smallskip

	\quad
	\dedsr{ \standb{s} [C \sqsubseteq \exists R. D] \quad \standb{u} [ \top \sqsubseteq \standb{s} [ D \Rightarrow E ]] \quad \standb{s} [ R \sqsubseteq R' ]}
	{ \standb{s} [C \sqsubseteq \exists R'. E] }{\exA}
	\qquad\quad
	\dedsr{ \standb{s} [C \sqsubseteq \exists R_1. D] \quad \standb{s} [D \sqsubseteq \exists R_2.E] \quad  \standb{s}	[ R_1 \circ R_2  \sqsubseteq R' ]}{ \standb{s} [C \sqsubseteq \exists R'. E] }{\exC}

	\medskip

	\quad
	\dedsr{ \standb{s} [C \sqsubseteq \exists R. D] \quad \standb{s} [\exists R. D \sqsubseteq F] } {\standball [ \top \sqsubseteq \standb{s} [C \Rightarrow F] ] }{\exB}
	\qquad\qquad\qquad\quad
	\dedsr{	\standb{t} [ B \sqsubseteq \standb{s} [ C \Rightarrow  C_1 ] ] \quad \standb{t} [ B \sqsubseteq \standb{s} [ C \Rightarrow  C_2 ] ] \quad \standb{s} [C_1 \sqcap C_2 \sqsubseteq  D] } { \standb{t} [ B \sqsubseteq \standb{s} [C \Rightarrow  D] ] }{\con}

	\medskip

	\hrule\vspace{1ex}
	Individual-based inferences

	\smallskip

	\quad
	\dedsr{\standb{u} [ \top \sqsubseteq \standb{s} [B \Rightarrow C]]} {\standball [\{a\} \sqsubseteq \standb{s} [B \Rightarrow C]]}{\abAA}
	\qquad\qquad\qquad\quad\quad\ \ 
	\dedsr{  \standb{u}	[\{a\} \sqsubseteq \standb{s} [\top \Rightarrow C] ]}{	\standball [ \top \sqsubseteq \standb{s} [\{a\} \Rightarrow C] ] }{\latestrule}
	\qquad\qquad\qquad\qquad\!\!\!\!
	\dedsr{	\standb{u} [ \{a\} \sqsubseteq \standb{s} [B \Rightarrow C] ] } {  \standball	[\{a\} \sqsubseteq \standb{s} [B \Rightarrow C] ]}{\locB}

	\medskip

	\quad
	\dedsr{ \standb{s} [ R(a,b) ] \quad \standb{s} [ R \sqsubseteq R' ]  } { \standb{s} [ R'(a,b) ]  }{\abB}
	\qquad\qquad\qquad\quad\ \
	\dedsr{ \standb{s} [ R_1(a,b) ] \quad \standb{s} [ R_2(b,c) ] \quad \standb{s}	[ R_1 \circ R_2  \sqsubseteq R' ]  } { \standb{s} [ R'(a,c) ]  }{\abC}

	\medskip

	\quad
	\dedsr{ \standb{s} [ R(a,b) ] \quad \standb{u} [\{b\} \sqsubseteq \standb{s}[\top \Rightarrow B]]  } {\standb{s} [\{a\} \sqsubseteq \exists R.B ] }{\abeA}
	\qquad\quad
	\dedsr{ \standb{s} [ R_1(a,b) ] \quad \standb{s}[\{b\} \sqsubseteq \exists R_2.C ] \quad \standb{s}	[ R_1 \circ R_2  \sqsubseteq R' ] } {\standb{s}[\{a\} \sqsubseteq \exists R'.C ]}{\abeB}

	\medskip

	\quad
	\dedsr{ \standb{s} [ R_1(a,b) ] \quad \standb{u} [\{b\} \sqsubseteq \standb{s}[\top \Rightarrow B]] \quad \standb{s}[B \sqsubseteq \exists R_2.C ] \quad \standb{s}	[ R_1 \circ R_2  \sqsubseteq R' ] } {\standb{s}[\{a\} \sqsubseteq \exists R'.C ]}{\lastminute}

	\medskip

	\hrule\vspace{1ex}
	Interaction of self-loops with other statements

	\smallskip

	\quad
	\dedsr{ \standb{u} [\{a\} \sqsubseteq \standb{s}[\top \Rightarrow \exists R.\mathsf{Self}]] } {\standb{s} [ R(a,a) ]}{\selfF}
	\qquad\!\!
	\dedsr{\standb{u} [\top \sqsubseteq \standbs [C \Rightarrow \exists R.\mathsf{Self}]] } {\standbs [C \sqsubseteq \exists R.C]}{\selfG}
	\qquad\qquad\quad\!\!
	\dedsr{\standb{s}[\exists R.D \sqsubseteq C]}{ \standb{s}[\exists R.\mathsf{Self} \sqcap D \sqsubseteq C]}{\selfC}

	\medskip

	\quad
	\dedsr{ \standb{s} [ R(a,a) ] } {\standball [\{a\} \sqsubseteq \standb{s}[\top \Rightarrow \exists R.\mathsf{Self}]]}{\selfE}
	\qquad\!\!
	\dedsr{\standb{s}[R \sqsubseteq R']} {\standball[\top \sqsubseteq \standb{s}[\exists R.\mathsf{Self} \Rightarrow \exists R'.\mathsf{Self}]]}{\selfA}
	\qquad\!\!
	\dedsr{\standb{s} [ R_1 \circ R_2  \sqsubseteq R' ]}{\standb{s} [\exists R_1.\mathsf{Self} \sqcap \exists R_2.\mathsf{Self} \sqsubseteq \exists R'.\mathsf{Self}]}{\selfB}


	%
	%
	%


	\medskip

	\hrule\vspace{1ex}
	Backpropagation of $\bot$-inferences

	\smallskip

	\quad
	\dedsr{ \standb{s} [C \sqsubseteq \exists R. \bot ] } {\standball	[\top \sqsubseteq \standb{s} [C \Rightarrow \bot] ]}{\retA}
	\qquad\!\!\!
	\dedsr{	\standb{t} [ C \sqsubseteq  \standb{s} [ \top \Rightarrow \bot ] ] } { \standball [ \top \sqsubseteq \standb{t} [ C \Rightarrow \bot ] ] }{\retB}
	\qquad\!\!\!
	\dedsr{ \standb{t} [ C \sqsubseteq  \standd{s} \bot ] } { \standball [ \top \sqsubseteq \standb{t} [ C \Rightarrow \bot ] ] }{\retC}
	\qquad\!\!\!
	\dedsr{ \standb{u} [\{a\} \sqsubseteq \standb{s}[\top \Rightarrow \bot]]  } { \standball [\top \sqsubseteq \standball[\top \Rightarrow \bot]]   }{\abr}

	\bigskip
	\hrule\vspace{1ex}

	\caption{Deduction calculus for $\SELp$ \label{fig:calculus}}

\end{figure*}

We next argue that the presented calculus has the desired properties.
As usual, soundness of the calculus is easy to show and can be argued for each deduction rule separately by referring to the definition of the semantics.
What remains to be shown is a particular type of completeness: Among the inferrable axioms, the particular intrinsically contradictory statement \mbox{$\standball [ \top \sqsubseteq \standball [\top \Rightarrow \bot]]$} will play the pivotal role of indicating unsatisfiability of $\K$ (also referred to as refutation).
We will show that our calculus is \define{refutation-complete}, meaning that for any unsatisfiable $\SELp$ knowledge base $\K$, we have that \mbox{$\standball [ \top \sqsubseteq \standball [\top \Rightarrow \bot]] \in \K^\vdash$}.
More concretely, we prove the contrapositive by establishing the existence of a model whenever \mbox{$\standball [ \top \sqsubseteq \standball [\top \Rightarrow \bot]] \not\in \K^\vdash$}.
This model is canonical in a sense but, as opposed to canonical models of the $\EL$ family, it will typically be infinite.

\subsubsection*{Canonical Model Construction}
Given a $\SELp$ knowledge base \KB with \mbox{$\standball [ \top \sqsubseteq \standball [\top \Rightarrow \bot]] \not\in \K^\vdash$}, we construct a model $\dlstruct$ of $\KB$ in an infinite process:
we start from an initialised model $\dlstruct_0$ and extend it (both by adding domain elements and precisifications) in a stepwise fashion, resulting in a ``monotonic'' sequence of models.
The result of the process is the ``limit'' of this sequence, which can be expressed via an infinite union.

\smallskip

For the initialisation, we choose the standpoint structure
$\dlstruct_0 = \tuple{\Dom_0, \Precs_0, \sigma_0, \gamma_0}$ where:
\begin{itemize}
	\item $\Dom_0$ consists of one element $\delta_a$ for every individual name $a$ mentioned in \KB;
	\item $\Precs_0$ consists of one precisification $\pi_\mathsf{s}$ for every standpoint $\st$ mentioned in \KB (including $*$);
	\item $\sigma_0$ maps each standpoint $\st$ to $\{\pi_{\mathsf{s}}\} \cup \{\pi_{\mathsf{s'}} \mid \sp\preceq\st \in \KB^\vdash \}$;\pagebreak
	\item $\gamma_0$ maps each $\pi_\mathsf{s}$ to the description logic interpretation $\struct$ over $\Dom_0$, where
	      \begin{itemize}
		      \item $\interprets{a} = \delta_a$ for each individual name $a$,
		      \item $\interprets{A} = \{\delta_a \mid \standb{u}[\{a\} \sqsubseteq \standb{s}[\top {\Rightarrow} A]] {\,\in\,} \KB^\vdash\!\!, \st{\,\in\,}\sigma^{\!-1}_0(\pi_\mathsf{s})\}$ for every concept name $A$, and
		      \item $\interprets{R} = \{(\delta_a, \delta_b) \mid \standb{s}[R(a,b)] \in \KB^\vdash, \st\in\sigma^{-1}_0(\pi_\mathsf{s})\}$
		            for each role name $R$.
	      \end{itemize}
\end{itemize}

\smallskip

It can be shown that the obtained structure satisfies all axioms of $\KB$ except for those of the shape \mbox{$\standb{s}[E \sqsubseteq \exists R.F]$}.
This will also be the case for all structures $\dlstruct_1$, $\dlstruct_2$, \ldots\ subsequently produced.
The sequence arises by iteratively adding more domain elements in order to satisfy more axioms of the type $\standb{s}[E \sqsubseteq \exists R.F]$.
Thereby, the concept and role memberships of pre-existing elements with respect to pre-existing standpoints will remain unchanged.
This also justifies the definition of a labelling function $\Lambda_{\pi}$ immutably assigning to every domain element $\delta$ a set of concepts, satisfied in $\pi$.
For $\dlstruct_0$, \mbox{we let $\Lambda_{\pi}$ map elements $\delta_a$ according to}
$$\Lambda_{\pi}(\delta_a) = \{C \mid \standb{u}[\{a\} \sqsubseteq \standb{s}[\top {\Rightarrow} C]] \in \KB^\vdash, \st\in\sigma^{-1}_0(\pi)\}.$$

\smallskip

As discussed above, after having arrived at a structure $\dlstruct_i = \tuple{\Dom_i, \Precs_i, \sigma_i, \gamma_i}$, we inspect if $\dlstruct_i$ satisfies all axioms of the form $\standb{t}[E \sqsubseteq \exists R.F]$. If so, $\dlstruct_i$ is a model of \KB and we are done. Otherwise, we pick some  \mbox{$\delta^* \in \Delta_i$} and some \mbox{$\pi^* \in \Precs_i$} with $ \pi^*\in \sigma(\mathsf{t})$ and $\delta^* \in E^{\gamma(\pi^*)} \setminus (\exists R.F)^{\gamma(\pi^*)}$ for some (previously unsatisfied) axiom $\standb{t}[E \sqsubseteq \exists R.F]$ from \KB.
Among the eligible pairs $\delta^*$, $\pi^*$, we pick one for which the value $\min \{ j \leq i \mid \delta^* \in \Delta_j\} + \min \{ j \leq i \mid \pi^* \in \Pi_j\}$ is minimal; this ensures fairness in the sense that any axiom violation will ultimately be addressed.
Given $\delta^*$ and $\pi^*$ we now obtain $\dlstruct_{i+1} = \tuple{\Dom_{i+1}, \Precs_{i+1}, \sigma_{i+1}, \gamma_{i+1}}$ from
$\dlstruct_i = \tuple{\Dom_i, \Precs_i, \sigma_i, \gamma_i}$ as follows:

\begin{figure*}[t]
    \vspace*{-0.6ex}
    \begin{align*}
        \standbs[C\dlsub\standbsp[D\Rightarrow E]] & \leadsto \mathtt{gci\_nested(s, C, s', D, E)}          &
        \sts_1\cap\sts_2\preceq \sts_3             & \leadsto \mathtt{sharper\_intersection(s_1, s_2, s_3)}   \\
        \standbs[C\dland D\dlsub E]                & \leadsto \mathtt{gci\_conjunction\_left(s, C, D, E)}   &
        \sts_1\preceq\sts_2                        & \leadsto \mathtt{sharper(s_1, s_2)}                      \\
        \standbs[\exists R.C\dlsub D]              & \leadsto \mathtt{gci\_existential\_left(s, R, C, D)}   &
        R_1\circ R_2\dlsub R_3                     & \leadsto \mathtt{ria3(R_1, R_2, R_3)}                    \\
        \standbs[C\dlsub \exists R.D]              & \leadsto \mathtt{gci\_existential\_right(s, C, R, D)}  &
        R_1\dlsub R_2                              & \leadsto \mathtt{ria2(R_1, R_2)}                         \\
        \standbs[C\dlsub\standdsp D]               & \leadsto \mathtt{gci\_diamond\_right(s, C, s', D)}
    \end{align*}
    \vspace*{-4ex}
    \caption{Datalog atoms used to represent normal-form statements of $\SELp$ as used in the deduction calculus.
        Therein, \mbox{$\sts_1,\sts_2,\sts_3,\sts,\stsp\in\Stands$}, \mbox{$C,D,E\in\Concepts$}, and \mbox{$R_1,R_2, R_3,R\in\Roles$}.}
    \label{fig:datalog-atoms}
\end{figure*}

\begin{itemize}
	\item $\Dom_{i+1} = \Dom_{i} \cup \{\delta'\}$, where $\delta'$ is a fresh domain element;
	\item let $\mathrm{Con}(F,\pi^*)$ denote the concepts subsumed by $F$ under $\pi^*$, i.e., $\{ A \mid \standb{u}[ \top \sqsubseteq \standb{s} [F \Rightarrow A] ] \in \KB^\vdash\!, \st \in \sigma^{-1}(\pi^*) \}$
	\item $\Precs_{i+1}$ is obtained from $\Precs_{i}$ by adding a fresh precisification $\pi_{\delta',\standd{s} D}$
	      whenever there is some \mbox{$C \in \mathrm{Con}(F,\pi^*)$} with $\standb{t}[C \sqsubseteq \standd{s} D] \in \KB^\vdash$ for some $\mathsf{t} \in \sigma^{-1}(\pi^*)$
	\item let $\sigma_{i+1}$ be such that $\sigma_{i+1}(\spp)= \sigma_{i}(\spp)\cup \set{\pi_{\delta',\standd{s} D}}$ if  $\spp\in\{\st\} \cup \{\sp \mid \st\preceq\sp \in \KB^\vdash \}$ and  $\sigma_{i+1}(\spp)=\sigma_{i}(\spp)$ otherwise.
	\item for $\pi \in \Precs_{i+1}$ we let $\Lambda_{\pi}(\delta')$ be
	      \begin{align*}
		       & \mathrm{Con}(F,\pi^*)                                                                                                                                                                                             &  & \text{ if } \pi = \pi^*,                              \\[1ex]
		       & \mycup{{G \in \mathrm{Con}(F,\pi^*)},\ {\mathsf{s} \in \sigma^{-1}_{i+1}(\pi)},\ {\mathsf{t} \in \sigma^{-1}_{i+1}(\pi^*)}}\{ A \mid  \standb{t}[G \sqsubseteq \standb{s} [\top \Rightarrow A]] \in \KB^\vdash \} &  & \text{ if } \pi \in \Precs_{i} \setminus \{ \pi^* \}, \\[1ex]
		       & \mycup{{G \in \mathrm{Con}(F,\pi^*)},\ {\mathsf{s} \in \sigma^{-1}_{i+1}(\pi)},\ {\mathsf{t} \in \sigma^{-1}_{i+1}(\pi^*)}}\{ A \mid  \standb{t}[G \sqsubseteq \standb{s} [D \Rightarrow A]] \in \KB^\vdash \}    &  & \text{ if } \pi = \pi_{\delta',\standd{s} D}.         
	      \end{align*}
	\item Let $\gamma_{i+1}$ be the interpretation function defined as follows:
	      \begin{itemize}
		      \item $a^{\gamma_{i+1}(\pi)} = a^{\gamma_{i}(\pi)} = \delta_a$ for each individual name $a$ and each $\pi \in \Precs_{i+1}$.
		      \item for concept names $A$, we let $A^{\gamma_{i+1}(\pi)} = A^{\gamma_{i}(\pi)} \cup \{\delta'\}$ if $A \in \Lambda_{\pi}(\delta')$, and $A^{\gamma_{i+1}(\pi)} = A^{\gamma_{i}(\pi)}$ otherwise.
		      \item for all role names $T$, we obtain $T^{\gamma_{i+1}(\pi)}$ essentially by performing a concurrent saturation process under all applicable RIAs, that is,
		            $$T^{\gamma_{i+1}(\pi)} = \bigcup_{k\in \mathbb{N}} [T^{\gamma_{i+1}(\pi)}]_k,$$ where we let
		            $[T^{\gamma_{i+1}(\pi)}]_0 = \mathit{Self} \cup \mathit{Other}$, with
		            \begin{align*}
			            \mathit{Self} = & \begin{cases}
				                              \{(\delta',\delta')\} & \text{ if } \exists T.\mathsf{Self} \in \Lambda_{\pi}(\delta'), \\
				                              \emptyset             & \text{ otherwise.}
			                              \end{cases}                  \\
			            \mathit{Other}= & \begin{cases}
				                              \emptyset  \qquad\quad \text{ whenever } \pi \in \Precs_{i+1} \setminus \Precs_{i} \\
				                              T^{\gamma_{i}(\pi)}  \quad \text{ for } \pi \in \Precs_{i},
				                              \text{ if } \ T \neq R \text{ or } \pi \neq \pi^*                                  \\
				                              R^{\gamma_{i}(\pi^*)} \cup (\delta^*\!,\delta') \quad \  \text{ if } T = R \text{ and } \pi = \pi^*,
			                              \end{cases}
		            \end{align*}
		            and obtain
		            \begin{align*}
			            [T^{\gamma_{i+1}(\pi)}]_{k+1} \, = \ \  & [T^{\gamma_{i+1}(\pi)}]_{k}
			            \cup \mycup{\hspace{-4ex}{\mathsf{s}\in \sigma^{-1}_{i+1}(\pi)},\ {\standb{s} [ R_0 \sqsubseteq T ] \in \kb^\vdash}} \  [R_0^{\gamma_{i+1}(\pi)}]_{k}                                                                                           \\
			                                                    & \cup \mycup{\hspace{-4ex}{\mathsf{s}\in \sigma^{-1}_{i+1}(\pi)},\ {\standb{s} [ R_1 \circ R_2  \sqsubseteq T ] \in \kb^\vdash}} \  [R_1^{\gamma_{i+1}(\pi)}]_{k} \circ [R_2^{\gamma_{i+1}(\pi)}]_{k}.
		            \end{align*}
	      \end{itemize}
\end{itemize}

\noindent After producing the (potentially infinite) sequence $\dlstruct_0, \dlstruct_1, \ldots$ we obtain the wanted model $\dlstruct$ via
$$\dlstruct = \tuple{\Dom,\Precs,\sigma,\gamma} = \tuple{\bigcup_i\Dom_i,\bigcup_i\Precs_i,\bigcup_i\sigma_i,\bigcup_i\gamma_i}.$$

We then establish that the $\dlstruct$ resulting from this construction indeed is a well-defined structure that satisfies all axioms of $\K$. To this end, an important observation is that for all domain elements $\delta \in \Dom$ and precisifications $\pi \in \Precs$ of $\dlstruct$, it holds that $C \in \Lambda_{\pi}(\delta)$ implies $\delta \in C^{\gamma(\pi)}$. Furthermore, we show that if $\bot \in \Lambda_{\pi}(\delta)$ were to hold for any $\delta \in \Dom$ and $\pi \in \Precs$ (which is the only way the model construction could possibly fail, by declaring an existing domain element to be contradictory), then this would necessarily imply $\K^\vdash \models \incons$. By virtue of these considerations, we arrive at the aspired result.

\begin{theorem}\label{thm:soundcomplete}
	The deduction calculus displayed in \Cref{fig:calculus} is sound and refutation-complete for $\SELp$ knowledge bases.
\end{theorem}

Then, together with our previous insights and \Cref{thm:reducibility}, we can use \Cref{thm:soundcomplete} to  establish tractability of the fundamental standard reasoning tasks in $\SELp$.

\begin{corollary}
	$\SELp$ \textsc{knowledge base satisfiability} and $\SELp$ \textsc{Statement entailment} are \PTime-complete.
\end{corollary}

Thereby, \PTime-hardness follows immediately from the \PTime-hardness of reasoning in plain $\EL$.

\section{Datalog-Based Implementation}
\label{sec:implementation}

We have prototypically implemented our approach in the Datalog-based language \Souffle~\cite{JordanSS16Souffle}.
The prototype's source code is available from our group's github site at \implementationurl.
The implementation currently does not scale well, so optimising both calculus and implementation is an important issue for future work.

\begin{figure*}[t]
    \vspace*{-1ex}
    \begin{flushleft}
        \ttfamily
        \smaller
        \input{sections/implementation-rules-tbox.tex}
    \end{flushleft}
    \vspace*{-2ex}
    \caption{Datalog implementation of the calculus excluding rules dealing with assertions and self-loops.
        Predicates with prefix ``$\mathtt{is\_}$'' declare the vocabulary:
        \mbox{$C\in\Concepts\leadsto\mathtt{is\_cn(C)}$},
        \mbox{$R\in\Roles\leadsto\mathtt{is\_rn(R)}$},
        \mbox{$\sts\in\Stands\leadsto\mathtt{is\_sn(s)}$}, and
        \mbox{$a\in\Individuals\leadsto\mathtt{is\_nom(\set{a})}$}.
        Note that
        deduction rule schema (\genr) needs multiple concrete instantiations.
        The symbols $\mathtt{*}$, $\top$, and $\bot$ are used for readability here;
        in the actual implementation, we use proper (but globally fixed) Datalog constants.
    }
    \label{fig:implementation-rules}
    \vspace{-2ex}
\end{figure*}

\subsubsection{Calculus}
The calculus of \Cref{sec:calculus} is implemented in the pure Datalog fragment of \Souffle's input language.
Following the common approach, as e.g.\ detailed by \citeauthor{Krotzsch10OWLELreasoning}~(\citeyear{Krotzsch10OWLELreasoning}), we introduce a predicate symbol for each possible (normal-form) formula shape as shown in \Cref{fig:datalog-atoms}.
Implementing the deduction calculus then boils down to writing Datalog rules for all deduction rules;
the main rules of the calculus are shown in \Cref{fig:implementation-rules}~(p.\,\pageref{fig:implementation-rules}).
For the axiom schemas (tautologies) we make use of helper predicates that keep track of the vocabulary;
likewise, for nominals and self-loops we use binary predicates to translate back and forth between individual names/nominal concepts, and role names/self-loop concepts, respectively, hence treating nominals and self-loops as ``ordinary'' concept names.

\subsubsection{Normalisation}
For obtaining the normal form of a given $\SELp$ knowledge base in its full expressiveness, we employ several \Souffle features that are not strictly Datalog.
For one, we use algebraic data types\footnote{\smaller\ttfamily\url{https://souffle-lang.github.io/types\#algebraic-data-types-adt}} 
to define term-based encodings of all structured constructs involved in representing knowledge bases, such as concept terms, axioms, formulas, etc., where the base types “standpoint name”, “role name”, and “concept name” are subtypes of the built-in type $\mathtt{symbol}$ (i.e., string).
More importantly, during normalisation we employ \Souffle's built-in functor $\mathtt{cat}$\footnote{\smaller\ttfamily\url{https://souffle-lang.github.io/arguments\#intrinsic-functor}} for concatenating strings to create unambiguous identifiers for newly introduced concept, role, and standpoint names.

\section{Conclusion and Future Work}\label{sec:conclusion}

In this paper, we presented the knowledge representation formalism Standpoint~$\ELp$, which extends the formerly proposed Standpoint~$\EL$ language~\cite{ourijcaisubmission} by a row of new modelling features: role chain axioms and self-loops, extended sharpening statements including standpoint disjointness, negated axioms, and modalised axiom sets.
We designed a deduction calculus that is sound and refutation-complete when applied to appropriately pre-processed $\SELp$ knowledge bases.
As preprocessing and exhaustive deduction rule application are shown to run in \PTime, we thereby established tractability of statisfiability checking of $\SELp$ knowledge bases and -- by virtue of a \PTime Turing reduction -- also tractability of checking the entailment of $\SELp$ statements from $\SELp$ knowledge bases, notably also allowing negated statements.%

We note that, if tractability is to be preserved, the options of further extending the expressivity of $\SELp$ are limited.
Clearly, any modelling feature that would turn the description logic $\EL$ intractable -- atomic negation, disjunction, cardinality restrictions, universal quantification as well as inverse or functional roles \cite{Baader05ELenvelope} -- would also destroy tractability of $\SELp$.
But also the free use of nominal concepts, which is known to still warrant \PTime reasoning when added to $\EL$ with role chain axioms and self-loops, has been shown to be computationally detrimental as soon as standpoints are involved.
The same holds if one allows for the declaration of roles to be rigid or for a more liberal semantics that would admit empty standpoints \cite{ourijcaisubmission}.

Beyond the theoretical advancement, we also believe that the developed deduction calculus can pave the way to practical reasoner implementations by means of Datalog materialisation, a method already proven to be competitive for reasoning in lightweight description logics. In order to demonstrate the principled feasibility of this approach, we implemented a publicly available prototype in \Souffle.

There are numerous avenues for future work. While the calculus is adequate to show our theoretical results and demonstrate feasibility, we are confident that there is much room for improvement when it comes to optimisation.
We expect that refactoring the set of deduction rules can significantly improve performance in practice.
In Datalog terms, it would be beneficial to reduce the number and arity of the predicates involved, the number of variables per rule, and the number of alternative derivations of the same fact. These goals may be in conflict and it is typically not straightforward to find the optimal sweet spot.
In this regard, realistic benchmarks can provide guidance.
While no off-the-shelf standpoint ontologies exist yet, we expect that sensible test cases can be generated from linked open data, ontology alignment settings, or ontology repositories with versioning.%

Likewise, the calculus can be analysed and improved in terms of more comprehensive completeness guarantees;
in fact, we conjecture that it already yields all entailed ``boxed'' assertions and concept inclusions over concept names.

More generally, we will investigate standpoint extensions of other light- or heavyweight ontology languages regarding computational properties and efficient reasoning.




\bibliographystyle{kr}
\bibliography{bib/references}

\clearpage

\onecolumn

\longtrue

\appendix

\section{Proofs for \Cref{sec:syntax-semantics}}

\section{Proofs for \Cref{sec:calculus}}

\def\premise{\Gamma}
\def\conclusion{\theta}
\def\elem{\delta}

\begin{theorem}[Soundness]
    If a $\SELp$ knowledge base $\kb$ is satisfiable, there is a non-contradictory calculus derivation for knowledge base $\kb$.
\end{theorem}

\begin{proof}

    Suppose $\kb$ is satisfiable. Then there is a model $\dlstruct = \tuple{\Dom, \Precs, \sigma, \gamma}$ such that $\dlstruct\models\kb$.

    \noindent To show soundness it suffices to show that for each rule $\premise |\conclusion$, we have $\dlstruct\models\conclusion$ if $\dlstruct\models\premise$. 
    This can be seen easily for both axioms and rules, and hence limit ourselves to show some examples. 

    \begin{itemize}
        \item Axions
              \begin{description}
                  \item[(\axA)] By the definition of $\sigma$, for all $\st\in\SC$ we have $\sigma(\st)\subseteq\sigma(\star)$. Hence, $\dlstruct\models\st\preceq\star$ as desired.
                  \item[(\axB)] Let $\pi\in\Pi$ be a precisification and let $\elem\in\Delta$ be a domain element.
                      Since $B \Rightarrow  C\equiv\neg B \sqcup  C$,
                      if $\elem\in C^{\gamma(\pi)}$ then clearly
                      $\elem\in (C\Rightarrow C)^{\gamma(\pi)}$.
                      Similarly, if $\elem\notin C^{\gamma(\pi)}$
                      then also $\elem\in (C\Rightarrow C)^{\gamma(\pi)}$.
                      Thus clearly $(C\Rightarrow C)^{\gamma(\pi)}=\Delta$,
                      and hence $\dlstruct\models\top \sqsubseteq \standball [ C \Rightarrow C ]$ as desired.
              \end{description}
        \item Rules
              \begin{description}
                  \item[(\hierA)] Assume that $\dlstruct\models{\st \preceq \sp}$ and $ \dlstruct\models \sp \preceq \spp$. Then we have that $\sigma(\st)\subseteq\sigma(\sp)$ and $\sigma(\sp)\subseteq\sigma(\spp)$, thus $\sigma(\st)\subseteq\sigma(\spp)$ and consequently $\dlstruct\models{\st \preceq \spp}$ as desired.
                  \item[(\flatD)] Assume that  $\standb{t} [C \sqsubseteq  \standd{s'} D] $ and $\standb{s'} [ D \sqsubseteq \standd{s} E ]$. If $\dlstruct\models\standb{t} [C \sqsubseteq  \standd{s'} D]$, then for all  $\pr\in\sigma(\mathrm{t})$ if $\delta\in C^{\gamma(\pi)}$ then there is some $\pr'\in\sigma(\st')$ such that $\delta\in D^{\gamma(\pi')}$. Moreover, since $\dlstruct\models\standb{s'} [ D \sqsubseteq \standd{s} E]$ and $\delta\in D^{\gamma(\pi')}$, there is some $\pr''\in\sigma(\st)$ such that $\delta\in E^{\gamma(\pi'')}$, and thus $\dlstruct\models\standb{t} [ C \sqsubseteq \standd{s} E]$ as desired.
              \end{description}
    \end{itemize}

\end{proof}

\begin{theorem}[Completeness]
    If there is a non-contradictory calculus derivation for a $\SELp$ knowledge base $\kb$, then $\kb$ is satisfiable.
\end{theorem}

\subsection*{Model construction}

\newcommand{\hh}[1]{\textcolor{green!70!black}{#1}}

Given a Standpoint \EL knowledge base \KB in normal form, let $\KB^\vdash$ denote the set of axioms obtained by saturating \KB under the above deduction rules.
Assuming $\KB^\vdash$ does not contain $\standball[ \top \sqsubseteq \standball[\top \Rightarrow \bot]]$, we construct a model of $\KB$  in an infinite process: we start from an initialized model and extend it (both by adding domain elements and precisifications) in a stepwise fashion, resulting in a ``monotonic'' sequence of models. The result of the process is the ``limit'' of this sequence, which can be expressed via an infinite union.

\medskip

For the initialization, we choose the standpoint structure
$\dlstruct_0 = \tuple{\Dom_0, \Precs_0, \sigma_0, \gamma_0}$ where:
\begin{itemize}
    \item $\Dom_0$ consists of one element $\delta_a$ for every individual name $a$ mentioned in \KB;
    \item $\Precs_0$ consists of one precisification $\pi_\mathsf{s}$ for every standpoint $\st$ mentioned in \KB (including $*$);
    \item $\sigma_0$ maps each standpoint $\st$ to the set $\{\pi_{\mathsf{s}}\} \cup \{\pi_{\mathsf{s'}} \mid \sp\prec\st \in \KB^\vdash \}$;
    \item $\gamma_0$ maps each $\pi_\mathsf{s}$ to the description logic interpretation $\struct$ over $\Dom_0$, where
          \begin{itemize}
              \item $\interprets{a} = \delta_a$ for each individual name $a$.
              \item $\interprets{A} = \{\delta_a \mid \standb{u}[\{a\} \sqsubseteq \standb{s}[\top \Rightarrow A]] \in \KB^\vdash, \st\in\sigma^{-1}_0(\pi_\mathsf{s})\}$
              \item $\interprets{R} = \{(\delta_a, \delta_b) \mid \standb{s}[R(a,b)] \in \KB^\vdash, \st\in\sigma^{-1}_0(\pi_\mathsf{s})\}$
                    for each role name $R$;
          \end{itemize}
\end{itemize}

\medskip

Note that the obtained structure satisfies all axioms of $\KB$ except for those of the shape $\standb{s}[E \sqsubseteq \exists R.F]$.
This will also be the case for all structures $\dlstruct_1$, $\dlstruct_2$, ... produced in the following.

Moreover, we define $\Lambda_{\pi}(\delta_a)= \{C \mid \standb{u}[\{a\} \sqsubseteq \standb{s}[\top \Rightarrow C]] \in \KB^\vdash, \st\in\sigma^{-1}_0(\pi)\}$.

\medskip

Given such a structure $\dlstruct_i = \tuple{\Dom_i, \Precs_i, \sigma_i, \gamma_i}$, check if it satisfies all axioms of the form $\standb{t}[E \sqsubseteq \exists R.F]$. If so, $\dlstruct_i$ is a model of \KB and we are done. Otherwise, we pick some  $\delta^* \in \Delta_i$ and some $\pi^* \in \Precs_i$ with $ \pi^*\in \sigma(\mathsf{t})$ for which $\delta^* \in E^{\gamma(\pi^*)} \setminus (\exists R.F)^{\gamma(\pi^*)}$ for some (previously unsatisfied) axiom $\standb{t}[E \sqsubseteq \exists R.F]$ from \KB.
Among the admissible pairs $\delta^*$, $\pi^*$, we pick one that is minimal wrt. $\min \{ j \leq i \mid \delta^* \in \Delta_j\} + \min \{ j \leq i \mid \pi^* \in \Pi_j\}$.

Then obtain $\dlstruct_{i+1} = \tuple{\Dom_{i+1}, \Precs_{i+1}, \sigma_{i+1}, \gamma_{i+1}}$ from
$\dlstruct_i = \tuple{\Dom_i, \Precs_i, \sigma_i, \gamma_i}$ as follows:

\begin{itemize}
    \item $\Dom_{i+1} = \Dom_{i} \cup \{\delta'\}$, where $\delta'$ is a fresh domain element;
    \item let $\mathrm{Con}(F,\pi^*) = \{ A \mid \standb{u}[ \top \sqsubseteq \standb{s} [F \Rightarrow A] ] \in \KB^\vdash, \st \in \sigma^{-1}(\pi^*) \}$
    \item $\Precs_{i+1}$ is obtained from $\Precs_{i}$ by adding a fresh precisification $\pi_{\delta',\standd{s} D}$
          whenever there is some $C \in \mathrm{Con}(F,\pi^*)$ with $\standb{t}[C \sqsubseteq \standd{s} D] \in \KB^\vdash$ for some $\mathsf{t} \in \sigma^{-1}(\pi^*)$
    \item let $\sigma_{i+1}$ be such that $\sigma_{i+1}(\spp)= \sigma_{i}(\spp)\cup \set{\pi_{\delta',\standd{s} D}}$ if  $\spp\in\{\st\} \cup \{\sp \mid \st\prec\sp \in \KB^\vdash \}$ and  $\sigma_{i+1}(\spp)= \sigma_{i}(\spp)$ otherwise.
    \item for $\pi \in \Precs_{i+1}$ we let
          $$ \Lambda_{\pi}(\delta') := \begin{cases}
                  \mathrm{Con}(F,\pi^*)                                                                                                                                                                                                            & \text{ if } \pi = \pi^*                                                            \\
                  \displaystyle\bigcup_{{G \in \mathrm{Con}(F,\pi^*)},\ {\mathsf{s} \in \sigma^{-1}_{i+1}(\pi)},\ {\mathsf{t} \in \sigma^{-1}_{i+1}(\pi^*)}}\{ A \mid  \standb{t}[G \sqsubseteq \standb{s} [\top \Rightarrow A]] \in \KB^\vdash \} & \text{ if } \pi \in \Precs_{i} \setminus \{ \pi^* \}                               \\
                  \displaystyle\bigcup_{{G \in \mathrm{Con}(F,\pi^*)},\ {\mathsf{s} \in \sigma^{-1}_{i+1}(\pi)},\ {\mathsf{t} \in \sigma^{-1}_{i+1}(\pi^*)}}\{ A \mid  \standb{t}[G \sqsubseteq \standb{s} [D \Rightarrow A]] \in \KB^\vdash \}    & \text{ if } \pi = \pi_{\delta',\standd{s} D} \in \Precs_{i+1} \setminus \Precs_{i} \\
              \end{cases}$$
    \item Let $\gamma_{i+1}$ be the interpretation function defined as follows:
          \begin{itemize}
              \item $a^{\gamma_{i+1}(\pi)} = a^{\gamma_{i}(\pi)} = \delta_a$ for each individual name $a$ and each $\pi \in \Precs_{i+1}$.
              \item for concept names $A$, we let $A^{\gamma_{i+1}(\pi)} = A^{\gamma_{i}(\pi)} \cup \{\delta'\}$ if $A \in \Lambda_{\pi}(\delta')$, and $A^{\gamma_{i+1}(\pi)} = A^{\gamma_{i}(\pi)}$ otherwise.
              \item for all role names $T$, we obtain $T^{\gamma_{i+1}(\pi)}$ through a concurrent saturation process, that is,
                    $T^{\gamma_{i+1}(\pi)} = \bigcup_{k\in \mathbb{N}} [T^{\gamma_{i+1}(\pi)}]_k$, where we let
                    $[T^{\gamma_{i+1}(\pi)}]_0 = \mathit{Self} \cup \mathit{Other}$, with
                    \begin{align*}
                        \mathit{Self} = & \begin{cases}
                                              \{(\delta',\delta')\} & \text{ if } \exists T.\mathsf{Self} \in \Lambda_{\pi}(\delta'), \\
                                              \emptyset             & \text{otherwise,}
                                          \end{cases}                                                        \\
                        \mathit{Other}= & \begin{cases}
                                              \emptyset                                     & \text{ whenever } \pi \in \Precs_{i+1} \setminus \Precs_{i}                             \\
                                              T^{\gamma_{i}(\pi)}                           & \text{ for } \pi \in \Precs_{i} \text{ whenever } \ T \neq R \text{ or } \pi \neq \pi^* \\
                                              R^{\gamma_{i}(\pi^*)} \cup (\delta^*,\delta') & \text{ if } T = R \text{ and } \pi = \pi^*
                                          \end{cases}
                    \end{align*}
                    Moreover, obtain $$[T^{\gamma_{i+1}(\pi)}]_{k+1} = [T^{\gamma_{i+1}(\pi)}]_{k} \cup \bigcup_{{\mathsf{s}\in \sigma^{-1}_{i+1}(\pi)}\atop{\standb{s} [ R_0 \sqsubseteq T ] \in \kb^\vdash}} [R_0^{\gamma_{i+1}(\pi)}]_{k}
                        \cup \bigcup_{{\mathsf{s}\in \sigma^{-1}_{i+1}(\pi)}\atop{\standb{s} [ R_1 \circ R_2  \sqsubseteq T ] \in \kb^\vdash}} [R_1^{\gamma_{i+1}(\pi)}]_{k} \circ [R_2^{\gamma_{i+1}(\pi)}]_{k}$$
          \end{itemize}
\end{itemize}

\noindent After producing the (potentially infinite) sequence $\dlstruct_0, \dlstruct_1, \ldots$ we obtain the wanted model $\dlstruct$ via
$$\dlstruct = \tuple{\bigcup_i\Dom_i,\bigcup_i\Precs_i,\bigcup_i\sigma_i,\bigcup_i\gamma_i}$$

\newpage

\begin{lemma}\label{lemma:standpoint-nestings}
    Let $\dlstruct_i$ be as described by the model construction, and let $\pi\in\Pi_i$. Then, there is a standpoint $\s{\pi}\in\sigma^{-1}_i(\pi)$ such that for all $\st\in\sigma^{-1}_i(\pi)$ we have $\s{\pi}\prec\st\in \KB^\vdash$.
\end{lemma}

\begin{proof}
    This follows easily by the construction of $\sigma$.
    First, consider the case of $\dlstruct_0$. By definition, $\sigma_0(\st)=\{\pi_{\mathsf{s}}\} \cup \{\pi_{\mathsf{s'}} \mid \sp\prec\st \in \KB^\vdash \}$, and since by $\axE$ we have $\st\prec\st \in \KB^\vdash$, hence the claim holds for $\dlstruct_0$. Now we assume that the claim holds for $\dlstruct_i$ and we show that it also holds for $\dlstruct_{i+1}$. Recall that $\sigma_{i+1}(\st)= \sigma_{i}(\st)\cup \set{\pi_{\delta',\standd{s''} D}}$ if  $\st\in\{\spp\} \cup \{\sp \mid \st\prec\sp \in \KB^\vdash \}$ and  $\sigma_{i+1}(\st)= \sigma_{i}(\st)$ otherwise. Again by $\axE$ we have $\spp\prec\spp \in \KB^\vdash$, hence the claim holds for $\dlstruct_{i+1}$ and thus for $\dlstruct$ as desired.
\end{proof}

\begin{lemma}\label{lemma:sharpest-box}
    Let $\dlstruct_i$ be as described by the model construction, and let $\pi\in\Pi_i$. Then, if $\phi\in\kb^{\vdash}$, $\pi\in\sigma_i(\st)$ and $\standbs$ is a symbol in $\phi$, then the formula $\phi^{\pi}$ replacing some occurrences of $\standbs$ by $\standb{\s{\pi}}$ is also in the saturated set, $\phi^{\pi}\in\kb^{\vdash}$.
\end{lemma}

\begin{proof}
    By \Cref{lemma:standpoint-nestings} and $\pi\in\sigma_i(\st)$, we have that $\s{\pi}\prec\st\in \KB^\vdash$. Then it is easy to see that for $\phi$ having any formula shape with a $\standbs$, we can apply rule \genr and rule \inhierB and obtain precisely that each formula $\phi^{\pi}\in\kb^{\vdash}$, with $\phi^{\pi}$ being any formula resulting from replacing one or more occurrences of $\standbs$ by $\standb{\s{\pi}}$.
\end{proof}

\begin{lemma}\label{lemma:self-loops}
    Let $\dlstruct_i$ be as described by the model construction. Then $A\in\Lambda_{\pi}(\delta)$ iff $\delta\in A^{\gamma_i(\pi)}$ and $\,\exists R.\Self\in\Lambda_{\pi}(\delta)$ iff  $(\delta,\delta)\in R^{\gamma_i(\pi)}$.
\end{lemma}

\begin{proof}
    For the concept names the case is trivial since the definition coincides. Thus, we show the case of self-loops by induction on the model construction.
    \begin{description}
        \item[Base Case]: Consider $\exself\in\Lambda_{\pi}(\delta)$ of $\dlstruct_0$.  $(\Rightarrow)$ If $\exself\in\Lambda_{\pi}(\delta)$ then $\standb{u}[\{a\} \sqsubseteq \standb{s}[\top \Rightarrow \exself]] \in \KB^\vdash$ for $\pi\in\sigma_0(\st)$. Then we obtain $\standbs[R(a,a)]\in \KB^\vdash$ by rule \selfF, and thus by construction $(\delta,\delta)\in R^{\gamma_0(\pi)}$. $(\Leftarrow)$ If $(\delta,\delta)\in R^{\gamma_0(\pi)}$ then by construction we have some $\standbs[R(a,a)]\in \KB^\vdash$ with $\pi\in\sigma_0(\st)$. Then by the rule \selfE we obtain $\standb{u}[\{a\} \sqsubseteq \standb{s}[\top \Rightarrow \exself]] \in \KB^\vdash$ and hence $\exself\in\Lambda_{\pi}(\delta)$ as desired.

        \item[Inductive Step]: 
        By the inductive hypothesis, assume that the statement holds for $\dlstruct_{i}$ and we show that it also holds for $\dlstruct_{i+1}$.

        Consider $\exself\in\Lambda_{\pi}(\delta')$ of $\dlstruct_{i+1}$. $(\Rightarrow)$ If $\exself\in\Lambda_{\pi}(\delta')$ then by construction $(\delta',\delta')\in R^{\gamma_{i+1}(\pi)}$, and we notice that for $\delta\in\Delta_{i}\setminus\Delta_{i+1}$ the same holds by induction. $(\Leftarrow)$ This is the harder case. We show that if $(\delta',\delta')\in R^{\gamma_{i+1}(\pi)}$ then $\exself\in\Lambda_{\pi}(\delta')$. We observe the saturation process that defines $R^{\gamma_{i+1}(\pi)}$ and we notice that for every pair $(\delta_1,\delta_2)\in[R^{\gamma_{i+1}(\pi)}]_{k+1}$ and $(\delta_1,\delta_2)\notin[R^{\gamma_{i+1}(\pi)}]_{k}$, we have $\delta'=\delta_2$. With regards to self-loops, we will ext show  that, since $\delta'$ is fresh, relations of the shape $(\delta',\delta')\in[R^{\gamma_{i+1}(\pi)}]_{k+1}$ can exclusively be triggered by relations $(\delta',\delta')\in[R'^{\gamma_{i+1}(\pi)}]_{0}$, i.e. elements in the set $\mathit{Self}$.
        \begin{itemize}
            \item By definition if $(\delta',\delta')\in [R^{\gamma_{i+1}(\pi)}]_0$ then necessarily $\exself\in\Lambda_{\pi}(\delta')$ since $\delta'$ is fresh.
            \item Let $(\delta',\delta')\notin[R^{\gamma_{i+1}(\pi)}]_{k}$ and $(\delta',\delta')\in [R^{\gamma_{i+1}(\pi)}]_{k+1}$. Then either
                  \begin{description}
                      \item[Case 1] $(\delta',\delta')\in\{[R_0^{\gamma_{i+1}(\pi)}]_{k} \mid  \standb{s'} [ R_0 \sqsubseteq R ] \in \kb^\vdash,\ \sp \in\sigma^{-1}_{i+1}(\pi)\}$ Thus, by \Cref{lemma:sharpest-box}, there must be some $\standb{\s{\pi}} [ R_0 \sqsubseteq R ] \in \kb^\vdash$  such that $(\delta',\delta')\in[R_0^{\gamma_{i+1}(\pi)}]_{k}$.
                          \begin{description}
                              \item[Case 1.1] Assume $\pi = \pi^*$ and $\pi\in\sigma(\st)$
                                  \begin{enumerate}[label={(\arabic*)}]
                                      \item $\standb{u} [\top\sqsubseteq \standb{\s{\pi}} [F \Rightarrow \exists R_0.\Self]] \in \KB^\vdash$, by the assumption that $\exists R_0.\Self\in\Lambda_\pi(\delta')$
                                      \item $\standb{\s{\pi}} [ R_0 \sqsubseteq R ] \in \kb^\vdash$ by the assumption
                                      \item $\standball[\top \sqsubseteq \standb{\s{\pi}}	[\exists R_0.\Self \Rightarrow \exself ]]\in \KB^\vdash$, by (2) and rule \selfA
                                      \item $\standb{u} [\top\sqsubseteq \standb{\s{\pi}} [F \Rightarrow \exists R.\Self]] \in \KB^\vdash$, by (1), (3) and rule \subA, thus $\exists R.\Self\in \Lambda_{\pi}(\delta')$ as desired
                                  \end{enumerate}
                              \item[Case 1.2]  Assume $\pi \in \Precs_{i} \setminus \{ \pi^* \}$ and $\pi\in\sigma(\st)$. Since $\exists R_0.\Self\in \Lambda_{\pi}(\delta')$ then
                                  \begin{enumerate}[label={(\arabic*)}]
                                      \item $\standb{\s{\pi^*}}[G \sqsubseteq \standb{\s{\pi}} [\top \Rightarrow \exists R_0.\Self] ]\in \KB^\vdash$ by \Cref{lemma:sharpest-box} for some ${G \in \mathrm{Con}(F,\pi^*)}$
                                      \item $\standb{\s{\pi}} [ R_0 \sqsubseteq R ] \in \kb^\vdash$ by the assumption
                                      \item $\standball[\top \sqsubseteq \standb{\s{\pi}}	[\exists R_0.\Self \Rightarrow \exself ]]\in \KB^\vdash$, by (2) and rule \selfA
                                      \item $\standball[\top \sqsubseteq \standball [G\Rightarrow \top ]]\in \KB^\vdash$ by axiom \axC
                                      \item $\standball[G \sqsubseteq \standb{\s{\pi}}[\exists R_0.\Self \Rightarrow \exself ]]\in \KB^\vdash$ by (3), (4) and rule \subB
                                      \item $\standb{\s{\pi^*}}[G \sqsubseteq \standb{\s{\pi}}[\exists R_0.\Self \Rightarrow \exself ]]\in \KB^\vdash$ by (7) and \Cref{lemma:sharpest-box}
                                      \item $\standb{\s{\pi^*}}[G \sqsubseteq \standb{\s{\pi}} [\top \Rightarrow \exists R.\Self] ]\in \KB^\vdash$, by (1), (8) and rule \subA, thus $\exists R.\Self\in \Lambda_{\pi}(\delta')$ as desired
                                  \end{enumerate}
                              \item[Case 1.3]  Assume $ \pi = \pi_{\delta',\standd{s} E}$ and $\pi\in\sigma(\st)$. Since $\exists R_0.\Self\in \Lambda_{\pi}(\delta')$ then
                                  \begin{enumerate}[label={(\arabic*)}]
                                      \item $\standb{\s{\pi^*}}[G \sqsubseteq \standb{\s{\pi}} [E \Rightarrow \exists R_0.\Self] ]\in \KB^\vdash$ by \Cref{lemma:sharpest-box} for some ${G \in \mathrm{Con}(F,\pi^*)}$
                                      \item $\standb{\s{\pi}} [ R_0 \sqsubseteq R ] \in \kb^\vdash$ by the assumption
                                      \item $\standball[\top \sqsubseteq \standb{\s{\pi}}	[\exists R_0.\Self \Rightarrow \exself ]]\in \KB^\vdash$, by (2) and rule \selfA
                                      \item $\standball[\top \sqsubseteq \standball [G\Rightarrow \top ]]\in \KB^\vdash$ by axiom \axC
                                      \item $\standball[G \sqsubseteq \standb{\s{\pi}}[\exists R_0.\Self \Rightarrow \exself ]]\in \KB^\vdash$ by (3), (4) and rule \subB
                                      \item $\standb{\s{\pi^*}}[G \sqsubseteq \standb{\s{\pi}}[\exists R_0.\Self \Rightarrow \exself ]]\in \KB^\vdash$ by (7) and \Cref{lemma:sharpest-box}
                                      \item $\standb{\s{\pi^*}}[G \sqsubseteq \standb{\s{\pi}} [E \Rightarrow \exists R.\Self] ]\in \KB^\vdash$, by (1), (8) and rule \subA, thus $\exists R.\Self\in \Lambda_{\pi}(\delta')$ as desired
                                  \end{enumerate}
                          \end{description}

                      \item[Case 2] $(\delta',\delta')\in\{[R_1^{\gamma_{i+1}(\pi)}]_{k} \circ [R_2^{\gamma_{i+1}(\pi)}]_{k} \mid \standb{s'} [ R_1 \circ R_2  \sqsubseteq R ] \in \kb^\vdash,\ \mathsf{s'}\in \sigma^{-1}_{i+1}(\pi)\}$. Thus, by \Cref{lemma:sharpest-box}, there must be some $\standb{\s{\pi}}[ R_1 \circ R_2  \sqsubseteq R ] \in \kb^\vdash$  such that $(\delta',\delta')\in[R_1^{\gamma_{i+1}(\pi)}]_{k}$ and $(\delta',\delta')\in[R_2^{\gamma_{i+1}(\pi)}]_{k}$.

                          \begin{description}
                              \item[Case 2.1] Assume $\pi = \pi^*$ and $\pi\in\sigma(\st)$
                                  \begin{enumerate}[label={(\arabic*)}]
                                      \item $\standball [\top\sqsubseteq \standb{\s{\pi}} [F \Rightarrow \exists R_1.\Self]] \in \KB^\vdash$, by the assumption that $\exists R_1.\Self\in\Lambda_\pi(\delta')$ and rule \globtop
                                      \item $\standball [\top\sqsubseteq \standb{\s{\pi}} [F \Rightarrow \exists R_2.\Self]] \in \KB^\vdash$, by the assumption that $\exists R_2.\Self\in\Lambda_\pi(\delta')$ and rule \globtop
                                      \item $\standb{\s{\pi}} [ R_1 \circ R_2  \sqsubseteq R ] \in \kb^\vdash$ by the assumption
                                      \item $\standb{\s{\pi}}	[\exists R_1.\Self \sqcap \exists R_2.\Self \Rightarrow \exself ]\in \KB^\vdash$, by (3) and rule \selfB
                                      \item $\standball [\top\sqsubseteq \standb{\s{\pi}} [F \Rightarrow \exists R.\Self]] \in \KB^\vdash$, by (1), (2), (4) and rule \con, thus $\exists R.\Self\in \Lambda_{\pi}(\delta')$ as desired
                                  \end{enumerate}
                              \item[Case 2.2]  Assume $\pi \in \Precs_{i} \setminus \{ \pi^* \}$ and $\pi\in\sigma(\st)$. Since $\exists R_0.\Self\in \Lambda_{\pi}(\delta')$ then
                                  \begin{enumerate}[label={(\arabic*)}]
                                      \item $\standb{\s{\pi^*}}[G_1 \sqsubseteq \standb{\s{\pi}} [\top \Rightarrow \exists R_1.\Self] ]\in \KB^\vdash$ by \Cref{lemma:sharpest-box} for some ${G \in \mathrm{Con}(F,\pi^*)}$
                                      \item $\standb{\s{\pi^*}}[G_2 \sqsubseteq \standb{\s{\pi}} [\top \Rightarrow \exists R_2.\Self] ]\in \KB^\vdash$ by \Cref{lemma:sharpest-box} for some ${G \in \mathrm{Con}(F,\pi^*)}$
                                      \item $\standb{\s{\pi}} [ R_1 \circ R_2  \sqsubseteq R ]\in \kb^\vdash$ by the assumption
                                      \item $\standb{\s{\pi}}	[\exists R_1.\Self \sqcap \exists R_2.\Self \Rightarrow \exself ]\in \KB^\vdash$, by (3) and rule \selfB
                                      \item $\standb{u_1} [\top\sqsubseteq \standb{\s{\pi^*}} [F \Rightarrow G_1]]\in \KB^\vdash$ by the construction of $\mathrm{Con}(F,\pi^*)$ and \Cref{lemma:sharpest-box}
                                      \item $\standb{u_2} [\top\sqsubseteq \standb{\s{\pi^*}} [F \Rightarrow G_2]] \in \KB^\vdash$ by the construction of $\mathrm{Con}(F,\pi^*)$ and \Cref{lemma:sharpest-box}
                                      \item $\standb{\s{\pi^*}}[F\sqsubseteq \standb{\s{\pi}} [\top \Rightarrow \exists R_1.\Self] ]\in \KB^\vdash$ by (1), (5) and rule \subB
                                      \item $\standb{\s{\pi^*}}[F\sqsubseteq \standb{\s{\pi}} [\top \Rightarrow \exists R_2.\Self] ]\in \KB^\vdash$ by (2), (6) and rule \subB
                                      \item $\standb{\s{\pi^*}}[F \sqsubseteq \standb{\s{\pi}} [\top \Rightarrow \exists R.\Self] ]\in \KB^\vdash$, by (4), (7), (8) and rule \con, thus $\exists R.\Self\in \Lambda_{\pi}(\delta')$ as desired
                                  \end{enumerate}
                              \item[Case 2.3]  Assume $ \pi = \pi_{\delta',\standd{s} E}$ and $\pi\in\sigma(\st)$. Since $\exists R_0.\Self\in \Lambda_{\pi}(\delta')$ then
                                  \begin{enumerate}[label={(\arabic*)}]
                                      \item $\standb{\s{\pi^*}}[G_1 \sqsubseteq \standb{\s{\pi}} [E \Rightarrow \exists R_1.\Self] ]\in \KB^\vdash$ by \Cref{lemma:sharpest-box} for some ${G \in \mathrm{Con}(F,\pi^*)}$
                                      \item $\standb{\s{\pi^*}}[G_2 \sqsubseteq \standb{\s{\pi}} [E \Rightarrow \exists R_2.\Self] ]\in \KB^\vdash$ by \Cref{lemma:sharpest-box} for some ${G \in \mathrm{Con}(F,\pi^*)}$
                                      \item $\standb{\s{\pi}} [ R_1 \circ R_2  \sqsubseteq R ]\in \kb^\vdash$ by the assumption
                                      \item $\standb{\s{\pi}}	[\exists R_1.\Self \sqcap \exists R_2.\Self \Rightarrow \exself ]\in \KB^\vdash$, by (3) and rule \selfB
                                      \item $\standb{u_1} [\top\sqsubseteq \standb{\s{\pi^*}} [F \Rightarrow G_1]]\in \KB^\vdash$ by the construction of $\mathrm{Con}(F,\pi^*)$ and \Cref{lemma:sharpest-box}
                                      \item $\standb{u_2} [\top\sqsubseteq \standb{\s{\pi^*}} [F \Rightarrow G_2]] \in \KB^\vdash$ by the construction of $\mathrm{Con}(F,\pi^*)$ and \Cref{lemma:sharpest-box}
                                      \item $\standb{\s{\pi^*}}[F\sqsubseteq \standb{\s{\pi}} [E \Rightarrow \exists R_1.\Self] ]\in \KB^\vdash$ by (1), (5) and rule \subB
                                      \item $\standb{\s{\pi^*}}[F\sqsubseteq \standb{\s{\pi}} [E \Rightarrow \exists R_2.\Self] ]\in \KB^\vdash$ by (2), (6) and rule \subB
                                      \item $\standb{\s{\pi^*}}[F \sqsubseteq \standb{\s{\pi}} [E \Rightarrow \exists R.\Self] ]\in \KB^\vdash$, by (4), (7), (8) and rule \con, thus $\exists R.\Self\in \Lambda_{\pi}(\delta')$ as desired
                                  \end{enumerate}
                          \end{description}
                  \end{description}

        \end{itemize}
    \end{description}

\end{proof}

\begin{proof}[Completeness proof]
    In order to show completeness, we will show that if the model construction fails, then necessarily the saturated knowledge base contains the innconsistency axiom, $\incons\in\kb^\vdash$. First, we will use the previously constructed model to  show that for all $\phi\in\kb$ then $\dlstruct\models\phi$. This is proved by induction for all types of axioms allowed in the normal form. By \Cref{lemma:self-loops}, we use  $C\in\Lambda_{\pi}(\delta)$ in place of $\delta\in C^{\gamma_i(\pi)}$ for $C$ a concept name and of $(\delta,\delta)\in R^{\gamma_i(\pi)}$ when $C$ is of the form $\exself$.

    \begin{numberclaim}\label{claim:compl-claim}
        Let $\kb$ be a $\SELp$ knowledge base and let $\dlstruct$ be the model obtained by the model construction. Then for all $\phi\in\kb$ we have $\dlstruct\models\phi$.
    \end{numberclaim}

    \begin{description}[leftmargin=0.3cm, itemindent=-0.1cm, itemsep=0.5cm]
        \item[$\phi=\st\preceq\sp$] Assume $\st\preceq\sp\in\kb$. Then it is clear from the definition of $\sigma_0$ that by construction $\dlstruct_0\models\st\preceq\sp$. Now, we assume that $\dlstruct_i\models\st\preceq\sp$ and show that $\dlstruct_{i+1}\models\st\preceq\sp$.
            \begin{enumerate}
                \item Notice that we have $\sigma_i(\st)\subseteq\sigma_i(\sp)$ and $\sigma_{i+1}$ is such that for all $\pi\in\Pi_{i}$ and standpoint $\spp$, $\pi\in\sigma_{i+1}(\spp)$ if $\pi\in\sigma_{i}(\spp)$.
                \item So we must only show that if $\pi\in\Pi_{i+1}\setminus\Pi_{i}$ and $\pi\in\sigma_{i+1}(\st)$ then we must have $\pi\in\sigma_{i+1}(\sp)$.
                \item By the construction of $\sigma_{i+1}$, if $\pi\in\Pi_{i+1}\setminus\Pi_{i}$, then there is a standpoint $\spp$ for which $\pi$ has been created such that either $\st=\spp$ or   $\st\prec\spp \in \KB^\vdash$.
                \item If $\st=\spp$, then $\pi\in\sigma_{i+1}(s)$ as desired, from the definition of $\sigma_{i+1}$ and from $\st\preceq\sp\in\KB^\vdash$.
                \item Otherwise, since $\KB^\vdash$ is saturated under the deduction rules and by the assumption, we have $\st\preceq\sp\in\KB^\vdash$, then by the non-applicability of the rule \hierA, we must also have $\sp\preceq\spp\in\KB^\vdash$, which again by the definition of $\sigma_{i+1}$ implies that $\pi\in\sigma_{i+1}(\sp)$ as desired.
                \item Thus, necessarily $\sigma_{i+1}(\st)\subseteq\sigma_{i+1}(\sp)$ and $\dlstruct_{i+1}\models\st\preceq\sp$.
                \item Finally, from the definition of $\dlstruct$ we obtain $\dlstruct\models\st\preceq\sp$ as desired.
            \end{enumerate}

        \item[$\phi=\st_{1} \cap \st_{2}\preceq\sp$] Assume $\st_{1} \cap \st_{2}\preceq\sp\in\kb$. Let us first show that $\dlstruct_0\models\st_{1} \cap \st_{2}\preceq\sp$. 
            \begin{description}
                \item[Case 1:] Assume that $\sigma_0(\st_{1})\cap\sigma_0(\st_{2})=\emptyset$. Then by construction there is no $\st$ such that $\st\prec\st_{1} \in \KB^\vdash$ and $\st\prec\st_{2} \in \KB^\vdash$, and $\dlstruct_0\models\st_{1} \cap \st_{2}\preceq\sp$ is trivially satisfied.
                \item[Case 2:] Assume that $\sigma_0(\st_{1})\cap\sigma_0(\st_{2})\neq\emptyset$. Then, by the construction of $\sigma_0$, if $\pi\in\sigma_0(\st_{1})\cap\sigma_0(\st_{2})$ then there must be some $\st$ such that $\pi\in\sigma_0(\st)$ and both $\st\prec\st_{1} \in \KB^\vdash$ and  $\st\prec\st_{2} \in \KB^\vdash$, thus $\pi\in\sigma_0(\st_{1})$ and $\pi\in\sigma_0(\st_{2})$.
                    Then, by the non-applicability of the rules \hierA and \hierB, also $\st\prec\sp \in \KB^\vdash$, and thus also $\pi\in\sigma_0(\sp)$. Hence, $\dlstruct_0\models\st_{1} \cap \st_{2}\preceq\sp$ as desired.
            \end{description}
            We now assume that $\dlstruct_i\models\st_{1} \cap \st_{2}\preceq\sp$ and, in a similar way, we show that then $\dlstruct_{i+1}\models\st_{1} \cap \st_{2}\preceq\sp$.
            \begin{description}
                \item[Case 1:] Assume that $\sigma_{i+1}(\st_{1})\cap\sigma_{i+1}(\st_{2})=\emptyset$. Then by construction there is no $\st$ such that $\st\prec\st_{1} \in \KB^\vdash$ and $\st\prec\st_{2} \in \KB^\vdash$, and $\dlstruct_{i+1}\models\st_{1} \cap \st_{2}\preceq\sp$ is trivially satisfied.
                \item[Case 2:] Assume that $\sigma_{i+1}(\st_{1})\cap\sigma_{i+1}(\st_{2})\neq\emptyset$. Then, by the construction of $\sigma_{i+1}$, if $\pi\in\sigma_{i+1}(\st_{1})\cap\sigma_{i+1}(\st_{2})$ then either (a) $\pi\in\Pi_{i+1}\cap\Pi_{i}$ and thus $\pi\in\sigma_i(\sp)$ by the inductive hypothesis, giving us $\pi\in\sigma_{i+1}(\sp)$ as required or (b) $\pi\in\Pi_{i+1}\setminus\Pi_{i}$, and hence there must be some $\st$ such that $\pi\in\sigma_{i+1}(\st)$ and both $\st\prec\st_{1} \in \KB^\vdash$ and $\st\prec\st_{2} \in \KB^\vdash$, thus $\pi\in\sigma_{i+1}(\st_{1})$ and $\pi\in\sigma_0(\st_{2})$. Then, again by the non-applicability of the rules \hierA and \hierB, also $\st\prec\sp \in \KB^\vdash$, and thus also $\pi\in\sigma_{i+1}(\sp)$. Hence, $\dlstruct_{i+1}\models\st_{1} \cap \st_{2}\preceq\sp$ as desired.
            \end{description}
            Finally, from the definition of $\dlstruct$ we obtain $\dlstruct\models\st_{1} \cap \st_{2}\preceq\sp$ as desired.

            \item[$\phi=\standbs {[R \sqsubseteq R']} $] Assume $\standbs [R \sqsubseteq R'] \in\KB^\vdash$.
            \begin{description}
                \item[Base Case]: First, let us show that $\dlstruct_0\models\standbs [R \sqsubseteq R']$. We let $\pi\in\sigma_0(\st)$ and $(\delta,\delta')\in R^{\gamma_{0}(\pi)}$.
                \begin{enumerate}[label=(\arabic*)]
                    \item $\standb{\s{\pi}}[R(a,b)]\in\KB^\vdash$ by construction and \Cref{lemma:sharpest-box}
                    \item $\standb{\s{\pi}} [R \sqsubseteq R'] \in\KB^\vdash$ by assumption and \Cref{lemma:sharpest-box}
                    \item $\standb{\s{\pi}}[R'(a,b)]\in\KB^\vdash$ by (1), (2) and rule \abB
                \end{enumerate}
                From this we obtain that if $\pi\in\sigma_0(\st)$ and $(\delta,\delta')\in R^{\gamma_{0}(\pi)}$ then $(\delta,\delta')\in R'^{\gamma_{0}(\pi)}$, thus $\dlstruct_0, \pi\models R \sqsubseteq R' $ for all $\pi\in\sigma_0(\st)$ and hence $\dlstruct_0\models \standbs [R \sqsubseteq R']$ by the semantics.
                \item[Inductive Step]: Now, we assume that $\dlstruct_i\models\standbs [R \sqsubseteq R'] $ and show that $\dlstruct_{i+1}\models\standbs [R \sqsubseteq R']$.
                We must show that if $(\delta,\delta')\in R^{\gamma_{i+1}(\pi)}$ then $(\delta,\delta')\in R'^{\gamma_{i+1}(\pi)}$ for all $\pi\in\sigma_{i+1}(\st)$.
                From the concurrent saturation process that determines $R'^{\gamma_{i+1}(\pi)}$, we precisely have that $(\delta,\delta')\in [R'^{\gamma_{i+1}(\pi)}]_{k+1}$ if $(\delta,\delta')\in [R^{\gamma_{i+1}(\pi)}]_{k}$, $\standbs [R \sqsubseteq R']\in\KB^\vdash$ and $\pi\in\sigma(\st)$, and hence it easily follows that $\dlstruct_{i+1}\models\standbs [R \sqsubseteq R']$ as desired.

            \end{description}
            \item[$\phi=\standbs{[R_{1} \circ R_{2}\sqsubseteq R']}$] Assume $\standbs[R_{1} \circ R_{2}\sqsubseteq R'] \in\KB$.
            \begin{description}
                \item[Base Case]: First, let us show that $\dlstruct_0\models\standbs[R_{1} \circ R_{2}\sqsubseteq R']$. We let $\pi\in\sigma_0(\st)$, $(\delta,\delta')\in R_{1}^{\gamma_{0}(\pi)}$ and $(\delta',\delta'')\in R_{2}^{\gamma_{0}(\pi)}$.
                \begin{enumerate}[label=(\arabic*)]
                    \item $\standb{\s{\pi}} (R_{1} \circ R_{2}\sqsubseteq R') \in\KB^\vdash$ and $\pi\in\sigma_0(\st)$ by assumption and \Cref{lemma:sharpest-box}
                    \item $\standb{\s{\pi}}[R(a,b)]\in\KB^\vdash$ by construction and \Cref{lemma:sharpest-box}
                    \item $\standb{\s{\pi}}[R(b,c)]\in\KB^\vdash$ by construction and \Cref{lemma:sharpest-box}
                    \item $\standb{\s{\pi}}[R'(a,c)]\in\KB^\vdash$ by (1), (2), (3) and rule \abC
                \end{enumerate}
                From this we obtain that if $\pi\in\sigma_0(\st)$, $(\delta,\delta')\in R_{1}^{\gamma_{0}(\pi)}$ and $(\delta',\delta'')\in R_{2}^{\gamma_{0}(\pi)}$ then $(\delta,\delta'')\in R'^{\gamma_{0}(\pi)}$. Thus $\dlstruct_0, \pi\models R_{1} \circ R_{2}\sqsubseteq R' $ for all $\pi\in\sigma_0(\st)$ and hence $\dlstruct_0\models \standbs[R_{1} \circ R_{2}\sqsubseteq R']$ by the semantics.
                \item[Inductive Step]: Now, we assume that $\dlstruct_i\models\standbs ( R_{1} \circ R_{2} \sqsubseteq R') $ and show that $\dlstruct_{i+1}\models\standbs [R_{1} \circ R_{2} \sqsubseteq R']$.
                We must show that if $(\delta,\delta')\in R_{1}^{\gamma_{i+1}(\pi)}$ and $(\delta',\delta'')\in R_{2}^{\gamma_{i+1}(\pi)}$ then $(\delta,\delta'')\in R'^{\gamma_{i+1}(\pi)}$ for all $\pi\in\sigma_{i+1}(\st)$.
                From the concurrent saturation process that determines $R_1^{\gamma_{i+1}(\pi)}, R_2^{\gamma_{i+1}(\pi)}$ and $R'^{\gamma_{i+1}(\pi)}$ and the assumption, we have that for some $k < m$ $(\delta,\delta')\in [R_{1}^{\gamma_{i+1}(\pi)}]_{k}$ and for some $l < m$
                $(\delta',\delta'')\in [R_{2}^{\gamma_{i+1}(\pi)}]_{l}$, thus also $(\delta,\delta')\in [R_{1}^{\gamma_{i+1}(\pi)}]_{m-1}$ and $(\delta',\delta'')\in [R_{2}^{\gamma_{i+1}(\pi)}]_{m-1}$. Hence we also have $(\delta,\delta'')\in [R'^{\gamma_{i+1}(\pi)}]_{m}$ and finally  $(\delta,\delta'')\in R'^{\gamma_{i+1}(\pi)}$ as desired.
            \end{description}
            \item[$\phi=\standbs{[C\sqsubseteq D]}$]  Assume $\standbs[C\sqsubseteq D]\in\KB$ (i.e. $\standb{u}[\top \sqsubseteq \standb{s}[C \Rightarrow D]] \in \KB^\vdash$ and thus also $\standball[\top \sqsubseteq \standb{s}[C \Rightarrow D]] \in \KB^\vdash$ due to rule \globtop).
            \begin{description}
                \item[Base Case]: First, let us show that $\dlstruct_0\models\standbs[C\sqsubseteq D]$. We let $\pi\in\sigma_0(\st)$ and $C\in\Lambda_{\pi}(\delta)$.
                \begin{enumerate}[label=(\arabic*)]
                    \item $\standball[\top \sqsubseteq \standb{s}[C \Rightarrow D]]\in \KB^\vdash$, by assumption
                    \item $\standb{u}[\{a\} \sqsubseteq \standb{\s{\pi}}[\top \Rightarrow C] ]\in \KB^\vdash$, by construction, \Cref{lemma:sharpest-box} and the assumption that $C\in\Lambda_{\pi}(\delta)$ \label{proof:1}
                    \item $\standball[\top \sqsubseteq \standb{\s{\pi}}[C \Rightarrow D] ]\in \KB^\vdash$ by (1) and \Cref{lemma:sharpest-box}
                    \item $\allstandb[\{a\} \sqsubseteq \standb{\s{\pi}}[C \Rightarrow D] ]\in \KB^\vdash$ by (3) and rule \abAA
                    \item $\standb{u}[\{a\} \sqsubseteq \standb{\s{\pi}}[C \Rightarrow D] ]\in \KB^\vdash$ by (4) and rule \inhierB
                    \item $\standb{u}[\{a\} \sqsubseteq \standb{\s{\pi}}[\top \Rightarrow D]] \in \KB^\vdash$ by (2), (5) rule \subA, and thus $D\in\Lambda_{\pi}(\delta)$ as required
                \end{enumerate}
                From this we obtain that if $\pi\in\sigma_0(\st)$ and $C\in\Lambda_{\pi}(\delta)$ then $D\in\Lambda_{\pi}(\delta)$, thus $\dlstruct_0, \pi\models C\sqsubseteq D$ for all $\pi\in\sigma_0(\st)$ and hence $\dlstruct_0\models \standbs[C\sqsubseteq D]$ by the semantics.

                \item[Inductive Step]: Now, we assume that $\dlstruct_i\models\standbs[C\sqsubseteq D]$ and show that $\dlstruct_{i+1}\models\standbs[C\sqsubseteq D]$. Consider $\delta\in\Delta_{i+1}\setminus\Delta_i$ (the case of $\delta\in\Delta_i$ is trivial). Thus, $\delta$ has been introduced to satisfy some axiom $\standb{t}[E \sqsubseteq \exists R.F]\in\KB^{\vdash}$ with $\stand{t}\in\sigma_{i+1}(\pi^*)$. We must show that if $C\in\Lambda_{\pi}(\delta)$ then $D\in\Lambda_{\pi}(\delta)$ for all $\pi\in\sigma_{i+1}(\st)$. We assume $C\in \Lambda_{\pi}(\delta)$ and consider the three cases.
                \begin{description}
                    \item[Case 1\namedlabel{proof:3}{(Case 1)}]: Assume $\pi = \pi^*$ and $\pi\in\sigma(\st)$. Since $C\in \Lambda_{\pi}(\delta)$ and by rule \globtop and \Cref{lemma:sharpest-box}, we have $\standball [\top\sqsubseteq \standb{\s{\pi}} [F \Rightarrow C]] \in \KB^\vdash$. By assumption $ \standball[\top\sqsubseteq \standb{\s{\pi}} [C \Rightarrow D]] \in \KB^\vdash$. Then by rule \subA we obtain $\standball [\top\sqsubseteq \standb{\s{\pi}} [F \Rightarrow D]] \in \KB^\vdash$, and thus $D\in \Lambda_{\pi}(\delta')$ as required.
                    \item[Case 2\namedlabel{proof:4}{(Case 2)}]: Assume $\pi \in \Precs_{i} \setminus \{ \pi^* \}$ and $\pi\in\sigma(\st)$. Since $C\in \Lambda_{\pi}(\delta')$ then for some ${G \in \mathrm{Con}(F,\pi^*)}$ there is $\standb{\s{\pi^*}}[G \sqsubseteq \standb{\s{\pi}} [\top \Rightarrow C] ]\in \KB^\vdash$ by \Cref{lemma:sharpest-box}. By assumption, $\standball[\top \sqsubseteq \standb{\s{\pi}}[C\Rightarrow D]] \in \KB^\vdash$.
                    \begin{enumerate}[label=(\arabic*)]
                        \item $\standball[\top \sqsubseteq \standb{\s{\pi}}[C \Rightarrow D]] \in \KB^\vdash$, by assumption
                        \item $\standb{\s{\pi^*}}[G \sqsubseteq \standb{\s{\pi}} [\top \Rightarrow C] ]\in \KB^\vdash$, by assumption
                        \item $\standb{\s{\pi}}[\top \sqsubseteq \standball	[ C \Rightarrow \top ]]\in \KB^\vdash$, by axiom \axC and rule \inhierF
                        \item $\allstandb[C \sqsubseteq \standb{\s{\pi}} [C \Rightarrow D ]]\in \KB^\vdash$, by (1), (3) and axiom \subB
                        \item $\standb{\s{\pi^*}}[G \sqsubseteq \standb{\s{\pi}} [C \Rightarrow D ]]\in \KB^\vdash$, by (2), (4) and axiom \flatA
                        \item $\standb{\s{\pi^*}}[G \sqsubseteq \standb{\s{\pi}} [\top \Rightarrow D ]]\in \KB^\vdash$, by (2), (5) and axiom \subA. Thus we obtain $D\in \Lambda_{\pi}(\delta)$ as desired
                    \end{enumerate}
                    \item [Case 3\namedlabel{proof:5}{(Case 3)}]: Assume $ \pi = \pi_{\delta',\standd{s} E}$ and $\pi\in\sigma(\st)$. Since $C\in \Lambda_{\pi}(\delta')$
                          then for some ${G \in \mathrm{Con}(F,\pi^*)}$ by \Cref{lemma:sharpest-box} there is
                          $\standb{\s{\pi^*}}[G \sqsubseteq \standb{\s{\pi}} [E \Rightarrow C] ]\in \KB^\vdash$. Recall that, by assumption,
                          $\standball[\top \sqsubseteq \standb{\s{\pi}}[C\Rightarrow D] ]\in \KB^\vdash$.
                          \begin{enumerate}[label=(\arabic*)]
                              \item $\standball[\top \sqsubseteq \standb{\s{\pi}}[C \Rightarrow D]] \in \KB^\vdash$, by assumption
                              \item $\standb{\s{\pi^*}}[G \sqsubseteq \standb{\s{\pi}} [E \Rightarrow C] ]\in \KB^\vdash$, by assumption
                              \item $\allstandb[\top \sqsubseteq \standb{\s{\pi^*}}	[ G \Rightarrow \top ]]\in \KB^\vdash$, by axiom \axC and \inhierB
                              \item $\standb{\s{\pi^*}}[G \sqsubseteq \standb{\s{\pi}} [C \Rightarrow D ]]\in \KB^\vdash$, by (1), (3) and axiom \subB
                              \item $\standb{\s{\pi^*}}[G \sqsubseteq \standb{\s{\pi}} [E \Rightarrow D ]]\in \KB^\vdash$, by (2), (4) and axiom \subA. Thus we obtain $D\in \Lambda_{\pi}(\delta)$ as desired
                          \end{enumerate}
                \end{description}
            \end{description}

            \item[$\phi=\standbs{[C_1 \sqcap C_2 \sqsubseteq  D]}$] Let us now show that if $\standbs[C_1 \sqcap C_2 \sqsubseteq  D]\in\KB$ then  $\dlstruct\models \standbs[C_1 \sqcap C_2 \sqsubseteq  D]$.
            \begin{description}
                \item[Base Case]: First, let us show that $\dlstruct_0\models\standb{s}[C_1 \sqcap C_2  \sqsubseteq  D]$. Let $\pi\in\Pi$, $C_1\in\Lambda_{\pi}(\delta)$, $C_2\in\Lambda_{\pi}(\delta)$ and $\delta=a^{\gamma_{0}}$. We need to show that $D\in\Lambda_{\pi}(\delta)$
                \begin{enumerate}[label=(\arabic*)]
                    \item $\standball[\{a\} \sqsubseteq \standb{\s{\pi}}[\top \Rightarrow C_1]] \in \KB^\vdash$ by construction, rule \locB, and \Cref{lemma:sharpest-box}
                    \item $\standball[\{a\} \sqsubseteq \standb{\s{\pi}}[\top \Rightarrow C_2]] \in \KB^\vdash$  by construction, rule \locB, and \Cref{lemma:sharpest-box}
                    \item $\standb{\s{\pi}}[C_1 \sqcap C_2  \sqsubseteq  D]$  by assumption and \Cref{lemma:sharpest-box}
                    \item $\standball[\{a\} \sqsubseteq \standb{\s{\pi}}[\top \Rightarrow D]] \in \KB^\vdash$ by (1), (2), (3) and rule \con, thus we obtain $D\in \Lambda_{\pi}(\delta)$ as desired
                \end{enumerate}
                \item[Inductive Step]: Now, we assume that $\dlstruct_i\models\standbs[C_1 \sqcap C_2 \sqsubseteq  D]$ and show that $\dlstruct_{i+1}\models\standbs[C_1 \sqcap C_2 \sqsubseteq  D]$. Consider $\delta\in\Delta_{i+1}\setminus\Delta_i$ (the case of $\delta\in\Delta_i$ is trivial). Thus, $\delta$ has been introduced to satisfy some axiom $\standb{t}[E \sqsubseteq \exists R.F]\in\KB^{\vdash}$ with $\stand{t}\in\sigma_{i+1}(\pi^*)$. We must show that if $C_1\in\Lambda_{\pi}(\delta)$ and $C_2\in\Lambda_{\pi}(\delta)$ then $D\in \Lambda_{\pi}(\delta)$ for all $\pi\in\sigma_{i+1}(\st)$. We assume $C_1,C_2\in \Lambda_{\pi}(\delta)$ and consider the three cases.
                \begin{description}
                    \item[Case 1]: Assume $\pi = \pi^*$ and $\pi\in\sigma_{i+1}(\st)$. Since $C_1,C_2\in \Lambda_{\pi}(\delta)$ then (by construction and rule \locB) $\standball [\top\sqsubseteq \standb{\s{\pi}} [F \Rightarrow C_1]], \standball [\top\sqsubseteq \standb{\s{\pi}} [F \Rightarrow C_2]] \in \KB^\vdash$ by \Cref{lemma:sharpest-box} and also $\standb{\s{\pi}}[C_1 \sqcap C_2 \sqsubseteq  D]\in \KB^\vdash$. Then by rule \con we obtain $\standball [\top\sqsubseteq \standb{\s{\pi}} [F \Rightarrow D]] \in \KB^\vdash$, and thus $D\in \Lambda_{\pi}(\delta')$ as required.
                    \item[Case 2]: Assume $\pi \in \Precs_{i} \setminus \{ \pi^* \}$ and $\pi\in\sigma_{i+1}(\st)$. Since $C_1,C_2\in \Lambda_{\pi}(\delta')$ then for some ${G_1,G_2 \in \mathrm{Con}(F,\pi^*)}$ and by \Cref{lemma:sharpest-box}, there are $\standb{\s{\pi^*}}[G_1 \sqsubseteq \standb{\s{\pi}} [\top \Rightarrow C_1]],\standb{\s{\pi^*}}[G_2 \sqsubseteq \standb{\s{\pi}} [\top \Rightarrow C_2]] \in \KB^\vdash$:
                    \begin{enumerate}[label=(\arabic*)]
                        \item $\standb{\s{\pi}}[C_1 \sqcap C_2  \sqsubseteq  D]\in \KB^\vdash$,
                        \item $\standb{\s{\pi^*}}[G_1\sqsubseteq \standb{\s{\pi}} [\top \Rightarrow C_1] ]\in \KB^\vdash$
                        \item $\standb{\s{\pi^*}}[G_2\sqsubseteq \standb{\s{\pi}} [\top \Rightarrow C_2] ]\in \KB^\vdash$
                        \item $\standb{u_1} [\top\sqsubseteq \standb{\s{\pi^*}} [F \Rightarrow G_1]]\in \KB^\vdash$ by the construction of $\mathrm{Con}(F,\pi^*)$ and \Cref{lemma:sharpest-box}
                        \item $\standb{u_2} [\top\sqsubseteq \standb{\s{\pi^*}} [F \Rightarrow G_2]] \in \KB^\vdash$ by the construction of $\mathrm{Con}(F,\pi^*)$ and \Cref{lemma:sharpest-box}
                        \item $\standb{\s{\pi^*}}[F\sqsubseteq \standb{\s{\pi}} [\top \Rightarrow C_1] ]\in \KB^\vdash$ by (2), (4) and rule \subB
                        \item $\standb{\s{\pi^*}}[F\sqsubseteq \standb{\s{\pi}} [\top \Rightarrow C_2] ]\in \KB^\vdash$ by (3), (5) and rule \subB
                        \item $\standb{\s{\pi^*}}[F\sqsubseteq \standb{\s{\pi}} [\top \Rightarrow D] ]\in \KB^\vdash$ by (1), (6), (7) and rule \con, thus $D\in \Lambda_{\pi}(\delta')$ as desired
                    \end{enumerate}
                    \item [Case 3]: Assume $ \pi = \pi_{\delta',\standd{s} E} \in \Precs_{i+1} \setminus \Precs_{i} $  and $\pi\in\sigma_{i+1}(\st)$. Since $C_1,C_2\in \Lambda_{\pi}(\delta')$ then for some ${G_1,G_2 \in \mathrm{Con}(F,\pi^*)}$ and by \Cref{lemma:sharpest-box}, there are $\standb{\s{\pi^*}}[G_1 \sqsubseteq \standb{\s{\pi}} [E \Rightarrow C_1]],\standb{\s{\pi^*}}[G_2 \sqsubseteq \standb{\s{\pi}} [E \Rightarrow C_2]] \in \KB^\vdash$
                          \begin{enumerate}[label=(\arabic*)]
                              \item $\standb{\s{\pi}}[C_1 \sqcap C_2  \sqsubseteq  D]\in \KB^\vdash$,
                              \item $\standb{\s{\pi^*}}[G_1\sqsubseteq \standb{\s{\pi}} [E \Rightarrow C_1] ]\in \KB^\vdash$
                              \item $\standb{\s{\pi^*}}[G_2\sqsubseteq \standb{\s{\pi}} [E \Rightarrow C_2] ]\in \KB^\vdash$
                              \item $\standb{u_1} [\top\sqsubseteq \standb{\s{\pi^*}} [F \Rightarrow G_1]]\in \KB^\vdash$ by the construction of $\mathrm{Con}(F,\pi^*)$ and \Cref{lemma:sharpest-box}
                              \item $\standb{u_2} [\top\sqsubseteq \standb{\s{\pi^*}} [F \Rightarrow G_2]] \in \KB^\vdash$ by the construction of $\mathrm{Con}(F,\pi^*)$ and \Cref{lemma:sharpest-box}
                              \item $\standb{\s{\pi^*}}[F\sqsubseteq \standb{\s{\pi}} [E \Rightarrow C_1] ]\in \KB^\vdash$ by (2), (4) and rule \subB
                              \item $\standb{\s{\pi^*}}[F\sqsubseteq \standb{\s{\pi}} [E \Rightarrow C_2] ]\in \KB^\vdash$ by (3), (5) and rule \subB
                              \item $\standb{\s{\pi^*}}[F\sqsubseteq \standb{\s{\pi}} [E \Rightarrow D] ]\in \KB^\vdash$ by (6), (7) and rule \con, thus $D\in \Lambda_{\pi}(\delta')$ as desired
                          \end{enumerate}
                \end{description}

            \end{description}

            \item[$\phi=\standbsp{[C\sqsubseteq \standbs D]}$] Assume $\standbsp[C\sqsubseteq \standbs D]\in\KB$ (i.e. $\standbsp [C \sqsubseteq \standb{s}[\top \Rightarrow D]] \in \KB^\vdash$).
            \begin{description}
                \item[Base Case]: First, let us show that $\dlstruct_0\models \standbsp[C \sqsubseteq \standb{s}[\top \Rightarrow D]]$. We let $C\in \Lambda_{\pi'}(\delta)$ for $\pi'\in\sigma_0(\sp)$, and we show that $D\in \Lambda_{\pi}(\delta)$ for $\pi\in\sigma_0(\st)$.
                \begin{enumerate}[label=(\arabic*)]
                    \item $\standb{\s{\pi'}} [C \sqsubseteq \standb{\s{\pi}}[\top \Rightarrow D]] \in \KB^\vdash$, by assumption\label{proof:9} and \Cref{lemma:sharpest-box}
                    \item $\standb{u}[\{a\} \sqsubseteq \standb{\s{\pi'}}[\top \Rightarrow C]] \in \KB^\vdash$ for some $\sp$ such that $\pi\in\sigma_0(\sp)$, by construction and the assumption that $C\in \Lambda_{\pi'}(\delta)$\label{proof:7}
                    \item $\standb{u}[\{a\} \sqsubseteq \standb{\s{\pi}}[\top \Rightarrow D]] \in \KB^\vdash$ by rule \flatA, (1) and (2), thus $D\in \Lambda_{\pi}(\delta')$ as desired
                \end{enumerate}

                From this we obtain that if $\pi\in\Pi$ then $\dlstruct_0\models  \standbsp [C \sqsubseteq \standb{s}[\top \Rightarrow D]]$ by the semantics.
                \item[Inductive Step]: Now, we assume that $\dlstruct_i\models \standbsp [C \sqsubseteq \standb{s}[\top \Rightarrow D]]$ and show that $\dlstruct_{i+1}\models\standbsp [C \sqsubseteq \standb{s}[\top \Rightarrow D]]$. Consider $\delta\in\Delta_{i+1}\setminus\Delta_i$ (the case of $\delta\in\Delta_i$ is trivial). Thus, $\delta$ has been introduced to satisfy some axiom $\standb{t}[E \sqsubseteq \exists R.F]\in\KB^{\vdash}$ with $\stand{t}\in\sigma_{i+1}(\pi^*)$. We must show that if $C\in\Lambda_{\pi'}(\delta)$ then $D\in\Lambda_{\pi}(\delta)$ for all $\pi\in\sigma_{i+1}(\sp)$ and $\pi'\in\sigma_{i+1}(\st)$. We consider the three cases.
                \begin{description}
                    \item[Case 1\namedlabel{proof:3}{(Case 1)}]: Assume $\pi' = \pi^*$, hence $\sp,t\in\sigma^{-1}_{i+1}(\pi')$ and $C\in \Lambda_{\pi'}(\delta)$.
                    \begin{enumerate}[label=(\arabic*)]
                        \item $\standball[\top\sqsubseteq \standb{\s{\pi}} [F \Rightarrow C]] \in \KB^\vdash$ from $C\in \Lambda_{\pi}(\delta)$, rule XXX and \Cref{lemma:sharpest-box}
                        \item $\standb{\s{\pi'}} [C \sqsubseteq \standb{\s{\pi}}[\top \Rightarrow D]]\in \KB^\vdash$, by assumption and \Cref{lemma:sharpest-box}, hence $D\in \Lambda_{\pi}(\delta)$ if $\pi=\Pi\setminus\pi^*$ as desired
                        \item $\allstandb[\top \sqsubseteq \standball	[ C \Rightarrow \top ]]\in \KB^\vdash$ by axiom \axC
                        \item $\allstandb[\top \sqsubseteq \standball	[ E \Rightarrow \top ]]\in \KB^\vdash$ by axiom \axC
                        \item $\allstandb[C \sqsubseteq \standball	[ E \Rightarrow \top ]]\in \KB^\vdash$ by (3), (4) and rule \subB
                        \item $\standb{\s{\pi'}}[C \sqsubseteq \standb{\s{\pi}} [ E \Rightarrow \top ]]\in \KB^\vdash$ by (5) \Cref{lemma:sharpest-box}
                        \item $\standb{\s{\pi'}}[C \sqsubseteq \standb{\s{\pi}}[E \Rightarrow D] ]\in \KB^\vdash$, from (2), (6) and rule \subA, hence $D\in \Lambda_{\pi}(\delta)$ if $\pi=\pi_{\delta',\standd{s_1} E}$ as desired
                        \item $\allstandb[\top \sqsubseteq \standb{\s{\pi}}[C \Rightarrow D]] \in \KB^\vdash$, by (2) and rules \loc
                        \item $\standb{\s{\pi'}}[\top \sqsubseteq \standb{\s{\pi}}[C \Rightarrow D]] \in \KB^\vdash$, by (8) and \Cref{lemma:sharpest-box}
                        \item $\standb{\s{\pi'}}[\top \sqsubseteq \standb{\s{\pi}}[F \Rightarrow D]] \in \KB^\vdash$, by (1), (9) and rule \subA, hence $D\in \Lambda_{\pi}(\delta)$ if $\pi=\pi^*$ as desired
                    \end{enumerate}
                    \item[Case 2\namedlabel{proof:4}{(Case 2)}]: Assume $\pi' \in \Precs_{i} \setminus \{ \pi^* \}$, and $C\in \Lambda_{\pi'}(\delta)$
                    \begin{enumerate}[label=(\arabic*)]
                        \item $\standb{\s{\pi^*}} [G \sqsubseteq \standb{\s{\pi'}} [\top \Rightarrow C] ]\in \KB^\vdash$ from $C\in \Lambda_{\pi'}(\delta)$ for some ${G \in \mathrm{Con}(F,\pi^*)}$ by construction
                        \item $\standb{\s{\pi'}} [C \sqsubseteq \standb{\s{\pi}}[\top \Rightarrow D]]\in \KB^\vdash$, by assumption and \Cref{lemma:sharpest-box}
                        \item $\standb{\s{\pi^*}} [G \sqsubseteq \standb{\s{\pi}}[\top \Rightarrow D]] \in \KB^\vdash$, by (1), (2) and rule \flatA, hence $D\in \Lambda_{\pi}(\delta)$ if $\pi=\Pi\setminus\pi^*$ as desired
                        \item $\standball	[\top \sqsubseteq \standball	[ G \Rightarrow \top ]]\in \KB^\vdash$ by axiom \axC
                        \item $\standball	[\top \sqsubseteq \standball	[ E \Rightarrow \top ]]\in \KB^\vdash$ by axiom \axC
                        \item $\standball	[G \sqsubseteq \standball	[ E \Rightarrow \top ]]\in \KB^\vdash$ by (4), (5) and rule \subB
                        \item $\standb{\s{\pi'}}[G \sqsubseteq \standb{\s{\pi}} [ E \Rightarrow \top ]]\in \KB^\vdash$ by (6) and \Cref{lemma:sharpest-box}
                        \item $\standb{\s{\pi'}} [G \sqsubseteq \standb{\s{\pi}}[E \Rightarrow D] ]\in \KB^\vdash$, from (3), (7) and rule \subA, hence $D\in \Lambda_{\pi}(\delta)$ if $\pi=\pi_{\delta',\standd{s_1} E}$ as desired
                        \item $\standb{u}[\top \sqsubseteq \standb{\s{\pi^*}} [F \Rightarrow G]] \in \KB^\vdash$ from 1 by construction and \Cref{lemma:sharpest-box}
                        \item $\standb{\s{\pi^*}}[F \sqsubseteq \standb{\s{\pi}} [\top \Rightarrow D]] \in \KB^\vdash$ from (3), (9) and rule \subB
                        \item $\allstandb[ \top \sqsubseteq \standb{\s{\pi}}[F \Rightarrow D]] \in \KB^\vdash$, by (10) and rules \loc, hence $D\in \Lambda_{\pi}(\delta)$ if $\pi=\pi^*$ as desired
                    \end{enumerate}

                    \item [Case 3\namedlabel{proof:5}{(Case 3)}]: Assume $ \pi' = \pi_{\delta',\standd{s_1} E} \in \Precs_{i+1} \setminus \Precs_{i} $ and $C\in \Lambda_{\pi'}(\delta)$
                          \begin{enumerate}[label=(\arabic*)]
                              \item $\allstandb[C \sqsubseteq \standb{\s{\pi}}[\top \Rightarrow D]] \in \KB^\vdash$ by assumption
                              \item $\standb{\s{\pi^*}} [J \sqsubseteq \exists R.F]\in \KB^\vdash$ by construction from the existential that triggered the iteration $\dlstruct_{i+1}$ and \Cref{lemma:sharpest-box}
                              \item $\standb{u_1} [\top \sqsubseteq \standb{\s{\pi^*}} [F \Rightarrow G]]\in \KB^\vdash$ by construction for $G\in\mathrm{Con}(F,\pi^*)$ and \Cref{lemma:sharpest-box}
                              \item $\standb{u_2} [\top \sqsubseteq \standb{\s{\pi^*}} [F \Rightarrow H]]\in \KB^\vdash$ by construction for $H\in\mathrm{Con}(F,\pi^*)$ and \Cref{lemma:sharpest-box}
                              \item $\standb{\s{\pi^*}}[H \sqsubseteq \standd{\s{\pi'}} E] \in \KB^\vdash$ by construction and \Cref{lemma:sharpest-box}
                              \item $\standb{\s{\pi^*}}[G \sqsubseteq \standb{\s{\pi'}} [E \Rightarrow C]] \in \KB^\vdash$
                                    by construction and \Cref{lemma:sharpest-box}
                              \item $\allstandb[\top \sqsubseteq \standb{\s{\pi}}[C \Rightarrow D]] \in \KB^\vdash$, by (1), \Cref{lemma:sharpest-box} and rule \loc
                              \item $\allstandb[\top \sqsubseteq \allstandb[G \Rightarrow \top]] \in \KB^\vdash$, by axiom \axC
                              \item $\allstandb[G \sqsubseteq \standb{\s{\pi}}[C \Rightarrow D]] \in \KB^\vdash$, by (7), (8) and rule \subB
                              \item if $\pi=\pi'$, $\standb{\s{\pi^*}}[G \sqsubseteq \standb{\s{\pi}}[E \Rightarrow D] ]\in \KB^\vdash$ from (6), (9), \Cref{lemma:sharpest-box} and rule \subA, hence $D\in \Lambda_{\pi}(\delta)$ if $\pi=\pi_{\delta',\standd{s_1} E}$ as desired
                              \item $\allstandb [C \sqsubseteq \standb{\s{\pi}}[\top \Rightarrow D]] \in \KB^\vdash$, from (1) and rule \inhierB.
                              \item $\standb{\s{\pi^*}}[F \sqsubseteq \standd{\s{\pi'}} E] \in \KB^\vdash$ from (4), (5) and rule \subC
                              \item $\standb{\s{\pi^*}}[G \sqsubseteq \standb{\s{\pi'}} [E \Rightarrow C]] \in \KB^\vdash$ from (6), (12) and rule \inhierF
                              \item $\standb{\s{\pi^*}}[F \sqsubseteq \standb{\s{\pi'}} [E \Rightarrow C]] \in \KB^\vdash$ from (3), (6) and rule \subB
                              \item $\standb{\s{\pi^*}}[F \sqsubseteq \standd{\s{\pi'}} C] \in \KB^\vdash$ from (12), (14) and rule \subD
                              \item $\standb{\s{\pi^*}} [F \sqsubseteq \standb{\s{\pi}}[\top \Rightarrow D]] \in \KB^\vdash$, from (11), (15) and rule \flatC, hence $D\in \Lambda_{\pi}(\delta)$ if $\pi=\Pi\setminus\pi^*$ as desired
                              \item $\standb{\s{\pi'}} [C \sqsubseteq \standb{\s{\pi^*}}[\top \Rightarrow D]] \in \KB^\vdash$, from (11) and \Cref{lemma:sharpest-box}
                              \item if $\pi=\pi^*$, $\allstandb[\top\sqsubseteq\standb{\s{\pi^*}}[F\Rightarrow D]]$, by (17) and rule \loc, hence $D\in \Lambda_{\pi}(\delta)$ if $\pi=\pi^*$ as desired
                          \end{enumerate}
                \end{description}
            \end{description}

            \item[$\phi=\standbsp{[C\sqsubseteq \standds D]}$] Assume $\standbsp[C\sqsubseteq \standds D]\in\KB$
            \begin{description}
                \item[Base Case]: First, let us show that $\dlstruct_0\models \standbsp[C\sqsubseteq \standds D]$. Let $\pi'\in\sigma_0(\sp)$ and $C\in\Lambda_{\pi'}(\delta)$.
                \begin{enumerate}[label=(\arabic*)]
                    \item $\standb{\s{\pi'}}[C\sqsubseteq \standds D]\in\KB$ by assumption and \Cref{lemma:sharpest-box}
                    \item $\standb{u} [\{a\} \sqsubseteq \standb{\s{\pi'}}[\top \Rightarrow C]] \in \KB^\vdash$, by construction and \Cref{lemma:sharpest-box}
                          \item$\mathsf{s}[a,D] \prec \mathsf{s}$     by construction
                    \item $\standball [ \{a\} \sqsubseteq \standb{s} [ D \Rightarrow P_{\mathsf{s},a,D}] ]$ by construction
                    \item $\standball [ P_{\mathsf{s},a,D} \sqsubseteq {\standb{s}}_{[a,D]} [ \top \Rightarrow D ]]$ by construction
                    \item $\standball [\{a\} \sqsubseteq \standb{\s{\pi'}}[\top \Rightarrow C]] \in \KB^\vdash$, by (2) and rule \locB.
                    \item $\standb{\s{\pi'}} [\{a\} \sqsubseteq \standb{\s{\pi'}}[\top \Rightarrow C]] \in \KB^\vdash$, by (6) and rule \inhierF
                    \item $\standb{\s{\pi'}} [\top \sqsubseteq \standb{\s{\pi'}}[\{a\} \Rightarrow C]] \in \KB^\vdash$, by (7) and rule \loc
                    \item $\standb{\s{\pi'}}[\{a\} \sqsubseteq \standds D]\in\KB$ by (8) and rule \subC
                    \item $\standb{\s{\pi'}} [ \{a\} \sqsubseteq \standb{s} [ D \Rightarrow P_{\mathsf{s},a,D}] ]$ by (4) and rule \inhierF
                    \item $\standb{\s{\pi'}}[\{a\} \sqsubseteq \standds P_{\mathsf{s},a,D}]\in\KB$ by (10) and rule \subD
                    \item $\standbs [ P_{\mathsf{s},a,D} \sqsubseteq {\standb{s}}_{[a,D]} [ \top \Rightarrow D ]]$ by (5) and rule \inhierF
                    \item $\standb{\s{\pi'}} [  \{a\} \sqsubseteq {\standb{s}}_{[a,D]} [ \top \Rightarrow D ]]$ by (11), (12) and rule \flatC, as desired
                \end{enumerate}
                From this we obtain that if $\pi'\in\sigma_0(\sp)$ and $C\in\Lambda_{\pi'}(\delta)$, then $D\in\Lambda_{\pi}(\delta)$ for $\pi\in\sigma_0(\mathsf{s}[a,D])$ and hence $\pi\in\sigma_0(\st)$, thus $\dlstruct_0, \pi\models C\sqsubseteq \standds D$ for all $\pi\in\sigma_0(\sp)$ and hence $\dlstruct_0\models  \standbsp[C\sqsubseteq \standds D]$ by the semantics.
                \medskip

                \item[Inductive Step]: Now, we assume that $\dlstruct_i\models \standbsp[C\sqsubseteq \standds D]$ and show that $\dlstruct_{i+1}\models \standbsp[C\sqsubseteq \standds D]$. Consider $\delta\in\Delta_{i+1}\setminus\Delta_i$ (the case of $\delta\in\Delta_i$ is trivial). Thus, $\delta$ has been introduced to satisfy some axiom $\standb{t}[E \sqsubseteq \exists R.F]\in\KB^{\vdash}$. We must show that if $C\in\Lambda_{\pi'}(\delta)$ then $D\in\Lambda_{\pi}(\delta)$ for some $\pi\in\sigma_{i+1}(\st)$, specifically for $\pi=\pi_{\delta',\standds D}$. We consider the three cases.
                \begin{description}
                    \item[Case 1]: Assume $\pi' = \pi^*$, $\pi\in\sigma_{i+1}(\sp)$ and $C\in \Lambda_{\pi'}(\delta)$.
                    \begin{enumerate}[label=(\arabic*)]
                        \item $\standb{u}[\top\sqsubseteq \standb{\s{\pi'}} [F \Rightarrow C]] \in \KB^\vdash$ from $C\in \Lambda_{\pi}(\delta)$ and \Cref{lemma:sharpest-box}
                        \item $\standb{\s{\pi'}}[C\sqsubseteq \standds D]$, by assumption and \Cref{lemma:sharpest-box}
                        \item $\standball[ \top \sqsubseteq \standball [ D \Rightarrow  D ]]$ by axiom \axB
                        \item $\standball[ \top \sqsubseteq \standball [ C \Rightarrow  \top ]]$ by axiom \axC
                        \item $\standball[ C \sqsubseteq \standball [ D \Rightarrow  D ]]$ by (3), (4) and rule \subB
                        \item By (1), (2) and by construction, there is a precisification $\pi_{\delta',\standds D}\in\pi\in\sigma_{i+1}(\st)$ and by (5) we have $D\in \Lambda_{\pi_{\delta',\standds D}}(\delta)$ as desired.
                    \end{enumerate}
                    \item[Case 2]: Assume $\pi' \in \Precs_{i} \setminus \{ \pi^* \}$, $\pi\in\sigma_{i+1}(\sp)$ and $C\in \Lambda_{\pi'}(\delta)$
                    \begin{enumerate}[label=(\arabic*)]
                        \item $\standb{\s{\pi^*}}[G \sqsubseteq \standb{\s{\pi'}} [\top \Rightarrow C]] \in \KB^\vdash$ for some ${G \in \mathrm{Con}(F,\pi^*)}$ by construction and \Cref{lemma:sharpest-box}.
                        \item $\standb{\s{\pi'}}[C\sqsubseteq \standds D]$, by assumption
                        \item $\standb{\s{\pi^*}}[G\sqsubseteq \standds D]$, by (1), (2) and rule \flatB
                        \item From (1) and (3), by construction there is a precisification $\pi_{\delta',\standds D}\in\pi\in\sigma_{i+1}(\st)$, and by \textbf{Case 1} we have $D\in \Lambda_{\pi_{\delta',\standds D}}(\delta)$ as desired
                    \end{enumerate}
                    \item [Case 3]: Assume $ \pi' = \pi_{\delta',\standd{s_1} E} \in \Precs_{i+1} \setminus \Precs_{i} $, $\pi'\in\sigma_{i+1}(\st')$ and $C\in \Lambda_{\pi'}(\delta)$
                          \begin{enumerate}[label=(\arabic*)]
                              \item $\standb{\s{\pi^*}}[G \sqsubseteq \standb{\s{\pi'}} [E \Rightarrow C]] \in \KB^\vdash$ for some ${G \in \mathrm{Con}(F,\pi^*)}$
                                    by construction and \Cref{lemma:sharpest-box}
                              \item $\standb{\s{\pi'}}[C\sqsubseteq \standds D]$, by assumption and \Cref{lemma:sharpest-box}
                              \item $\standb{\s{\pi^*}}[H\sqsubseteq \standd{\s{\pi'}} E]$ for some ${H \in \mathrm{Con}(F,\pi^*)}$
                                    by construction and \Cref{lemma:sharpest-box}
                              \item $\standb{u_1} [\top \sqsubseteq \standb{\s{\pi^*}} [F \Rightarrow G]]\in \KB^\vdash$ by construction since $G\in\mathrm{Con}(F,\pi^*)$) and \Cref{lemma:sharpest-box}
                              \item $\standb{u_2} [\top \sqsubseteq \standb{\s{\pi^*}} [F \Rightarrow H]]\in \KB^\vdash$ by construction since $H\in\mathrm{Con}(F,\pi^*)$) and \Cref{lemma:sharpest-box}
                              \item $\standb{\s{\pi^*}}[F \sqsubseteq \standd{\s{\pi'}} E] \in \KB^\vdash$ from (5), (3) and rule \subC
                              \item $\standb{\s{\pi^*}}[F \sqsubseteq \standb{\s{\pi'}} [E \Rightarrow C]] \in \KB^\vdash$ from (4), (1) and rule \subB
                              \item $\standb{\s{\pi^*}}[F \sqsubseteq \standd{\s{\pi'}} C] \in \KB^\vdash$ from (6), (7) and rule (\subD)
                              \item $\standb{\s{\pi^*}}[F \sqsubseteq \standd{s} D] \in \KB^\vdash$ from (2), (8) and rule (\flatD)
                              \item From (9) by construction there is a precisification $\pi_{\delta',\standds D}\in\sigma_{i+1}(\st)$, and by \textbf{Case 1} we have $D\in \Lambda_{\pi_{\delta',\standds D}}(\delta)$ as desired.

                          \end{enumerate}
                \end{description}
                Consequently, from \textbf{Cases 1-3}, for all $\pi\in\sigma_{i+1}(\sp)$, if $C\in \Lambda_{\pi}(\delta)$ then there is some $\pi_{\delta',\standds D}\in\sigma_{i+1}(\st)$ with $D\in \Lambda_{\pi_{\delta',\standds D}}(\delta)$, hence $\dlstruct_0\models \standbsp[C\sqsubseteq \standds D]$ and by induction also $\dlstruct\models \standbsp[C\sqsubseteq \standds D]$.
            \end{description}

            \item[$\phi=\standbs{[C\sqsubseteq \exists R.D]}$] (notice that $D\neq \mathrm{Self}$)
            Let us now show that if $\standbs[C\sqsubseteq \exists R.D]\in\KB^{\vdash}$ then  $\dlstruct\models \standbs[C\sqsubseteq \exists R.D]$. That is, we must show that if $C\in\Lambda_{\pi}(\delta)$ for some $\pi\in\sigma(\st)$ then there exists $\delta'$ such that $(\delta,\delta')\in R^{\gamma(\pi)}$ and $D\in\Lambda_{\pi}(\delta')$. Assume that for some $\dlstruct_i$ we have $C\in\Lambda_{\pi}(\delta)$ for $\pi\in\sigma_{i}(\st)$ but there is no $\delta'$ such that $(\delta,\delta')\in R^{\gamma_{i}(\pi)}$ and $D\in\Lambda_{\pi}(\delta')$. Then for some iteration $k>i$ we will pick $\pi=\pi^*$ and $\delta=\delta^*$ to produce $\dlstruct_{i+1}$ and set $\Dom_{i+1} = \Dom_{i} \cup \{\delta'\}$, where $\delta'$ is a fresh domain element. By the construction, $\Lambda_{\pi}(\delta')=\{ A \mid
                \standb{u}[ \top \sqsubseteq \standb{s} [D \Rightarrow A] ] \in \KB^\vdash, \st \in \sigma^{-1}(\pi^*) \}$. By axiom \axB we have $ \standball[ \top \sqsubseteq \standball [ D \Rightarrow  D ]]$, hence $D\in \Lambda_{\pi}(\delta')$ as desired. We must now show that also $(\delta,\delta')\in R^{\gamma_{i+1}(\pi)}$. This easily follows since $[R^{\gamma_{i+1}(\pi)}]_0= \mathsf{Self}\cup\mathsf{Other}$ and $(\delta,\delta')\in \mathsf{Other}$. Thus $(\delta,\delta')\in R^{\gamma_{i+1}(\pi)}$ as required.

            \item[$\phi=\standbs{[\exists R.C\sqsubseteq D]}$]
            Let us show that if $\standbs[\exists R.C\sqsubseteq D]\in\KB^{\vdash}$ then  $\dlstruct_i\models \standbs[\exists R.C\sqsubseteq D]$. That is, we must show that if $C\in\Lambda_{\pi}(\delta_2)$ and $(\delta_1,\delta_2)\in R^{\gamma_i(\pi)}$ for $\pi\in\sigma_{i}(\st)$, then $D\in\Lambda_{\pi}(\delta_1)$.
            \begin{description}
                \item[Remark 1]: If there is a sequence $\standb{t_1}[R_1\sqsubseteq R_2], \dots, \standb{t_j}[R_j\sqsubseteq R_{j+1}]\in\KB^{\vdash}$ for $j\geq0$, such that $\standb{t_{j+1}}[R_{j+1}\sqsubseteq R]\in\KB^{\vdash}$  and $\mathsf{t_1},\dots,\mathsf{t_{j+1}}\in\sigma^{-1}_{i}(\pi)$, then by \Cref{lemma:sharpest-box} and the successive application of rule \rsub we have $\standb{\s{\pi}}[R_1\sqsubseteq R]\in\KB^{\vdash}$

                \item[Base Case]: First, let us show that $\dlstruct_0\models\standbs[\exists R.C\sqsubseteq D]$. We let $\pi\in\sigma_0(\st)$, $(\delta_1,\delta_2)\in R^{\gamma_{0}(\pi)}$ and $C\in\Lambda_{\pi}(\delta_2)$.
                \begin{enumerate}[label=(\arabic*)]
                    \item $\standb{\s{\pi}}[\exists R.C\sqsubseteq D]\in\KB$ by assumption and \Cref{lemma:sharpest-box}
                    \item $\allstandb[\{b\}\sqsubseteq \standb{\s{\pi}}[\top \Rightarrow C]] \in \KB$  by construction (since $\delta_2=b^{\gamma_{0}}$ for some atom $b$) and \Cref{lemma:sharpest-box}
                    \item $\standb{\s{\pi}}[R(a,b)]\in\KB^{\vdash}$ for some atom $a$ with $\delta_1=a^{\gamma_{0}}$ by construction and \Cref{lemma:sharpest-box}
                    \item $\standb{\s{\pi}} [\{a\} \sqsubseteq \exists R.C ]\in \KB$ by (2), (3) and rule \abeA
                    \item $\allstandb[\top\sqsubseteq \standb{\s{\pi}}[\{a\} \Rightarrow D]] \in \KB$ by (1), (4) and rule \exB
                    \item $\standb{\s{\pi}}[\{a\}\sqsubseteq \standb{\s{\pi}}[\top \Rightarrow D]] \in \KB$ by (5), \Cref{lemma:sharpest-box} and rule \latestrule, hence $D\in\Lambda_{\pi}(\delta_1)$ as desired
                \end{enumerate}

                \item[Inductive Step]: Now, we assume that $\dlstruct_{i-1}\models\standbs[\exists R.C\sqsubseteq D]$ and show that $\dlstruct_{i}\models\standbs[\exists R.C\sqsubseteq D]$.
                So, we need to consider the case where $\delta_2$ is fresh, $\delta_2\in\Delta_{i}\setminus\Delta_{i-1}$. Thus, $\delta_2$ has been introduced to satisfy some axiom $\standb{t}[E \sqsubseteq \exists T.F]\in\KB^{\vdash}$. We must show that if $(\delta_1,\delta_2)\in R^{\gamma_{i}(\pi)}$ and $C\in\Lambda_{\pi}(\delta_2)$ then $D\in\Lambda_{\pi}(\delta_1)$ for $\pi\in\sigma_{i}(\st)$. Hence $\delta_2=\delta'$. First, \textbf{(Case Loop)} we consider the -uniquely- self-loop case (where $\delta_1=\delta'$) for all $\pi\in\Pi$. Then, \textbf{(Case Forward)} we assume that $\delta_1\neq\delta'$, in which case we focus on $\pi=\pi^*$ since this is the only case that introduces a fresh relation into the saturation process, hence the other cases are trivial. Notice that \textbf{(Case Forward)} may contain loops but they do not need to be treated specially.

                \begin{description}[leftmargin=1em]
                    \item[(Case Loop)] Assume $\delta_1=\delta_2=\delta'$, $C\in\Lambda_{\pi}(\delta')$ and $(\delta',\delta')\in R^{\gamma_{i}(\pi)}$.
                        \begin{description}
                            \item[Case 1]: Assume $\pi = \pi^*$.
                            Since $C\in \Lambda_{\pi}(\delta)$ and by \Cref{lemma:sharpest-box} then $\standb{u} [\top\sqsubseteq \standb{\s{\pi}} [F \Rightarrow C]] \in \KB^\vdash$.
                            And since $(\delta',\delta')\in R^{\gamma_{i}(\pi)}\setminus R^{\gamma_{i-1}(\pi)}$, then there must be $\standb{u'} [\top\sqsubseteq \standb{\s{\pi}} [F \Rightarrow \exists R.\Self]] \in \KB^\vdash$. 
                            By application of rules \selfG and \exA, we obtain $\standb{\s{\pi}} [F \sqsubseteq \exists R.C]\in \KB^\vdash$. By the premise and \Cref{lemma:sharpest-box}, we obtain $\standb{\s{\pi}} [\exists R.C \sqsubseteq D]\in \KB^\vdash$. Now by rule \exB we obtain $\standball [\top\sqsubseteq \standb{\s{\pi}} [F \Rightarrow D]] \in \KB^\vdash$ and hence $D\in \Lambda_{\pi}(\delta)$ as desired

                            \item[Case 2]: Assume $\pi \in \Precs_{i} \setminus \{ \pi^* \}$ and $\pi\in\sigma(\st)$. Since $C\in \Lambda_{\pi}(\delta')$ then for some ${G,H \in \mathrm{Con}(F,\pi^*)}$ there is $\standb{\s{\pi^*}}[G \sqsubseteq \standb{\s{\pi}} [\top \Rightarrow C] ]\in \KB^\vdash$ by \Cref{lemma:sharpest-box}. And since $(\delta',\delta')\in R^{\gamma_{i}(\pi)}\setminus R^{\gamma_{i-1}(\pi)}$, then there must be $\standb{\s{\pi^*}}[H \sqsubseteq \standb{\s{\pi}} [\top \Rightarrow \exself] ]\in \KB^\vdash$.
                            \begin{enumerate}[label=(\arabic*)]
                                \item $\standb{\s{\pi}}[\exists R.C\sqsubseteq D]\in\KB^{\vdash}$ by the assumption
                                \item $\standb{\s{\pi^*}}[E \sqsubseteq \exists R.F]\in\KB^{\vdash}$, by construction and \Cref{lemma:sharpest-box}
                                \item $\standb{u_1} [\top\sqsubseteq \standb{\s{\pi^*}} [F \Rightarrow G]] \in \KB^\vdash$
                                \item $\standb{u_2} [\top\sqsubseteq \standb{\s{\pi^*}} [F \Rightarrow H]] \in \KB^\vdash$
                                \item $\standb{\s{\pi^*}}[G \sqsubseteq \standb{\s{\pi}} [\top \Rightarrow C] ]\in \KB^\vdash$, by construction and \Cref{lemma:sharpest-box}
                                \item $\standb{\s{\pi^*}}[H \sqsubseteq \standb{\s{\pi}} [\top \Rightarrow \exself] ]\in \KB^\vdash$, by construction and \Cref{lemma:sharpest-box}
                                \item $\standb{\s{\pi^*}}[F \sqsubseteq \standb{\s{\pi}} [\top \Rightarrow C] ]\in \KB^\vdash$, by (3), (5) and rule \subB
                                \item $\standb{\s{\pi^*}}[F \sqsubseteq \standb{\s{\pi}} [\top \Rightarrow \exself] ]\in \KB^\vdash$, by (4), (6) and rule \subB
                                \item $\standb{\s{\pi}}[\exists R.\Self\sqcap C \sqsubseteq D]\in\KB^{\vdash}$ by (1) and rule \selfC
                                \item $\standb{\s{\pi^*}}[F \sqsubseteq \standb{\s{\pi}} [\top \Rightarrow D ] ]\in \KB^\vdash$, by (8), (9), (10) and rule \con, and hence $D\in \Lambda_{\pi}(\delta)$ as desired.
                            \end{enumerate}
                            \item [Case 3]: Assume $ \pi = \pi_{\delta',\standd{s} E}$ and $\pi\in\sigma(\st)$. Since $C\in \Lambda_{\pi}(\delta')$
                                  then for some ${G,H \in \mathrm{Con}(F,\pi^*)}$ by \Cref{lemma:sharpest-box} there is
                                  $\standb{\s{\pi^*}}[G \sqsubseteq \standb{\s{\pi}} [E \Rightarrow C] ]\in \KB^\vdash$ and also
                                  $\standb{\s{\pi^*}}[H \sqsubseteq \standb{\s{\pi}} [E \Rightarrow \exself] ]\in \KB^\vdash$
                                  \begin{enumerate}[label=(\arabic*)]
                                      \item $\standb{\s{\pi}}[\exists R.C\sqsubseteq D]\in\KB^{\vdash}$ by the assumption
                                      \item $\standb{\s{\pi^*}}[E \sqsubseteq \exists R.F]\in\KB^{\vdash}$, by construction and \Cref{lemma:sharpest-box}
                                      \item $\standb{u_1} [\top\sqsubseteq \standb{\s{\pi^*}} [F \Rightarrow G]] \in \KB^\vdash$
                                      \item $\standb{u_2} [\top\sqsubseteq \standb{\s{\pi^*}} [F \Rightarrow H]] \in \KB^\vdash$
                                      \item $\standb{\s{\pi^*}}[E \sqsubseteq \exists R.F]\in\KB^{\vdash}$, by construction and \Cref{lemma:sharpest-box}
                                      \item $\standb{\s{\pi^*}}[G \sqsubseteq \standb{\s{\pi}} [E \Rightarrow C] ]\in \KB^\vdash$, by construction and \Cref{lemma:sharpest-box}
                                      \item $\standb{\s{\pi^*}}[H \sqsubseteq \standb{\s{\pi}} [E \Rightarrow \exself] ]\in \KB^\vdash$, by construction and \Cref{lemma:sharpest-box}
                                      \item $\standb{\s{\pi^*}}[F \sqsubseteq \standb{\s{\pi}} [E \Rightarrow C] ]\in \KB^\vdash$, by (3), (6) and rule \subB
                                      \item $\standb{\s{\pi^*}}[F \sqsubseteq \standb{\s{\pi}} [E \Rightarrow \exself] ]\in \KB^\vdash$, by (4), (7) and rule \subB
                                      \item $\standb{\s{\pi}}[\exists R.\Self\sqcap C \sqsubseteq D]\in\KB^{\vdash}$ by (1) and rule \selfC
                                      \item $\standb{\s{\pi^*}}[F \sqsubseteq \standb{\s{\pi}} [E \Rightarrow D ] ]\in \KB^\vdash$, by (8), (9), (10) and rule \con, and hence $D\in \Lambda_{\pi}(\delta')$ as desired
                                  \end{enumerate}
                        \end{description}

                    \item[(Case Forward)]
                        Now, since we have $(\delta_1,\delta')\in R^{\gamma_{i+1}(\pi)}$ then by the saturation process that constructs $[R^{\gamma_{i+1}(\pi)}]_k$ there must be a sequence of domain elements $\delta'_1,\dots,\delta'_m,\delta'_{m+1}$ with $\delta'_1=\delta_1$ and $\delta'_m=\delta$.
                        We notice that by construction we have the following formulas:
                        \begin{enumerate}[label={(P\arabic*)}]
                            \item We have $C=C'_{m+1}$, $R=R'_{{1|m}}$ and $T=R_{{m|m}}$
                            \item $C_{j},C'_{j}\in\Lambda_{\pi}(\delta'_{j})$ for all  $j \in \{k,\ldots,m\}$
                            \item There is some $k\in \{1,\ldots,m-1\}$ such that
                                  \begin{enumerate}
                                      \item[(P3.a)] $\standb{\s{\pi}}[C'_{j}\sqsubseteq\exists R_{j|j}.C_{j+1}]\in\kb^\vdash$ for all $j\in \{k,\ldots,m\}$, and
                                      \item[(P3.b)] $\standb{\s{\pi}}[R_{j|j}(a_j,a_{j+1})]\in\kb^\vdash$ for all $j\in \{1,\ldots,k-1\}$
                                  \end{enumerate}

                            \item $\standb{\s{\pi}}[R_{j|j} \sqsubseteq R'_{j|j}],\standb{\s{\pi}}[R_{j|m} \sqsubseteq R'_{j|m}]\in\kb^\vdash$ (Notice that by axiom \axD and \Cref{lemma:sharpest-box} we can obtain formulas $\standb{\s{\pi}}[R' \sqsubseteq R']\in\kb^\vdash$)
                            \item If $1\leq j< m$, $\standb{\s{\pi}}[R'_{{j|j}}\circ R'_{j+1|m}\sqsubseteq R_{j|m}]\in\kb^\vdash$
                        \end{enumerate}

                        We first show that we can obtain $\standb{\s{\pi}}[C'_m \sqsubseteq \exists R'_{m|m}.C'_{m+1}]\in\KB^{\vdash}$:

                        \begin{enumerate}[label={(\arabic*)}]
                            \item $\standb{u}[\top\sqsubseteq \standb{\s{\pi}} [C_{m+1} \Rightarrow C'_{m+1}]] \in \KB^\vdash$ from $C\in \Lambda_{\pi}(\delta'_{m+1})$ by construction with $C=C'_{m+1}$
                            \item $\standb{\s{\pi}}[C'_{m} \sqsubseteq \exists R_{m|m}.C_{m+1}]\in\KB^{\vdash}$ by the assumption of \textbf{Case 1}
                            \item $\standb{\s{\pi}}[R_{m|m}\sqsubseteq R'_{m|m}]\in\kb^{\vdash}$ by construction (P4)
                            \item $\standb{\s{\pi}}[C'_m \sqsubseteq \exists R'_{m|m}.C'_{m+1}]\in\KB^{\vdash}$ by (1), (2), (3) and rule \exA
                        \end{enumerate}

                        We show that for $j>k$, if we have $\standb{\s{\pi}}[C'_j \sqsubseteq \exists R'_{j|m}.C'_{m+1}]\in\KB^{\vdash}$ then we can obtain $\standb{\s{\pi}}[C'_{j-1} \sqsubseteq \exists R'_{{j-1}|m}.C'_{m+1}]\in\KB^{\vdash}$.

                        \begin{enumerate}[label={(\arabic*)}]
                            \item $\standb{u}[\top\sqsubseteq \standb{\s{\pi}} [C_{j} \Rightarrow C'_{j}]] \in \KB^\vdash$ from $C'_{j}\in \Lambda_{\pi}(\delta'_{j})$ by construction
                            \item $\standb{\s{\pi}}[C'_{j-1} \sqsubseteq \exists R_{{j-1}|{j-1}}.C_{j}]\in\KB^{\vdash}$ by construction (P3.a)
                            \item $\standb{\s{\pi}}[R_{{j-1}|{j-1}}\sqsubseteq R'_{{j-1}|{j-1}}]\in\kb^{\vdash}$ by construction (P4)
                            \item $\standb{\s{\pi}}[R'_{{{j-1}|{j-1}}}\circ R'_{j|m}\sqsubseteq R_{{j-1}|m}]\in\kb^\vdash$ by construction (P5)
                            \item $\standb{\s{\pi}}[C'_j \sqsubseteq \exists R'_{j|m}.C'_{m+1}]\in\KB^{\vdash}$ by inductive hypothesis
                            \item $\standb{\s{\pi}}[C'_{j-1} \sqsubseteq \exists R'_{{j-1}|{j-1}}.C'_{j}]\in\KB^{\vdash}$ by (1), (2), (3) and rule \exA
                            \item $\standb{\s{\pi}}[C'_{j-1} \sqsubseteq \exists R'_{{j-1}|m}.C'_{m+1}]\in\KB^{\vdash}$ by (4), (5), (6) and rule \exC
                        \end{enumerate}

                        Now, assume $k=1$; we show the base case where $j=1$.
                        \begin{enumerate}[label={(\arabic*)}]
                            \item $\standb{\s{\pi}}[\exists R'_{{1|m}}.C'_{m+1}\sqsubseteq D]\in\KB^{\vdash}$ by the assumption
                            \item $\standb{\s{\pi}}[C'_1 \sqsubseteq \exists R'_{{1|m}}.C'_{m+1}]\in\KB^{\vdash}$ by the inductive hypothesis
                            \item $\standb{u} [\top\sqsubseteq \standb{\s{\pi}} [F \Rightarrow C'_1]] \in \KB^\vdash$  from $C'_1\in \Lambda_{\pi}(\delta')$ by construction
                            \item $\standball[\top\sqsubseteq \standb{\s{\pi}}[C'_1 \Rightarrow D]]\in\KB^{\vdash}$ by (1), (2) and rule \exB
                            \item $\standb{u}[\top\sqsubseteq \standb{\s{\pi}}[C'_1 \Rightarrow D]]\in\KB^{\vdash}$ by (4), and rule \genr
                            \item $\standb{u} [\top\sqsubseteq \standb{\s{\pi}} [F \Rightarrow D]] \in \KB^\vdash$ by (3), (5) and rule \subA and hence $D\in \Lambda_{\pi}(\delta_1)$ as desired
                        \end{enumerate}

                        Else, in case $j=k\neq 1$, we first we obtain
                        \begin{enumerate}[label={(\arabic*)}]
                            \item $\standb{\s{\pi}}[R'_{j-1|j-1}(a_{j-1},a_{j})]\in \KB^\vdash$ by assumption (P3.b)
                            \item $\standb{\s{\pi}}[C'_{j} \sqsubseteq \exists R'_{{j}|m}.C'_{m+1}]\in\KB^{\vdash}$ by assumption
                            \item $\allstandb[\{a_{j}\} \sqsubseteq \standb{\s{\pi}} [\top \Rightarrow C'_{j}]]\in\kb^{\vdash}$
                            \item $\standb{\s{\pi}}[R'_{{j-1|j-1}}\circ R'_{j|m}\sqsubseteq R_{j|m}]\in\kb^\vdash$ by assumption (P5)
                            \item $\standb{\s{\pi}}[R_{{j-1}|{m}}\sqsubseteq R'_{{j-1}|{m}}]\in\kb^{\vdash}$ by construction (P4)
                            \item $\allstandb[\{a_{j-1}\} \sqsubseteq \exists R_{{j-1}|m}.C'_{m+1}]\in\kb^{\vdash}$ by (1), (2), (3), (4) and rule \lastminute
                            \item $\allstandb[\{a_{j-1}\} \sqsubseteq \exists R'_{{j-1}|m}.C'_{m+1}]\in\kb^{\vdash}$ by (5), (6) and rule \abB
                        \end{enumerate}

                        Else, in case $j \in \{2,\ldots,k-1\}$ if we have $\allstandb[\{a_{j}\} \sqsubseteq \exists R'_{{j}|m}.C'_{m+1}]\in\kb^{\vdash}$ then we can obtain $\allstandb[\{a_{j-1}\} \sqsubseteq \exists R'_{{j-1}|m}.C'_{m+1}]\in\kb^{\vdash}$.

                        \begin{enumerate}[label={(\arabic*)}]
                            \item $\standb{\s{\pi}}[R_{j-1|j-1}(a_{j-1},a_{j})]$ by construction (P3.b)
                            \item $\allstandb[\{a_{j}\} \sqsubseteq \exists R'_{{j}|m}.C'_{m+1}]\in\kb^{\vdash}$ by assumption
                            \item $\standb{\s{\pi}}[R'_{{j-1|j-1}}\circ R'_{j|m}\sqsubseteq R_{j|m}]\in\kb^\vdash$ by assumption (P5)
                            \item $\standb{\s{\pi}}[R_{{j-1}|{m}}\sqsubseteq R'_{{j-1}|{m}}]\in\kb^{\vdash}$ by construction (P4)
                            \item $\allstandb[\{a_{j-1}\} \sqsubseteq \exists R_{{j-1}|m}.C'_{m+1}]\in\kb^{\vdash}$ by (1), (2), (3) and rule \abeB
                            \item $\allstandb[\{a_{j-1}\} \sqsubseteq \exists R'_{{j-1}|m}.C'_{m+1}]\in\kb^{\vdash}$ by (4), (5) and rule \abB
                        \end{enumerate}

                        Finally, assume the base case where $j=1$ and $k\neq 1$
                        \begin{enumerate}[label={(\arabic*)}]
                            \item $\standb{\s{\pi}}[\exists R'_{{1|m}}.C'_{m+1}\sqsubseteq D]\in\KB^{\vdash}$ by the assumption
                            \item $\allstandb[\{a_{1}\} \sqsubseteq \exists R'_{{1}|m}.C'_{m+1}]\in\kb^{\vdash}$ by induction
                            \item $\standb{u} [\top\sqsubseteq \standb{\s{\pi}} [\{a_{1}\} \Rightarrow D]] \in \KB^\vdash$  by (1), (2) and rule \exB
                            \item $\standb{u} [\{a_{1}\}\sqsubseteq \standb{\s{\pi}} [\top \Rightarrow D]] \in \KB^\vdash$  by (3) and rule \loc, $D\in \Lambda_{\pi}(\delta_1)$ as desired

                        \end{enumerate}

                \end{description}
            \end{description}

        \item[$\phi=\standbs(C(a))\ $] Assume $\standbs(C(a))\in\KB$ (i.e. $\allstandb[\{a\}\sqsubseteq \standb{s}[\top \Rightarrow C]] \in \KB^\vdash$).
            Then, by construction $\delta_a= a^{\gamma_{0}}$ and $C\in\Lambda_{\pi}(\delta_a)$. 

        \item[$\phi=\standbs(R(a,b))\ $] Assume $\standbs(R(a,b))\in\KB$.
            Then, by construction $\delta_a= a^{\gamma_{0}}$, $\delta_b= b^{\gamma_{0}}$ and $(\delta_a,\delta_b)\in R^{\gamma_{0}(\pi)}$ for all $\pi\in\sigma_{0}(\st)$. Thus also by construction $\delta_a= a^{\gamma}$, $\delta_b= b^{\gamma}$ and $(\delta_a,\delta_b)\in R^{\gamma(\pi)}$ for all $\pi\in\sigma(\st)$ and hence $\dlstruct\models\standbs(R(a,b))$.
    \end{description}
    \medskip

    With this, Claim \ref{claim:compl-claim} is proved.
    \vspace{1em}

    Now, returning to the proof of the completeness theorem, we show that it follows from the claim above. It remains to show that $\dlstruct$ is indeed a model. First, we observe that clearly, by the construction of $\Pi_0$ and $\sigma_0$, no standpoint is empty. Moreover, by axiom $\axA$ we obtain that $\sigma_0(*)=\Pi_0$ as required, and we remark that by construction this carries to $\sigma(*)=\Pi$. Finally, we need to make sure that for all $\pi\in\Pi$ and $\delta\in\Delta$, $\bot\notin\Lambda_{\pi}(\delta)$, so all domain elements can be instantiated in all precisisfications.

    Assume for the sake of contradiction that $\standball [\top \sqsubseteq \standball[\top \Rightarrow \bot]]\notin\kb^\vdash$ but in $\dlstruct$ there is $\bot\in\Lambda_{\pi}(\delta')$ for some $\delta'\in\Dom$. First, assume that $\delta'\in\Dom_0$ of $\dlstruct_0$, and thus $\bot\in\Lambda_{\pi}(\delta')$. Then by construction there is some $\standb{u}[\{a\} \sqsubseteq \standb{\s{\pi}}[\top \Rightarrow \bot]] \in \KB^\vdash$ for some named individual $a$. But then by rule \abr we have $\standball [\top \sqsubseteq \standball[\top \Rightarrow \bot]]\in\kb^\vdash$, thus reaching a contradiction. Thus, if $\delta'\notin\Dom_0$ of $\dlstruct_0$, there must be a $\dlstruct_{i+1}$ with $\delta'\in\Dom_{i+1}\setminus\Dom_i$ and $\bot\in\Lambda_{\pi}(\delta')$. Notice that the iteration has been triggered at $\pi^*$ to satisfy an axiom of the form $\standb{t}[E \sqsubseteq \exists R.F]$, and by \Cref{lemma:sharpest-box} we have $\standb{\s{\pi^*}}[E \sqsubseteq \exists R.F]$.
    \begin{description}
        \item[Case $\pi=\pi^*$] By construction we have $\Lambda_{\pi^*}(\delta')= \{ A \mid \standb{u}[ \top \sqsubseteq \standb{s} [F \Rightarrow A] ] \in \KB^\vdash, \st \in \sigma^{-1}(\pi^*) \}$ and thus by the assumption we have some $\standb{u}[ \top \sqsubseteq \standb{\s{\pi^*}} [F \Rightarrow \bot]]$. By axiom \axD and \Cref{lemma:sharpest-box} we obtain also $\standb{\s{\pi^*}}[R \sqsubseteq R]$ and by rule \exA we get $\standb{\s{\pi^*}}[E \sqsubseteq \exists R.\bot]$. Then, by rule \retA we obtain $\standball	[\top \sqsubseteq \standb{\s{\pi^*}} [E \Rightarrow \bot] ]$. This means that in $\dlstruct_{i}$ there is $\bot\in\Lambda_{\pi^*}(\delta^*)$ since $\delta^*\in\Dom_i$.

        \item[Case $\pi\in\Pi_i\setminus\pi^*$] By construction and \Cref{lemma:sharpest-box} if $\bot\in\Lambda_{\pi}(\delta')$ then  $\standb{\s{\pi^*}}[G \sqsubseteq \standb{\s{\pi}} [\top \Rightarrow \bot]] \in \KB^\vdash$ for $G\in\Lambda_{\pi^*}(\delta')$. By rule \retB we obtain $\allstandb[\top \sqsubseteq \standb{\s{\pi^*}} [G \Rightarrow \bot]] \in \KB^\vdash$, which leads to $\bot\in\Lambda_{\pi^*}(\delta')$ and, by the previous case (\textbf{Case $\pi=\pi^*$}), in $\dlstruct_{i}$ there is $\bot\in\Lambda_{\pi^*}(\delta^*)$ since $\delta^*\in\Dom_i$.

        \item [Case $\pi=\pi_{\delta',\standd{s} E}$] By construction  and \Cref{lemma:sharpest-box} if $\bot\in\Lambda_{\pi_{\delta',\standd{s} E}}(\delta')$ then  $\standb{\s{\pi^*}}[G \sqsubseteq \standb{\s{\pi}} [E \Rightarrow \bot]] \in \KB^\vdash$ and $\standb{\s{\pi^*}}[H \sqsubseteq \standd{\s{\pi}} E] \in \KB^\vdash$ for $G,H\in\Lambda_{\pi^*}(\delta')$. By rule \subB we obtain that also $\standb{\s{\pi^*}}[F \sqsubseteq \standb{\s{\pi}} [E \Rightarrow \bot]] \in \KB^\vdash$ and by rule \subC we obtain $\standb{\s{\pi^*}}[F \sqsubseteq \standd{\s{\pi}} E] \in \KB^\vdash$. Then by rule \subD we obtain $\standb{\s{\pi^*}}[F \sqsubseteq \standd{\s{\pi}} \bot] \in \KB^\vdash$ and by rule \retC we get $\allstandb[\top \sqsubseteq \standb{\s{\pi^*}} [F \Rightarrow \bot]] \in \KB^\vdash$. Hence, by the first case (\textbf{Case $\pi=\pi^*$}) in $\dlstruct_{i}$ there is $\bot\in\Lambda_{\pi^*}(\delta^*)$ since $\delta^*\in\Dom_i$.
    \end{description}
    \medskip

    Thus, if there is some $\delta\in\Dom_i$ of $\dlstruct_i$ such that $\bot\in\Lambda_{\pi}(\delta)$, then there is $\delta'\in\Dom_0$ of $\dlstruct_0$ such that $\bot\in\Lambda_{\pi}(\delta')$, but as we showed this implies  $\standball [\top \sqsubseteq \standball[\top \Rightarrow \bot]]\in\kb^\vdash$, thus reaching a contradiction.
\end{proof}

\newpage

\end{document}